\documentclass[twoside,11pt]{article}
\usepackage{microtype}
\usepackage[abbrvbib]{jmlr2e}

\newenvironment{proofnosquare}{\par\noindent{\bf Proof\ }}{}

\newtheorem{assumption}[theorem]{Assumption}

\usepackage{lastpage}
\jmlrheading{26}{2025}{1-\pageref{LastPage}}{1/25}{6/25}{25-0035}{Justin Bunker, Mark Girolami, Hefin Lambley, Andrew M.~Stuart and T.~J.~Sullivan}

\ShortHeadings{Autoencoders in Function Space}{Bunker, Girolami, Lambley, Stuart, and Sullivan}
\firstpageno{1}

\usepackage{mathtools,mathrsfs,wrapfig}
\usepackage{subcaption,float,color,xcolor,graphicx,enumitem,booktabs}
\usepackage{dsfont} 
\usepackage[normalem]{ulem} 
\usepackage{hyperref,nicefrac}

\usepackage[noabbrev,capitalise,nosort,nameinlink]{cleveref}

\crefformat{footnote}{#2\footnotemark[#1]#3}

\usepackage{graphics}
\usepackage{tikz, tikz-cd}
\usetikzlibrary{decorations.pathreplacing,calc}

\tikzset{point/.style={inner sep=1.8pt,fill,circle}}

\setlist{topsep=-0.1em, itemsep=-0.1em} 

\newenvironment{proofof}[1]%
{\par\noindent{\textbf{Proof of #1}\ }}%
{\hfill\BlackBox\vspace{2mm}}

\DeclareMathOperator*{\argmax}{arg\,max}
\DeclareMathOperator*{\Var}{Var}
\DeclareMathOperator*{\argmin}{arg\,min}
\DeclareMathOperator{\Uniform}{\mathrm{Uniform}}

\DeclareMathOperator{\tr}{tr}

\newcommand*{\pushforwardmeas}{\Sigma}
\newcommand*{\objective}{\mathcal{J}}
\newcommand*{\objectiveFVAE}{\objective^{\text{FVAE}}}
\newcommand*{\objectiveFAE}{\objective^{\text{FAE}}}
\newcommand*{\persampleloss}{\mathcal{L}}
\newcommand*{\ELBO}{\mathrm{ELBO}}
\newcommand*{\spd}[1]{\mathcal{S}_{+}(#1)}

\newcommand*{\Pwint}[2]{\Innerprod{#1}{#2}^{\sim}}

\newcommand*{\KLdiv}[2]{D_{\mathrm{KL}}(#1 \,\|\, #2)}
\newcommand*{\bigKLdiv}[2]{D_{\mathrm{KL}}\bigl(#1 \,\big\|\, #2\bigr)}
\newcommand*{\dataspace}{\mathcal{U}}
\newcommand*{\datameas}{\Upsilon}
\newcommand*{\datadensity}{\upsilon}
\newcommand*{\encodermap}{\mathsf{f}}
\newcommand*{\encodermean}{f}
\newcommand*{\encodercov}{\Sigma}
\newcommand*{\decodermap}{g}
\newcommand*{\decodernoisemeas}{\PP_{\eta}}
\newcommand*{\autoencodermeas}{\mathbb{A}}
\newcommand*{\latentmeas}{\mathbb{P}_{z}}
\newcommand*{\decodermeas}{\PP_{u \mid z}}
\newcommand*{\encoderparam}{\theta}
\newcommand*{\Encoderparam}{\Theta}
\newcommand*{\latentgenparam}{\varphi}
\newcommand*{\Latentgenparam}{\Phi}
\newcommand*{\decoderparam}{\psi}
\newcommand*{\Decoderparam}{\Psi}
\newcommand*{\genmeas}{\PP_{u}}
\newcommand*{\encodermeas}{\QQ_{z \mid u}}
\newcommand*{\varmeas}{\QQ_{z \mid u}}
\newcommand*{\posteriormeas}{\PP_{z \mid u}}
\newcommand*{\referencemeas}{\Lambda}
\newcommand*{\statisticaldistance}{\mathsf{d}}
\newcommand*{\encoderjointmeas}{\mathbb{Q}_{z, u}}
\newcommand*{\encoderlatentmeas}{\mathbb{Q}_{z}}
\newcommand*{\decoderjointmeas}{\mathbb{P}_{z, u}}
\newcommand*{\latentspace}{\mathcal{Z}}
\newcommand*{\CMspace}[1]{H(#1)}

\newcommand*{\TT}{\mathbb{T}}

\newcommand*{\defterm}{\textbf}

\renewcommand*{\epsilon}{\varepsilon}

\newcommand*{\Naturals}{\mathbb{N}}

\newcommand*{\one}{\mathds{1}}
\newcommand*{\PP}{\mathbb{P}}
\newcommand*{\QQ}{\mathbb{Q}}

\DeclareMathOperator*{\EE}{\mathbb{E}}
\newcommand*{\qefed}{\eqqcolon}
\newcommand*{\quark}{\setbox0\hbox{$x$}\hbox to\wd0{\hss$\cdot$\hss}}
\newcommand*{\rd}{\mathrm{d}}

\newcommand*{\Reals}{\mathbb{R}}

\newcommand*{\prob}[1]{\mathscr{P}(#1)}

\renewcommand*{\geq}{\geqslant}

\renewcommand*{\leq}{\leqslant}

\renewcommand*{\mathbf}{\boldsymbol}

\newcommand*{\dirac}[1]{\delta_{#1}}

\newcommand*{\arXiv}[1]{\bgroup\color{black}\href{https://arxiv.org/abs/#1}{arXiv:#1}\egroup}
\newcommand*{\doi}[1]{\bgroup\color{black}\href{https://doi.org/#1}{doi:#1}\egroup}

\newcommand{\innerprod}[2]{\langle #1 , #2 \rangle}
\newcommand{\norm}[1]{\lVert #1 \rVert}
\newcommand{\set}[2]{\{ #1 \mid #2 \}}
\newcommand{\bigabsval}[1]{\bigl\vert #1 \bigr\vert}

\newcommand{\biginnerprod}[2]{\bigl\langle #1 , #2 \bigr\rangle}
\newcommand{\bignorm}[1]{\bigl\Vert #1 \bigr\Vert}
\newcommand{\bigset}[2]{\bigl\{ #1 \,\big\vert\, #2 \bigr\}}

\newcommand{\Innerprod}[2]{\left\langle #1 , #2 \right\rangle}
\newcommand{\Norm}[1]{\left\Vert #1 \right\Vert}
\newcommand{\Set}[2]{\left\{ #1 \,\middle\vert\, #2 \right\}}

\newcommand*{\qedremark}{\hfill $\blacksquare$}

\crefname{assumption}{Assumption}{Assumptions}

\begin{document}

\title{Autoencoders in Function Space}

\author{\name Justin Bunker \email jb2200@cantab.ac.uk \\
       \addr Department of Engineering\\
       University of Cambridge\\
       Cambridge, CB2 1TN, United Kingdom
       \AND
       \name Mark Girolami \email mgirolami@turing.ac.uk \\
	   \addr Department of Engineering, University of Cambridge\\
	   and Alan Turing Institute\\
	   Cambridge, CB2 1TN, United Kingdom
	   \AND
	   \name Hefin Lambley \email hefin.lambley@warwick.ac.uk\\
	   \addr Mathematics Institute\\
	   University of Warwick\\
	   Coventry, CV4 7AL, United Kingdom
	   \AND
	   \name Andrew M.~Stuart \email astuart@caltech.edu\\
	   \addr 
		Computing + Mathematical Sciences\\
		California Institute of Technology\\
		Pasadena, CA 91125, United States of America 
	   \AND
	   T.~J.~Sullivan \email t.j.sullivan@warwick.ac.uk\\
	   \addr Mathematics Institute \& School of Engineering\\
	   University of Warwick\\
	   Coventry, CV4 7AL, United Kingdom
   }
\editor{Mahdi Soltanolkotabi}

\maketitle

\begin{abstract}%
Autoencoders have found widespread application in both their original deterministic
form and in their variational formulation (VAEs). 
In scientific applications and in image processing it is often of interest to consider data that are viewed as functions;
while discretisation
(of differential equations arising in the sciences) or pixellation (of images) renders problems
finite dimensional in practice, conceiving first of algorithms that operate on functions, and only then discretising or pixellating, leads to better algorithms that smoothly operate between resolutions.
In this paper function-space versions of
the autoencoder (FAE) and variational autoencoder (FVAE) are introduced, analysed, and deployed.
Well-definedness of the objective governing VAEs is a subtle issue, particularly in function space, limiting applicability.
For the FVAE objective to be well defined requires compatibility of the data distribution with the chosen generative model; this can be achieved, for example, when the data arise from a stochastic differential equation, but is generally restrictive. The FAE objective, on the other hand, is well defined in many situations where FVAE fails to be.
Pairing the FVAE and FAE objectives with neural operator architectures that can be evaluated on any mesh enables new applications of autoencoders to inpainting, superresolution, and generative modelling of scientific data.
\end{abstract}


\section{Introduction}

Functional data, and data that can be viewed as a high-resolution approximation of functions, are ubiquitous in data science \citep{Ramsay2002}.
Recent years have seen much interest in machine learning in this setting, with the promise of architectures that can be trained and evaluated across resolutions.
A variety of methods now exist for the supervised learning of operators between function spaces, starting with the work of \cite{ChenChen1993}, followed by DeepONet \citep{LuJinPangZhangKarniadakis2021}, PCA-Net \citep{BhattacharyaHosseiniKovachkiStuart2021}, Fourier neural operators (FNOs; \citealp{Lietal2021}) and variants \citep{Kovachkietal2023}.
These methods have proven useful in diverse applications such as surrogate modelling for costly simulators of dynamical systems \citep{Azizzadeneshelietal2024}.

Practical algorithms for functional data must necessarily operate on discrete representations of the underlying infinite-dimensional objects, identifying salient features independent of resolution.
Some models make this dimension reduction explicit by representing outputs as a linear combination of basis functions---learned from data in DeepONet, and computed using principal component analysis (PCA) in PCA-Net.
Others do this implicitly, as in, for example, the deep layered structure of FNOs involving repeated application of the discrete Fourier transform followed by pointwise activation.

While linear dimension-reduction methods such as PCA adapt readily to function space, there are many types of data, such as solutions to advection-dominated partial differential equations (PDEs), for which linear approximations are provably inefficient---a phenomenon known as the Kolmogorov barrier \citep{Peherstorfer2022}.
This suggests the need for nonlinear dimension-reduction techniques on function space. 
Motivated by this we propose an extension of variational autoencoders (VAEs; \citealp{KingmaWelling2014}) to functional data using operator learning; we refer to the resulting model as the \textbf{functional variational autoencoder (FVAE)}.
As a probabilistic latent-variable model, FVAE allows for both dimension reduction and principled generative modelling on function space.

We define the FVAE objective as the Kullback--Leibler (KL) divergence between two joint distributions, both on the product of the data and latent spaces, defined by the encoder and decoder models;  
we then derive conditions under which this objective is well-defined.
This differs from the usual presentation of VAEs in which one maximises a lower bound on the data likelihood, the evidence lower bound (ELBO);
we show that our objective is equivalent and that it generalises naturally to function space.
The FVAE objective proves to be well defined only under a compatibility condition between the data and the generative model;
this condition is easily satisfied in finite dimensions but is restrictive for functional data. 
Our applications of FVAE rest on establishing such compatibility, which is possible for problems in the sciences such as those arising in Bayesian inverse problems with Gaussian priors \citep{Stuart2010}, and those governed by stochastic differential equations (SDEs) \citep{HairerStuartVoss2011}.
However, there are many problems in the sciences, and in generative models for PDEs in particular, for which application of FVAE fails because the generative model is incompatible with the data;
using FVAE in such settings leads to foundational theoretical issues and, as a result, to practical problems in the infinite-resolution and -data limits. 

To overcome these foundational issues we propose a deterministic regularised autoencoder that can be applied in very general settings, which we call the \defterm{functional autoencoder (FAE)}.
We show that FAE is an effective tool for dimension reduction, and that it can used as a versatile generative model for functional data.

We complement the FVAE and FAE objectives on function space with neural-operator architectures that can be discretised on any mesh. 
The ability to discretise both the encoder and the decoder on arbitrary meshes extends prior work such as the variational autoencoding neural operator (VANO; \citealp{SeidmanKissasPappasPerdikaris2023}), and is highly empowering, enabling a variety of new applications of autoencoders to scientific data such as inpainting and superresolution.
Code accompanying the paper is available at
\vspace{-0.5em}
\begin{center}
    \url{https://github.com/htlambley/functional_autoencoders}.
\end{center}
\vspace{-1em}

\paragraph{Contributions.}
We make the following contributions to the development and application of autoencoders to functional data:
\begin{enumerate}[label=(C\arabic*)]
	\item
    \label{item:contribution_FVAE}
    we propose FVAE, an extension of VAEs to function space, finding that the training objective is well defined so long as the generative model is compatible with the data;

    \item 
    \label{item:contribution_architecture}
    we complement the FVAE training objective with 
    mesh-invariant architectures that can be deployed on any discretisation---even irregular, non-grid discretisations;

    \item 
    \label{item:contribution_problems}
    we show that when the data and generative model are incompatible, the discretised FVAE objective may diverge in the infinite-resolution and -data limits or entirely fail to minimise the divergence between the encoder and decoder;

    \item 
    \label{item:contribution_FAE}
    we propose FAE, an extension of regularised autoencoders to function space, and show that its objective is 
    well defined in many cases where the FVAE objective is not;

    \item 
    \label{item:contribution_masking}
     exploiting mesh-invariance, we propose masked
     training schemes which exhibit greater robustness at inference
     time, faster training and lower memory usage;
    
    \item 
    \label{item:contribution_mesh-invariance}
    we validate FAE and FVAE on examples from the sciences, including problems governed by SDEs and PDEs, to discover low-dimensional latent structure from data, and use our models for inpainting, superresolution, and generative modelling, exploiting the ability to discretise the encoder and decoder on any mesh.
\end{enumerate}

\paragraph{Outline.} 
\Cref{sec:VAE} extends VAEs to function space (Contribution \ref{item:contribution_FVAE}).
We then describe mesh-invariant architectures of Contribution \ref{item:contribution_architecture}; we also validate our approach, FVAE, on examples such as SDE path distributions where compatibility between the data and the generative model can be verified.
\Cref{sec:problems} gives an example of the problems arising when applying FVAE in the ``misspecified'' setting of Contribution \ref{item:contribution_problems}. 
In \Cref{sec:regularised_AEs} we propose FAE (Contribution \ref{item:contribution_FAE}) and apply our data-driven method to two examples from scientific machine learning:
Navier--Stokes fluid flows and Darcy flows in porous media.
In these problems we make use of the mesh-invariance of the FAE to apply the masked training scheme of Contribution \ref{item:contribution_masking};
masking proves to be vital for the applications to inpainting and superresolution in Contribution~\ref{item:contribution_mesh-invariance}.
\Cref{sec:related_work} discusses related work, and \Cref{sec:outlook} discusses limitations and topics for future research.

\section{Variational Autoencoders on Function Space}
\label{sec:VAE}

In this section we define the VAE objective by minimizing
the KL divergence between two distinct joint distributions on the product of data and latent spaces, one defined by the encoder
and the other by the decoder. We show that this gives a tractable objective when written in terms of a suitable reference distribution (\cref{subsec:VAE_objective}). This objective coincides with maximising the ELBO in finite dimensions (\cref{subsec:VAE_in_finite_dimensions}) and extends readily to infinite dimensions.
The divergence between the encoder and decoder models is finite only under a compatibility condition between the data and the generative model. This condition is restrictive in infinite dimensions. We identify problem classes for which this compatibility holds (\cref{subsec:VAE_in_infinite_dimensions}) and pair the objective with mesh-invariant encoder and decoder architectures (\cref{subsec:VAE_architecture}), leading to a model we call the \defterm{functional variational autoencoder (FVAE)}.
We then validate our method on several problems from scientific machine learning (\cref{subsec:FVAE_numerical_examples}) where the data are governed by SDEs.

\subsection{Training Objective}
\label{subsec:VAE_objective}

We begin by formulating an objective in infinite dimensions, making the following standard assumption in unsupervised learning  throughout \cref{sec:VAE}.
At this stage we focus on the continuum problem, and address the question of discretisation in \cref{subsec:VAE_architecture}.
In what follows $\prob{X}$ is the set of Borel probability measures on the separable Banach space $X$.

\begin{assumption}
	\label[assumption]{assumption:FVAE_assumptions}
	Take the \defterm{data space} $(\dataspace, \norm{\quark})$ to be a separable Banach space.
	There exists a \defterm{data distribution} $\datameas \in \prob{\dataspace}$ from which we have access to $N$ independent and identically distributed samples $\{u^{(n)}\}_{n = 1}^{N} \subset \dataspace$. 
    \qedremark
\end{assumption}

This setting is convenient to work with yet general enough to include many spaces of interest;
$\dataspace$ could be, for example, a Euclidean space $\Reals^{k}$, or the infinite-dimensional space $L^{2}(\Omega)$ of (equivalence classes of) square-integrable functions with domain $\Omega \subseteq \Reals^{d}$.
Separability is a technical condition used to guarantee the existence of conditional distributions in the encoder and decoder models defined shortly \citep[see][Theorem~1]{ChangPollard1997}.

An autoencoder consists of two nonlinear transformations:
an \defterm{encoder} mapping data into a low-dimensional \defterm{latent space} $\latentspace$, and a \defterm{decoder} mapping from $\latentspace$ back into $\dataspace$.
The goal is to choose transformations such that, for $u \sim \datameas$, composing the encoder and decoder approximates the identity.
In this way the autoencoder can be used for dimension reduction, with the encoder output being a compressed representation of the data.

To make this precise, fix  $\latentspace = \Reals^{d_{\mathcal{Z}}}$ and define encoder and decoder transformations depending on parameters $\encoderparam \in \Encoderparam$ and $\decoderparam \in \Decoderparam$, respectively, by
\begin{subequations}
\begin{align}
    \text{(encoder)~~~~~} \dataspace \ni u &\mapsto \encodermeas^{\encoderparam} \in \prob{\latentspace}, \label{eq:FVAE_encoder} \\
    \text{(decoder)~~~~~} \latentspace \ni z &\mapsto \decodermeas^{\decoderparam} \,\in \prob{\dataspace}. \label{eq:FVAE_decoder}
\end{align}
\end{subequations}
To ensure the statistical models we are interested in are well defined, we require both maps to be Markov kernels---that is, for all Borel measurable sets $A \subseteq \latentspace$ and $B \subseteq \dataspace$ and all parameters $\encoderparam \in \Encoderparam$ and $\decoderparam \in \Decoderparam$, the maps $u \mapsto \varmeas^{\encoderparam}(A)$ and $z \mapsto \decodermeas^{\decoderparam}(B)$ are measurable.
Since the encoder and decoder both output probability distributions, we must be precise about what it means to compose them:
we mean the distribution 
\begin{equation}  \label{eq:autoencoding_distribution}
    \text{(autoencoding distribution)}~~~\autoencodermeas^{\encoderparam,\decoderparam}_{u}(B) = \int_{\latentspace} \decodermeas^{\decoderparam}(B) \,\encodermeas^{\encoderparam}(\rd z),\quad B \subseteq \dataspace \text{~measurable.}
\end{equation}

\begin{remark}
    \label{rem:dirac}
    Given $u \in \dataspace$, $\autoencodermeas^{\encoderparam,\decoderparam}_{u}(\cdot)$ is the distribution given by drawing $z \sim \varmeas^{\encoderparam}$ and then sampling from $\decodermeas^{\decoderparam}$;
we would like this to be close to a Dirac distribution $\dirac{u}$ when $u$ is drawn from $\datameas$; enforcing this, approximately,
will be used to determine the parameters $(\encoderparam,\decoderparam).$ 
\qedremark
\end{remark}

\paragraph{Autoencoding as Matching Joint Distributions.} 
Now let us fix a \defterm{latent distribution} $\latentmeas \in \prob{\latentspace}$ and define two distributions for $(z, u)$ on the product space $\latentspace \times \dataspace$:
\begin{subequations}
    \label{eq:encoder--decoder_model}
\begin{align}
	\text{(joint encoder model)}~~~~~~\encoderjointmeas^{\encoderparam}(\rd z, \rd u) &= \encodermeas^{\encoderparam}(\rd z) \datameas(\rd u), \label{eq:encoder_model} \\
	\text{(joint decoder model)}~~\,~~~~\decoderjointmeas^{\decoderparam}(\rd z, \rd u) &= \decodermeas^{\decoderparam}(\rd u) \latentmeas(\rd z). \label{eq:decoder_model}
\end{align}
\end{subequations}
The joint encoder model \eqref{eq:encoder_model} is the distribution on $(z, u)$ given by applying the encoder to data $u \sim \datameas$, while the joint decoder model \eqref{eq:decoder_model} is the distribution on $(z, u)$ given by applying the decoder  to latent vectors $z \sim \latentmeas$.
We emphasise that $\datameas$ is given, while the distributions $\varmeas^{\encoderparam}$, $\decodermeas^{\decoderparam}$, and $\latentmeas$ are to be specified;
the choice of decoder distribution $\decodermeas^{\decoderparam}$ on the (possibly infinite-dimensional) space $\dataspace$ will require particular attention in what follows.

The marginal and conditional distributions of \eqref{eq:encoder_model} and \eqref{eq:decoder_model} will be important and we write, for example, $\posteriormeas^{\decoderparam}$ for the $z \mid u$-conditional of the joint decoder model and $\encoderlatentmeas^{\encoderparam}$ for the $z$-marginal of the joint encoder model.
In particular, the $u$-marginal of the joint decoder model is the \defterm{generative model} $\genmeas^{\decoderparam}$.

To match the joint encoder model \eqref{eq:encoder_model} with the joint decoder model \eqref{eq:decoder_model}, we seek to solve, for some statistical distance or divergence $\statisticaldistance$ on $\prob{\latentspace \times \dataspace}$, the minimisation problem
\begin{equation} \label{eq:minimise_statistical_distance}
    \argmin_{\encoderparam \in \Encoderparam,\,\decoderparam \in \Decoderparam} \statisticaldistance\bigl(\encoderjointmeas^{\encoderparam} \,\|\, \decoderjointmeas^{\decoderparam}\bigr).
\end{equation}
Doing so determines an autoencoder and a generative model.
The choice of $\statisticaldistance$ is constrained by the need for it to be possible to evaluate the objective with $\Upsilon$ known only empirically; 
examples for which this is the case include the Wasserstein metrics and the KL divergence.

\paragraph{VAE Objective as Minimisation of KL Divergence.} 
Using the KL divergence as $\statisticaldistance$ in \eqref{eq:minimise_statistical_distance} leads to the goal of finding $\encoderparam \in \Encoderparam$ and $\decoderparam \in \Decoderparam$ to minimise
\begin{equation}
	\label{eq:joint_Kullback--Leibler_divergence}
     \bigKLdiv{\encoderjointmeas^{\encoderparam}}{\decoderjointmeas^{\decoderparam}} = \EE_{(z, u) \sim \encoderjointmeas^{\encoderparam}} \left[ \log \frac{\rd \encoderjointmeas^{\encoderparam}}{\rd \decoderjointmeas^{\decoderparam}}(z, u) \right] =  \EE_{u \sim \datameas} \EE_{z \sim \varmeas^{\encoderparam}} \left[ \log \frac{\rd \encoderjointmeas^{\encoderparam}}{\rd \decoderjointmeas^{\decoderparam}}(z, u) \right];
\end{equation}
here $\rd \encoderjointmeas^{\encoderparam} / \rd \decoderjointmeas^{\decoderparam}$ is the Radon--Nikodym derivative of $\encoderjointmeas^{\encoderparam}$ with respect to $\decoderjointmeas^{\decoderparam}$, the appropriate infinite-dimensional analogue of the ratio of probability densities. 
This exists only when $\encoderjointmeas^{\encoderparam}$ is absolutely continuous with respect to $\decoderjointmeas^{\decoderparam}$, meaning that $\encoderjointmeas^{\encoderparam}$ assigns probability zero to a set whenever $\decoderjointmeas^{\decoderparam}$ does;
we take \eqref{eq:joint_Kullback--Leibler_divergence} to be infinite otherwise. 
This objective is equivalent to the standard VAE objective in the case $\dataspace=\Reals^{k}$ but has the additional advantage that it can be used in the infinite-dimensional setting.

KL divergence has many benefits justifying its use across statistics.
Its value in inference and generative modelling comes from the connection between maximum-likelihood methods and minimising KL divergence, as well as its information-theoretic interpretation \citep[Sec.~2.3]{CoverThomas2006}.
KL divergence is asymmetric and has two useful properties justifying the order of the arguments in \eqref{eq:joint_Kullback--Leibler_divergence}:
it can be evaluated using only samples of the distribution in its first argument, and it requires no knowledge of the normalisation constant of the distribution in its second argument since minimisers of $\nu \mapsto \KLdiv{\nu}{\mu}$ are invariant when scaling $\mu$ (\citealp{chen2023sampling} and \citealp[Sec.~12.2]{BachBaptistaSanzAlonsoStuart2024}).

To justify the use of the joint divergence \eqref{eq:joint_Kullback--Leibler_divergence}, in \Cref{thm:FVAE_theoretical_justification} we decompose the objective as the sum of two interpretable terms.
In \Cref{thm:FVAE_tractable_objective} we write the objective in a
form that leads to actionable algorithms.

\begin{theorem}
	\label{thm:FVAE_theoretical_justification}
	For all parameters $\encoderparam \in \Encoderparam$ and $\decoderparam \in \Decoderparam$ for which $\KLdiv{\encoderjointmeas^{\encoderparam}}{\decoderjointmeas^{\decoderparam}} < \infty$,
	\begin{equation} 
		\label{eq:FVAE_theoretical_justification}
		\bigKLdiv{\encoderjointmeas^{\encoderparam}}{\decoderjointmeas^{\decoderparam}} = 
		 \underbrace{ \bigKLdiv{\datameas}{\genmeas^{\decoderparam}} }_{\textup{(I)}} + \underbrace{ \EE_{u \sim \datameas} \Bigl[\bigKLdiv{\varmeas^{\encoderparam}}{\posteriormeas^{\decoderparam}}\Bigr] }_{\textup{(II)}} .
	\end{equation}
\end{theorem}

\begin{proof}
	Factorise $\encoderjointmeas^{\encoderparam}(\rd z, \rd u) = \encodermeas^{\encoderparam}(\rd z) \datameas(\rd u)$ and $\decoderjointmeas^{\decoderparam}(\rd z, \rd u) = \posteriormeas^{\decoderparam}(\rd z) \genmeas^{\decoderparam}(\rd u)$ and substitute into \eqref{eq:joint_Kullback--Leibler_divergence} to obtain
	\begin{align*}
		\bigKLdiv{\encoderjointmeas^{\encoderparam}}{\decoderjointmeas^{\decoderparam}} &= \EE_{\substack{(z, u) \\\sim \encoderjointmeas^{\encoderparam}}} \Biggl[ \log \frac{\rd \datameas}{\rd \genmeas^{\decoderparam}}(u) \frac{\rd \encodermeas^{\encoderparam}}{\rd \posteriormeas^{\decoderparam}}(z) \Biggr]  
		= \EE_{u \sim \datameas} \EE_{z \sim \varmeas^{\encoderparam}} \Biggl[ \log \frac{\rd \datameas}{\rd \genmeas^{\decoderparam}}(u) + \log \frac{\rd \encodermeas^{\encoderparam}}{\rd \posteriormeas^{\decoderparam}}(z) \Biggr].
	\end{align*}
	The result follows by inserting the definition of the KL divergences on the right-hand side of \eqref{eq:FVAE_theoretical_justification}, and all terms are finite owing to the assumption that $\KLdiv{\encoderjointmeas^{\encoderparam}}{\decoderjointmeas^{\decoderparam}} < \infty$.
\end{proof}

\Cref{thm:FVAE_theoretical_justification} decomposes the divergence \eqref{eq:joint_Kullback--Leibler_divergence} into (I) the error in the generative model and (II) the error in approximating the \defterm{decoder posterior} $\posteriormeas^{\decoderparam}$ with $\varmeas^{\encoderparam}$ via variational inference;
jointly training a variational-inference model and a generative model is exactly the goal of a VAE.
However, minimising \eqref{eq:joint_Kullback--Leibler_divergence} makes sense only if $(\encoderparam, \decoderparam) \mapsto \KLdiv{\encoderjointmeas^{\encoderparam}}{\decoderjointmeas^{\decoderparam}}$ is finite for some $\encoderparam$ and $\decoderparam$.
Verification of this property is a necessary step that we perform for our settings of interest in \Cref{prop:joint_divergence_finite_infimum_finite_dimensions,prop:joint_divergence_finite_infimum_infinite_dimensions,prop:FAE_infimum_finite}.

\begin{remark}[Posterior collapse]
	While minimising \eqref{eq:joint_Kullback--Leibler_divergence} typically leads to a useful autoencoder, this is not guaranteed: 
	the autoencoding distribution \eqref{eq:autoencoding_distribution} may be very far from a Dirac distribution.
	For example, when $\posteriormeas^{\decoderparam} \approx \latentmeas$, the optimal decoder distribution $\decodermeas^{\decoderparam}$ may well ignore the latent variable $z$ entirely.
	Indeed, if $\datameas = \varmeas = \decodermeas = \latentmeas = N(0, 1)$ on $\dataspace = \Reals$, then $\KLdiv{\encoderjointmeas}{\decoderjointmeas} = 0$, but the autoencoding distribution $\autoencodermeas_{u} = N(0, 1)$ for all $u \in \dataspace$.
    This is not close to the desired Dirac distribution referred to in Remark \ref{rem:dirac}.
	This issue is known in the VAE literature as \defterm{posterior collapse} \citep{WangBleiCunningham2021}. In practice, choice of
    model classes for $\varmeas$ and $\decodermeas$ avoids this issue.
    \qedremark
\end{remark}

\paragraph{Tractable Training Objective.} 
Since we can access $\datameas$ only through training data, we must decompose the 
objective function into a sum of a term which involves $\datameas$ only through samples and a term which is independent of the parameters $\encoderparam$ and $\decoderparam$. The first of these terms may then be used to define a tractable objective function over the parameters. To address this issue we introduce a \defterm{reference distribution} $\referencemeas \in \prob{\dataspace}$ and impose the following conditions on the encoder and the reference distribution.

\begin{assumption}
	\label[assumption]{assumption:tractable_objective}

    \begin{enumerate}[label=(\alph*)]
        \item 
        \label{item:FVAE_tractable_objective_reference_distribution}

        There exists a reference distribution $\referencemeas \in \prob{\dataspace}$ such that:
	\begin{enumerate}[label=(\roman*)]
		\item
		\label{item:FVAE_tractable_objective_decoder_equivalence}
		for all $\decoderparam \in \Decoderparam$ and $z \in \latentspace$, $\decodermeas^{\decoderparam}$ is mutually absolutely continuous with $\referencemeas$; and

		\item 
		\label{item:FVAE_tractable_objective_finite-information_condition}
		the data distribution $\datameas$ satisfies the finite-information condition $\KLdiv{\datameas}{\referencemeas} < \infty$.
	\end{enumerate}
    
        \item 
        \label{item:FVAE_tractable_objective_encoder_qualification}
        For all $\encoderparam \in \Encoderparam$ and $\datameas$-almost all $u \in \dataspace$,
$\KLdiv{\varmeas^{\encoderparam}}{\latentmeas} < \infty$. \qedremark
\end{enumerate}
   
\end{assumption}

\begin{theorem}
	\label{thm:FVAE_tractable_objective}
    If \Cref{assumption:tractable_objective} is satisfied for some $\referencemeas \in \prob{\dataspace}$,
	then for all parameters $\encoderparam \in \Encoderparam$ and $\decoderparam \in \Decoderparam$ for which $\KLdiv{\encoderjointmeas^{\encoderparam}}{\decoderjointmeas^{\decoderparam}} < \infty$, we have
	\begin{subequations}
	\begin{align}
		\bigKLdiv{\encoderjointmeas^{\encoderparam}}{\decoderjointmeas^{\decoderparam}} &= \EE_{u \sim \datameas}\bigl[ \persampleloss(u; \encoderparam, \decoderparam) \bigr] + \KLdiv{\datameas}{\referencemeas}, \label{eq:FVAE_tractable_objective}\\
		\text{(per-sample loss)}~~~~\persampleloss(u; \encoderparam, \decoderparam) &=
		\EE_{z \sim \varmeas^{\encoderparam}}\Biggl[  -\log \frac{\rd \decodermeas^{\decoderparam}}{\rd \referencemeas}(u) \Biggr] + \bigKLdiv{\varmeas^{\encoderparam}}{\latentmeas} \label{eq:per-sample_loss}.
	\end{align}
	\end{subequations}
\end{theorem}
\begin{proof}
	Write $\encoderjointmeas^{\encoderparam}(\rd z, \rd u) = \encodermeas^{\encoderparam}(\rd z) \datameas(\rd u)$ and $\decoderjointmeas^{\decoderparam}(\rd z, \rd u) = \decodermeas^{\decoderparam}(\rd u) \latentmeas(\rd z)$, factor through the distribution $\referencemeas$ in \eqref{eq:joint_Kullback--Leibler_divergence}, and apply the definition of the KL divergence to obtain
	\begin{align*}
		\bigKLdiv{\encoderjointmeas^{\encoderparam}}{\decoderjointmeas^{\decoderparam}} &= \EE_{(z, u) \sim \encoderjointmeas^{\encoderparam}} \Biggl[\log \frac{\rd \varmeas^{\encoderparam}}{\rd \latentmeas}(z)  \frac{\rd \referencemeas}{\rd \decodermeas^{\decoderparam}}(u) \frac{\rd \datameas}{\rd \referencemeas}(u) \Biggr] \\
		&= \EE_{u \sim \datameas} \Biggl[\EE_{z \sim \varmeas^{\encoderparam}} \Biggl[ - \log \frac{\rd \decodermeas^{\decoderparam}}{\rd \referencemeas}(u) + \log \frac{\rd \varmeas^{\encoderparam}}{\rd \latentmeas}(z) \Biggr] + \log \frac{\rd \datameas}{\rd \referencemeas}(u)  \Biggr] \\
		&= \EE_{u \sim \datameas} \Biggl[ \EE_{z \sim \varmeas^{\encoderparam}} \Biggl[-\log \frac{\rd \decodermeas^{\decoderparam}}{\rd \referencemeas}(u) \Biggr] + \bigKLdiv{\varmeas^{\encoderparam}}{\latentmeas} \Biggr]  + \KLdiv{\datameas}{\referencemeas},
    \end{align*} 
    where all terms are finite as a consequence of Assumptions~\ref{assumption:tractable_objective}\ref{item:FVAE_tractable_objective_reference_distribution}\ref{item:FVAE_tractable_objective_finite-information_condition} and \ref{item:FVAE_tractable_objective_encoder_qualification}.
\end{proof}

Since $\KLdiv{\datameas}{\referencemeas}$ is assumed to be finite and depends on neither $\encoderparam$ nor $\decoderparam$, \Cref{thm:FVAE_tractable_objective} shows that the joint KL divergence \eqref{eq:joint_Kullback--Leibler_divergence} is equivalent, up to a finite constant, to
\begin{equation}
    \label{eq:FVAE_objective}
	\text{(FVAE objective)}~~~~~~~~~~~~~~\objectiveFVAE(\encoderparam, \decoderparam) = \EE_{u \sim \datameas} \bigl[ \persampleloss(u; \encoderparam, \decoderparam) \bigr].
\end{equation}
In particular, minimising $\objectiveFVAE$ is equivalent to minimising the joint divergence, and $\objectiveFVAE(\encoderparam, \decoderparam)$ is finite if and only if $\KLdiv{\encoderjointmeas^{\encoderparam}}{\decoderjointmeas^{\decoderparam}}$ is finite.
Moreover \eqref{eq:FVAE_objective} can be approximated using samples from $\datameas$:
\begin{equation} \label{eq:empiricalised_FVAE_objective}
	\text{(empirical FVAE objective)}~~~~\objectiveFVAE_{N}(\encoderparam, \decoderparam) = \frac{1}{N} \sum_{n = 1}^{N} \persampleloss\bigl(u^{(n)}; \encoderparam, \decoderparam\bigr).
\end{equation}

In the limit of infinite data, the empirical objective \eqref{eq:empiricalised_FVAE_objective} converges to \eqref{eq:FVAE_objective};
but both \eqref{eq:FVAE_objective} and \eqref{eq:empiricalised_FVAE_objective} may be infinite in many practical settings, as we shall see in the following sections.

\begin{remark}
	\begin{enumerate}[label=(\alph*)]
		\item 
		\Cref{assumption:tractable_objective}\ref{item:FVAE_tractable_objective_reference_distribution} ensures that the density $\rd \decodermeas^{\decoderparam}/\rd \referencemeas$ in the per-sample loss exists, and that minimising \eqref{eq:joint_Kullback--Leibler_divergence} and \eqref{eq:FVAE_objective} is equivalent;
		\Cref{assumption:tractable_objective}\ref{item:FVAE_tractable_objective_encoder_qualification} ensures that, when the joint divergence \eqref{eq:joint_Kullback--Leibler_divergence} is finite, the per-sample loss $\persampleloss$ is finite for $\datameas$-almost all $u \in \dataspace$.
		We could also formulate \Cref{assumption:tractable_objective}\ref{item:FVAE_tractable_objective_reference_distribution} with a $\sigma$-finite reference measure $\referencemeas$, e.g., Lebesgue measure, but for our theory it suffices to consider probability measures.
		
		\item
		The proof of \Cref{thm:FVAE_tractable_objective} shows why we take $\encoderjointmeas^{\encoderparam}$ as the first argument and $\decoderjointmeas^{\decoderparam}$ as the second in \eqref{eq:joint_Kullback--Leibler_divergence} to obtain a tractable objective.
		Reversing the arguments in the divergence gives an expectation with respect to the joint decoder model \eqref{eq:decoder_model}:
		\begin{equation*} 
			\bigKLdiv{\decoderjointmeas^{\decoderparam}}{\encoderjointmeas^{\encoderparam}} = \EE_{(z, u) \sim \decoderjointmeas^{\decoderparam}} \Biggl[\log \frac{\rd \latentmeas}{\rd \varmeas^{\encoderparam}}(z) + \log \frac{\rd \decodermeas^{\decoderparam}}{\rd \referencemeas}(u) \Biggr] + \EE_{(z, u) \sim \decoderjointmeas^{\decoderparam}} \Biggl[ \log \frac{\rd \referencemeas}{\rd \datameas}(u) \Biggr].
		\end{equation*}
        Unlike in \Cref{thm:FVAE_tractable_objective}, the term involving $\rd \referencemeas / \rd \datameas$ is not a constant: 
        it depends on the parameter $\decoderparam$ and would therefore need to be evaluated during the optimisation process, but this is intractable because we have only samples from $\datameas$ and not its density. \qedremark
	\end{enumerate}
\end{remark}

\subsection{Objective in Finite Dimensions}
\label{subsec:VAE_in_finite_dimensions}

We now show that the theory in \cref{subsec:VAE_objective} simplifies in finite dimensions to the usual VAE objective.
To do so we assume $\dataspace = \Reals^{k}$ and that $\datameas \in \prob{\dataspace}$ has strictly positive probability density $\datadensity \colon \dataspace \to (0, \infty)$.
We moreover assume that $\latentspace = \Reals^{d_{\latentspace}}$ and that the latent distribution and the distributions returned by the encoder \eqref{eq:FVAE_encoder} and decoder \eqref{eq:FVAE_decoder} are Gaussian, taking the form 
\begin{subequations}
\begin{align}
	\varmeas^{\encoderparam} &= N\bigl(\encodermean(u; \encoderparam), \encodercov(u; \encoderparam) \bigr) = \encodermean(u; \encoderparam) + \encodercov(u; \encoderparam)^{\frac12} N(0, I_{\latentspace}) \label{eq:encoder_finite_dimensions}, \\
	\decodermeas^{\decoderparam} &= N\bigl( \decodermap(z; \decoderparam), \beta I_{\dataspace} \bigr)  = \decodermap(z; \decoderparam) + \beta^{\frac12} N\bigl(0, I_{\dataspace}\bigr) \label{eq:decoder_finite_dimensions}, \\
	\mathbb{P}_{z}  &= N\bigl(0, I_{\latentspace}\bigr), \label{eq:latent_distribution_finite_dimensions}
\end{align}
\end{subequations}
where $\beta > 0$ is fixed and the parameters of \eqref{eq:encoder_finite_dimensions} and \eqref{eq:decoder_finite_dimensions} are given by learnable maps
\begin{subequations}
\begin{align}
	\encodermap = (\encodermean, \encodercov) &\colon \dataspace \times \Encoderparam \to \latentspace \times \spd{\latentspace} \label{eq:FVAE_encoder_map}, \\
	\decodermap &\colon \latentspace \times \Decoderparam \to \dataspace, \label{eq:FVAE_decoder_map}
\end{align}
\end{subequations}
with $\spd{\latentspace}$ denoting the set of positive semidefinite matrices on $\latentspace$.
The Gaussian model \eqref{eq:encoder_finite_dimensions}--\eqref{eq:latent_distribution_finite_dimensions} is the standard setting in which VAEs are applied, resulting in the joint decoder model \eqref{eq:decoder_model} being the distribution of $(z, u)$ in the model
\begin{equation}
    \label{eq:VAE_decoder_model}
	u \mid z = \decodermap(z; \decoderparam) + \eta,\qquad
	z \sim N\bigl(0, I_{\latentspace}\bigr),\qquad
	\eta \sim \decodernoisemeas = \beta^{\frac12} N\bigl(0, I_{\dataspace}\bigr).
\end{equation}
Other decoder models are also possible \citep{KingmaWelling2014}, and in infinite dimensions we will consider a wide class of decoders, including Gaussians as a particular case.
In practice \eqref{eq:FVAE_encoder_map}, \eqref{eq:FVAE_decoder_map} will come from a parametrised class of functions, e.g., a class of neural networks.
Provided these classes are sufficiently large and the data distribution has finite information with respect to $\decodernoisemeas$, the joint divergence \eqref{eq:joint_Kullback--Leibler_divergence} is finite for at least one choice of $\encoderparam$ and $\decoderparam$.

\begin{proposition}
	\label[proposition]{prop:joint_divergence_finite_infimum_finite_dimensions}
	Suppose that for parameters $\encoderparam^{\star} \in \Encoderparam$ and $\decoderparam^{\star} \in \Decoderparam$ we have $\encodermap(u; \encoderparam^{\star}) = (0, I_{\latentspace})$ and $\decodermap(z; \decoderparam^{\star}) = 0$.
	Then
	\begin{equation*} 
		\bigKLdiv{\encoderjointmeas^{\encoderparam^{\star}}}{\decoderjointmeas^{\decoderparam^\star}} = \KLdiv{\datameas}{\decodernoisemeas}.
	\end{equation*}
	In particular, if $\KLdiv{\datameas}{\decodernoisemeas} < \infty$, then the joint divergence \eqref{eq:joint_Kullback--Leibler_divergence} has finite infimum.
\end{proposition}

\begin{proofnosquare}
	Evaluating the joint encoder model \eqref{eq:encoder_model} and the joint decoder model \eqref{eq:decoder_model} at $\encoderparam^{\star}$ and $\decoderparam^{\star}$ gives $\encoderjointmeas^{\encoderparam^{\star}}(\rd z, \rd u) = \latentmeas(\rd z) \datameas(\rd u)$ and  $\decoderjointmeas^{\decoderparam^{\star}}(\rd z, \rd u) = \latentmeas(\rd z) \decodernoisemeas(\rd u)$, so
	\begin{equation*}
		\bigKLdiv{\encoderjointmeas^{\encoderparam^{\star}}}{\decoderjointmeas^{\decoderparam^{\star}}} = \EE_{u \sim \datameas} \EE_{z \sim \latentmeas} \left[\log \frac{\rd \latentmeas}{\rd \latentmeas}(z) + \log \frac{\rd \datameas}{\rd \decodernoisemeas}(u) \right] = \KLdiv{\datameas}{\decodernoisemeas}.~~~~~\BlackBox
	\end{equation*}
\end{proofnosquare}

\begin{remark}
In finite dimensions, many data distributions satisfy $\KLdiv{\datameas}{\decodernoisemeas} < \infty$. However there are cases in which this fails, e.g., when $\datameas$ is a Cauchy distribution on $\Reals$,
In such cases \eqref{eq:joint_Kullback--Leibler_divergence} is infinite even with the trivial maps $\encodermap(u; \encoderparam^{\star}) = (0, I_{\latentspace})$ and $\decodermap(z; \decoderparam^{\star}) = 0$.
\qedremark
\end{remark}

\paragraph{Training Objective.}
In \cref{subsec:VAE_objective} we proved that the joint divergence \eqref{eq:joint_Kullback--Leibler_divergence} can be approximated by the tractable objective \eqref{eq:FVAE_objective} and its empiricalisation \eqref{eq:empiricalised_FVAE_objective} whenever there is a reference distribution $\referencemeas$ satisfying \Cref{assumption:tractable_objective}.
We now show that taking $\referencemeas = \datameas$ satisfies this assumption and results in a training objective equivalent to maximising the ELBO. 
Recall the data density $\datadensity$ associated with data measure $\datameas.$

\begin{proposition}
    Under the model \eqref{eq:encoder_finite_dimensions}--\eqref{eq:latent_distribution_finite_dimensions} with reference distribution $\referencemeas = \datameas$, \Cref{assumption:tractable_objective} is satisfied, and, for some finite constant $C > 0$ independent of $u$, $\encoderparam$, and $\decoderparam$,
    \begin{subequations}
\begin{align}
	\persampleloss(u; \encoderparam, \decoderparam) &= \EE_{z \sim \varmeas^{\encoderparam}} \bigl[-\log p_{u \mid z}^{\decoderparam}(u)\bigr] + \log \datadensity(u) + \bigKLdiv{\varmeas^{\encoderparam}}{\latentmeas} \label{eq:per-sample_loss_finite_dims}
  \\&= \EE_{z \sim \varmeas^{\encoderparam}} \left[ (2\beta)^{-1} \bigl\| \decodermap(z; \decoderparam) - u \bigr\|^{2}_{2} \right] + \log \datadensity(u) + \bigKLdiv{\varmeas^{\encoderparam}}{\latentmeas} + C.
\end{align}
\end{subequations}
\end{proposition}

\begin{proof}
For \Cref{assumption:tractable_objective}\ref{item:FVAE_tractable_objective_reference_distribution}, 
mutual absolute continuity of $\decodermeas^{\decoderparam}$ and $\referencemeas$ follows as both distributions have strictly positive densities, and evidently
$\KLdiv{\datameas}{\referencemeas} = 0$. \Cref{assumption:tractable_objective}\ref{item:FVAE_tractable_objective_encoder_qualification} holds since both the encoder distribution \eqref{eq:encoder_finite_dimensions} and the latent distribution \eqref{eq:latent_distribution_finite_dimensions} are Gaussian, with KL divergence available in closed form (\Cref{rk:VAE_regularised_autoencoder}).
The expression for $\persampleloss$ follows from \eqref{eq:per-sample_loss} using that $\decodermeas^{\decoderparam}$ is Gaussian with density $p_{u \mid z}^{\decoderparam}$ and $\rd \decodermeas^{\decoderparam} / \rd \datameas(u) = p_{u \mid z}^{\decoderparam}(u) / \datadensity(u)$.
\end{proof}

While the per-sample loss \eqref{eq:per-sample_loss_finite_dims} involves the unknown density $\datadensity$, we can drop this without affecting the objective \eqref{eq:FVAE_objective} if $\datameas$ has finite \defterm{differential entropy} $\EE_{u \sim \datameas} \bigl[-\log \datadensity(u)\bigr]$. 
This follows because
\begin{subequations}
	\begin{align}
		\objectiveFVAE(\encoderparam, \decoderparam) &= \EE_{u \sim \datameas} \bigl[\persampleloss^{\text{VAE}}(u; \encoderparam, \decoderparam) \bigr] - \EE_{u \sim \datameas} \bigl[-\log \datadensity(u) \bigr] + C, \\
		\persampleloss^{\text{VAE}}(u; \encoderparam, \decoderparam) &= \EE_{z \sim \varmeas^{\encoderparam}} \Bigl[ (2\beta)^{-1} \bigl\| \decodermap(z; \decoderparam) - u \bigr\|^{2}_{2}  \Bigr] + \bigKLdiv{\varmeas^{\encoderparam}}{\latentmeas}. \label{eq:per-sample_loss_VAE}
	\end{align}
\end{subequations}
Thus, $\objective^{\text{VAE}}(\encoderparam, \decoderparam) = \EE_{u \sim \datameas} [\persampleloss^{\text{VAE}}(\encoderparam, \decoderparam)]$ is equivalent, up to a finite constant, to $\objectiveFVAE(\encoderparam, \decoderparam)$, and $\objective^{\text{VAE}}$ is tractable.
Requiring that $\datameas$ has finite differential entropy is a mild condition, and one can expect this to be the case for the vast majority of distributions arising in the finite-dimensional setting.

\begin{remark}[VAEs as regularised autoencoders]
    \label[remark]{rk:VAE_regularised_autoencoder}
	We can write the divergence in \eqref{eq:per-sample_loss_VAE} in closed form as
    $\KLdiv{\varmeas^{\encoderparam}}{\latentmeas} = \frac{1}{2} \bigl( \norm{\encodermean(u; \encoderparam)}^{2}_{2} + \tr\bigl(\encodercov(u; \encoderparam) - \log \encodercov(u; \encoderparam) \bigr) - d_{\latentspace} \bigr)$.
	Applying the reparametrisation trick \citep{KingmaWelling2014} to write the expectation over $z \sim \varmeas^{\encoderparam}$ in terms of $\xi \sim N(0, I_{\latentspace})$, the interpretation of a VAE as a regularised autoencoder is clear:
	\begin{align*}
		\persampleloss^{\text{VAE}}(u; \encoderparam, \decoderparam) &\propto \EE_{\xi \sim N(0, I_{\latentspace})} \biggl[ (2\beta)^{-1} \Bigl\| \decodermap\Bigl(\encodermean(u;\encoderparam) + \encodercov(u; \encoderparam)\xi ;\decoderparam\Bigr) - u\Bigr\|^{2}_{2}   \biggr] \\
        &+ \tfrac{1}{2} \Norm{\encodermean(u; \encoderparam)}^{2}_{2}  +  \tfrac{1}{2} \tr \bigl(\encodercov(u; \encoderparam) - \log \encodercov(u; \encoderparam)\bigr).~~~~~~~~~~~~~~~~~~~~~~~~~~~~~~~~\blacksquare
	\end{align*}
\end{remark}

\begin{remark}[The evidence lower bound]
	The usual derivation of VAEs specifies the decoder model \eqref{eq:VAE_decoder_model} and performs variational inference on the posterior for $z \mid u$, seeking to maximise a lower bound, the ELBO, on the likelihood of the data.
	Denoting by $p_{u}^{\decoderparam}$ the density of the generative model $\genmeas^{\decoderparam}$, by $q_{z \mid u}^{\encoderparam}$ the density of $\varmeas^{\encoderparam}$, and so forth, the log-likelihood of $u \in \dataspace$ is 
	\begin{subequations}
	\begin{align}
		\label{eq:VAE_ELBO_equality}
		\log p_{u}^{\decoderparam}(u) &= \bigKLdiv{q_{z \mid u}^{\encoderparam}}{p_{z \mid u}^{\decoderparam}} + \ELBO(u; \encoderparam, \decoderparam),\\
		\ELBO(u; \encoderparam, \decoderparam) &= \EE_{z \sim q_{z \mid u}^{\encoderparam}} \bigl[\log p_{u \mid z}^{\decoderparam}(u) \bigr] - \bigKLdiv{q_{z \mid u}^{\encoderparam}}{p_{z}}.
		\label{eq:VAE_ELBO}
	\end{align}
	\end{subequations}
	This demonstrates that $\ELBO(u; \encoderparam, \decoderparam)$ is indeed a lower bound on the log-data likelihood under the decoder model.
	Minimising our per-sample loss \eqref{eq:per-sample_loss_finite_dims} is equivalent to maximising the ELBO \eqref{eq:VAE_ELBO}, but, notably, our underlying objective is shown to correspond exactly to the underlying KL divergence \eqref{eq:joint_Kullback--Leibler_divergence} and avoids the use of any bounds on data likelihood. 
    \qedremark
\end{remark}

\begin{remark}
An interpretation of the VAE objective as minimisation of \eqref{eq:joint_Kullback--Leibler_divergence} is also adopted by \citet[Sec.~2.8]{Kingma2017} and \citet{KingmaWelling2019}.
Our approach differs by writing the joint decoder model \eqref{eq:decoder_model} in terms of $\datameas$ rather than the empirical distribution
\begin{equation*}
	\text{(empirical data distribution)~~~~~~} \datameas_{N} = \frac{1}{N} \sum_{n = 1}^{N} \dirac{u^{(n)}}.
\end{equation*}
Since 
$\datameas_{N}$ is not absolutely continuous with respect to the generative model $\genmeas^{\decoderparam}$ of \eqref{eq:encoder_finite_dimensions}--\eqref{eq:latent_distribution_finite_dimensions}, using this would result in the joint divergence \eqref{eq:joint_Kullback--Leibler_divergence} being infinite for all $\encoderparam$ and $\decoderparam$. \qedremark
\end{remark}

\subsection{Objective in Infinite Dimensions}
\label{subsec:VAE_in_infinite_dimensions}

We return to the setting of \Cref{assumption:FVAE_assumptions} and adopt a generalisation of the Gaussian model \eqref{eq:encoder_finite_dimensions}--\eqref{eq:latent_distribution_finite_dimensions}, with distributional parameters given by learnable maps $\encodermap = (\encodermean, \encodercov)\colon \dataspace \times \Encoderparam \to \latentspace \times \spd{\latentspace}$ and $\decodermap \colon \latentspace \times \Decoderparam \to \dataspace$:
\begin{subequations}
\begin{align}
	\varmeas^{\encoderparam} &= N\bigl(\encodermean(u; \encoderparam), \encodercov(u; \encoderparam)\bigr) = \encodermean(u; \encoderparam) + \encodercov(u; \encoderparam)^{\frac12} N(0, I_{\latentspace}), \label{eq:encoder_infinite_dimensions} \\
	\decodermeas^{\decoderparam} &= \decodernoisemeas\bigl(\quark - \decodermap(z; \decoderparam) \bigr) = \decodermap(z; \decoderparam) + \decodernoisemeas,  \label{eq:decoder_infinite_dimensions_1} \\
	\decodernoisemeas &\in \prob{\dataspace}, \label{eq:decoder_infinite_dimensions_2}  \\
	\mathbb{P}_{z} &= N\bigl(0, I_{\latentspace}\bigr).\label{eq:latent_distribution_infinite_dimensions} 
\end{align}
\end{subequations}
The only change from the finite-dimensional model is in the decoder distribution \eqref{eq:decoder_infinite_dimensions_1}, which we now write as the shift of the \defterm{decoder-noise distribution} $\decodernoisemeas$ by the mean $\decodermap(z; \decoderparam)$.
As we will see shortly, the choice of decoder noise will be very important in infinite dimensions, and restricting attention solely to Gaussian white noise will no longer be feasible.

\paragraph{Obstacles in Infinite Dimensions.}
Many new issues arise in infinite dimensions, necessitating a more careful treatment of the FVAE objective;
for example, we can no longer work with probability density functions, since there is no uniform measure analogous to Lebesgue measure \citep{Sudakov1959}.
The fundamental obstacle, however, is that---unlike in finite dimensions---the joint divergence \eqref{eq:joint_Kullback--Leibler_divergence} is often ill defined, satisfying
\begin{equation*}
	\bigKLdiv{\encoderjointmeas^{\encoderparam}}{\decoderjointmeas^{\decoderparam}} = \infty \text{~~~for all $\encoderparam \in \Encoderparam$ and $\decoderparam \in \Decoderparam$.}
\end{equation*}
An important situation in which this arises is misspecification of the generative model $\genmeas^{\decoderparam}$;
consequently, great care is needed in the choice of decoder to avoid this issue.
To illustrate this we show that the extension of the white-noise model of \eqref{eq:encoder_finite_dimensions}--\eqref{eq:latent_distribution_finite_dimensions} is ill-defined in infinite dimensions:
the resulting joint divergence \eqref{eq:joint_Kullback--Leibler_divergence} is always infinite.
To do this we define white noise using a Karhunen--Lo\`eve expansion \citep[Sec.~11.1]{Sullivan2015}.

\begin{definition}
    Let $\dataspace = L^{2}([0, 1])$ with orthonormal basis $e_{j}(x) = \sqrt{2} \sin(\pi j x)$, $j \in \Naturals$.
    We say that the random variable $\eta$ is \defterm{$L^{2}$-white noise} if it has Karhunen--Lo\`eve expansion 
    \begin{equation*}
        \text{($L^{2}$-white noise)}~~~~~~~~~~~~~~~~~~~\eta = \sum_{j \in \Naturals} \xi_{j} e_{j},\qquad \xi_{j} \overset{\text{i.i.d.}}{\sim} N(0, 1).~~~~~~~~~~~~~~~~~~~~~~~\blacksquare
    \end{equation*}
\end{definition}

\begin{proposition}
    \label[proposition]{prop:white_noise}
    Let $\decodernoisemeas$ be the distribution of the $L^{2}$-white noise $\eta$.
    Then $\decodernoisemeas$ is a probability distribution supported on the Sobolev space $H^{s}([0, 1])$ if and only if $s < -\nicefrac{1}{2}$; in particular
    $\decodernoisemeas$ assigns probability zero to $L^{2}([0, 1])$.
    Moreover, these statements remain true for any shift $\decodernoisemeas(\quark - h)$ of the distribution $\decodernoisemeas$ by $h \in L^{2}([0, 1])$.
\end{proposition}

The proof, which makes use of the Borel--Cantelli lemma and the Kolmogorov two-series theorem, is stated in \cref{sec:supporting_results}.

\begin{example}
\label[example]{ex:white-noise_on_L2}

Let $\dataspace = L^{2}([0, 1])$, fix a data distribution $\datameas \in \prob{\dataspace}$ and take the model \eqref{eq:encoder_infinite_dimensions}--\eqref{eq:latent_distribution_infinite_dimensions}, 
where we further assume that $\decodermap(z; \decoderparam) \in L^{2}([0, 1])$ for all $z \in \latentspace$ and $\decoderparam \in \Decoderparam$, and with
$\decodernoisemeas$ taken to be the distribution of $L^{2}$-white noise.
Under this model, the joint divergence \eqref{eq:joint_Kullback--Leibler_divergence} is infinite for all $\encoderparam$ and $\decoderparam$.
To see this, note that $\datameas$ assigns probability one to $\dataspace$, but, as $\decodermeas^{\decoderparam}$ assigns zero probability to $\dataspace$ by \Cref{prop:white_noise}, we have
\begin{equation*}
	\genmeas^{\decoderparam}(\dataspace) = \int_{\latentspace} \decodermeas^{\decoderparam}(\dataspace) \,\latentmeas(\rd z) = 0.
\end{equation*}
Thus $\datameas$ is not absolutely continuous with respect to $\genmeas^{\decoderparam}$, and so $\KLdiv{\encoderjointmeas^{\encoderparam}}{\decoderjointmeas^{\decoderparam}} = \infty$.
\qedremark
\end{example}

The VANO model \citep{SeidmanKissasPappasPerdikaris2023}, which we discuss in detail in the related work  (\cref{sec:related_work}), adopts the setting of \eqref{eq:encoder_infinite_dimensions}--\eqref{eq:latent_distribution_infinite_dimensions} with white decoder noise. It thus
suffers from not being well-defined.
The issues in \Cref{ex:white-noise_on_L2} stem from the difference in regularity between the data, which lies in $L^2([0,1])$, and draws from the generative model, which lie in $H^{s}([0, 1])$, $s < -\nicefrac{1}{2}$, and not in $L^2([0,1])$.
A difference in regularity is not the only possible issue:
even two Gaussians supported on the same space need not be absolutely continuous owing to the Feldman--H\'ajek theorem \citep[Ex.~2.7.4]{Bogachev1998}.
But, as in \Cref{prop:joint_divergence_finite_infimum_finite_dimensions}, we can state a sufficient condition for the divergence \eqref{eq:joint_Kullback--Leibler_divergence} to have finite infimum.

\begin{proposition}
	\label[proposition]{prop:joint_divergence_finite_infimum_infinite_dimensions}
	Suppose that $\encodermap(u; \encoderparam^{\star}) = (0, I_{\latentspace})$ and $\decodermap(z; \decoderparam^{\star}) = 0$ for some $\encoderparam^{\star} \in \Encoderparam$ and $\decoderparam^{\star} \in \Decoderparam$, and that $\KLdiv{\datameas}{\decodernoisemeas} < \infty$.
	Then \eqref{eq:joint_Kullback--Leibler_divergence} has finite infimum.
\end{proposition}

\paragraph{Well-Defined Objective for Specific Problem Classes.}
The issues we have seen suggest that one must choose the decoder model based on the structure of the data distribution:
fixing a decoder model \emph{a priori} typically results, in infinite dimensions, in the joint divergence being infinite. However we now give examples to show that, in important classes of problems arising in science and engineering, there is a clear choice of decoder noise $\decodernoisemeas$ arising from the problem structure.
Adopting this decoder noise and taking the reference distribution $\referencemeas = \decodernoisemeas$ will ensure that the hypotheses of \Cref{thm:FVAE_tractable_objective} are satisfied: 
the joint divergence can be shown to have finite infimum, and \Cref{assumption:tractable_objective} holds.
Thus we can apply the actionable algorithms derived in \Cref{subsec:VAE_objective}.

\subsubsection{SDE Path Distributions}
\label{subsubsec:SDE}

One class of data to which we can apply FVAE arises in the study of random dynamical systems \citep{ERenVandenEijnden2004}.
We choose $\datameas$ to be the distribution over paths of the diffusion process defined, for a standard Brownian motion $(w_{t})_{t \in [0, T]}$ on $\Reals^{m}$
and $\varepsilon > 0$, by
\begin{equation}
	\label{eq:SDE}
	\rd u_{t} = b(u_{t}) \,\rd t + \sqrt{\varepsilon} \,\rd w_{t},\qquad u_{0} = 0,\qquad t \in [0, T].
\end{equation}
We assume that the drift $b \colon \Reals^{m}\to \Reals^{m}$ is regular enough that \eqref{eq:SDE} is well defined.
The theory we outline also applies to systems with anisotropic diffusion and time-dependent coefficients \citep[Sec.~7.3]{SaerkkaeSolin2019}, and to systems with nonzero initial condition, but we focus on the setting \eqref{eq:SDE} for simplicity.
The path distribution $\datameas$ is defined on the space $\dataspace = C_{0}([0, T], \Reals^{m})$ of continuous functions $u\colon [0, T] \to \Reals^{m}$ with $u(0) = 0$. 

Recall \eqref{eq:decoder_infinite_dimensions_1} where we define the
decoder distribution $\decodermeas^{\decoderparam}$ as the shift of the noise distribution $\decodernoisemeas$ by $\decodermap(z; \decoderparam)$, and recall \cref{assumption:tractable_objective}, which in particular demands that $\KLdiv{\datameas}{\decodernoisemeas} < \infty$  and that $\decodermeas^{\decoderparam}$ is mutually absolutely continuous with $\decodernoisemeas$.
Here we choose $\decodernoisemeas$ to be the law of an auxiliary diffusion process, which in our examples will be an Ornstein--Uhlenbeck (OU) process.
This auxiliary process must have zero initial condition and must have the
same noise structure as that in \eqref{eq:SDE} to ensure that $\KLdiv{\datameas}{\decodernoisemeas} < \infty$. 
We will learn the shift $\decodermap$ and this will have to satisfy the zero initial condition to ensure that
$\decodermeas^{\decoderparam}$ is mutually absolutely continuous with $\decodernoisemeas$.

\paragraph{Decoder-Noise Distribution $\decodernoisemeas$.} 
We take $\decodernoisemeas$ to be the law of the auxiliary SDE
\begin{equation} \label{eq:modified_SDE}
	\rd \eta_{t} = c(\eta_{t}) \,\rd t + \sqrt{\varepsilon} \,\rd w_{t},\qquad \eta_{0} = 0,\qquad t \in [0, T],
\end{equation}
with drift $c \colon \Reals^{m} \to \Reals^{m}$, and with $\varepsilon > 0$ being the same in both \eqref{eq:SDE} and \eqref{eq:modified_SDE}.
To determine whether $\KLdiv{\datameas}{\decodernoisemeas} < \infty$, we will need to evaluate the term $\rd \datameas / \rd \decodernoisemeas$ in the KL divergence; 
an expression for this is given by the Girsanov theorem \citep[Chap.~7]{LiptserShiryaev2001}.

\begin{proposition}[Girsanov theorem]
    \label[proposition]{prop:Girsanov}
	Suppose that $\dataspace = C_{0}([0, T], \Reals^{m})$ and that $\mu \in \prob{\dataspace}$ and $\nu \in \prob{\dataspace}$ are the laws of the $\Reals^{m}$-valued diffusions
	\begin{equation*}
    \begin{alignedat}{3}
		\rd u_{t} &= p(u_{t}) \,\rd t + \sqrt{\varepsilon} \,\rd w_{t},\quad &&u_{0} =0,\quad &&t \in [0, T],\\
		\rd v_{t} &= q(v_{t}) \,\rd t + \sqrt{\varepsilon} \,\rd w_{t},\quad &&v_{0} = 0,\quad\, &&t \in [0, T].
        \end{alignedat}
	\end{equation*}
	Suppose that the Novikov condition \citep[eq.~(8.6.8)]{Oksendal2003} holds for both processes:
	\begin{equation} \label{eq:Novikov_condition}
		\EE_{u \sim \mu} \left[\int_{0}^{T} \norm{p(u_{t})}_{2}^{2} \,\rd t\right] < \infty \text{~~~and~~~}
		\EE_{v \sim \nu} \left[\int_{0}^{T} \norm{q(v_{t})}_{2}^{2} \,\rd t\right] < \infty.
	\end{equation}
	Then
	\begin{equation} \label{eq:Girsanov_density}
		\frac{\rd \mu}{\rd \nu}(u) = \exp\left(\frac{1}{2\varepsilon} \int_{0}^{T} \norm{q(u_{t})}_{2}^{2} - \norm{p(u_{t})}_{2}^{2} \,\rd t - \frac{1}{\varepsilon} \int_{0}^{T} \langle q(u_{t}) - p(u_{t}), \,\rd u_{t} \rangle\right).
	\end{equation}
\end{proposition}

The second integral in \eqref{eq:Girsanov_density} is a \defterm{stochastic integral} with respect to $(u_{t})_{t \in [0, T]}$ \citep[Chap.~4]{SaerkkaeSolin2019}.
The Novikov condition suffices for our needs, but the theorem also holds under weaker conditions such as the Kazamaki condition \citep[p.~249]{LiptserShiryaev2001}.
Applying \Cref{prop:Girsanov} to evaluate the KL divergence (\cref{sec:supporting_results}) yields
\begin{equation*}
	\bigKLdiv{\datameas}{\decodernoisemeas} = \EE_{u \sim \datameas} \left[\frac{1}{2\varepsilon} \int_{0}^{T} \norm{b(u_{t}) - c(u_{t})}^{2}_{2} \,\rd t \right].
\end{equation*}
Thus the condition $\KLdiv{\datameas}{\decodernoisemeas} < \infty$ is satisfied quite broadly, e.g., if $b$ and $c$ are bounded.

\paragraph{Per-Sample Loss.}
With the condition $\KLdiv{\datameas}{\decodernoisemeas} < \infty$ verified, it remains to choose $g$ to ensure that $\decodermeas^{\decoderparam}$ is mutually absolutely continuous with $\decodernoisemeas$, and to derive the corresponding density.
Once this is done, we arrive at the actionable per-sample loss derived in \cref{thm:FVAE_tractable_objective}.
To do this we again apply the Girsanov theorem, using that, when $\decodermap(z; \decoderparam) \in H^{1}([0, T])$ takes value zero at $t = 0$, the distribution $\decodermeas^{\decoderparam}$ is the law of the SDE
\begin{equation} \label{eq:decoder_process}
	\rd v_{t} = \decodermap(z; \decoderparam)'(t) \,\rd t + c\bigl(v_{t} - \decodermap(z; \decoderparam)\bigr) \,\rd t + \sqrt{\varepsilon} \,\rd w_{t},\quad v_{0} = 0,\quad t \in [0, T].
\end{equation}
As we will parametrise $\decodermap(z; \decoderparam)$ using a neural network, we can assume it to be in $C^{1}([0, T])$, and hence in
$H^{1}([0, T])$;
moreover we shall enforce that $\decodermap(z; \decoderparam)(0) \approx 0$ through the loss, as described shortly.
To be concrete, we restrict attention to decoder noise with $c(x) = -\kappa x$, making $(\eta_{t})_{t \in [0, T]}$ an OU process if $\kappa > 0$ and Brownian motion if $\kappa = 0$.
In this case the per-sample loss can be derived by applying the Girsanov theorem to \eqref{eq:modified_SDE} and \eqref{eq:decoder_process}.

\begin{proposition}[SDE per-sample loss]
\label{prop:psl}
    Suppose that $c(x) = -\kappa x$, $\kappa \geq 0$, and that $\decodermap(z; \decoderparam) \in H^{1}([0, T])$ with $\decodermap(z; \decoderparam)(0) = 0$ for all $z \in \latentspace$ and $\decoderparam \in \Decoderparam$.
    Then \cref{assumption:tractable_objective} holds and
    \begin{align*}
        \persampleloss(u; \encoderparam, \decoderparam) &= \EE_{z \sim \varmeas^{\encoderparam}} \biggl[-\log \frac{\rd \decodermeas^{\decoderparam}}{\rd \decodernoisemeas}(u) \biggr] + \KLdiv{\varmeas^{\encoderparam}}{\latentmeas},\\
        	\log \frac{\rd \decodermeas^{\decoderparam}}{\rd \decodernoisemeas}(u) &=
		\frac{1}{\varepsilon} \int_{0}^{T} \bigl\langle \decodermap(z; \decoderparam)'(t) + \kappa \decodermap(z; \decoderparam)(t), \rd u_{t} \bigr\rangle \\
		- \frac{1}{2\varepsilon}&\int_{0}^{T}\Bigl( \Norm{\decodermap(z; \decoderparam)'(t) - \kappa \bigl(u(t) - \decodermap(z; \decoderparam)(t)\bigr)}_{2}^{2} - \norm{\kappa u(t)}_{2}^{2}\Bigr) \,\rd t . 
    \end{align*}
\end{proposition}

In practice, we make two modifications to $\persampleloss$.
First, the initial condition $\decodermap(z; \decoderparam)(0) = 0$ is not enforced exactly;
instead, we add a Tikhonov-like zero-penalty term with regularisation parameter $\lambda > 0$ to favour $\decodermap(z; \decoderparam)(0) \approx 0$.
Second, to allow variation of the strength of the KL regularisation, we multiply the term $\KLdiv{\varmeas^{\encoderparam}}{\latentmeas}$ by a regularisation parameter $\beta > 0$.
Setting $\beta \neq 1$ breaks the exact correspondence between the FVAE objective and the joint KL divergence \eqref{eq:joint_Kullback--Leibler_divergence}, but can nevertheless be useful in computational practice \citep{Higginsetal2017}.
This leads us to the SDE per-sample loss
\begin{equation*} 
	\begin{split}
		\persampleloss^{\text{SDE}}_{\lambda, \beta}(u; \encoderparam, \decoderparam) = \EE_{z \sim \varmeas^{\encoderparam}} \biggl[  -\log \frac{\rd \decodermeas^{\decoderparam}}{\rd \decodernoisemeas}(u) + \lambda \norm{\decodermap(z; \decoderparam)(0)}_{2}^{2} \biggr] + \beta \bigKLdiv{\varmeas^{\encoderparam}}{\latentmeas}.
	\end{split}
\end{equation*}

\subsubsection{Posterior Distributions in Bayesian Inverse Problems}

Our theory can also be applied to to posterior distributions arising in Bayesian inverse problems \citep{Stuart2010}, which we illustrate through the following additive-noise inverse problem. 
Let $\dataspace$ be a separable Hilbert space with norm $\norm{\quark}_{\dataspace}$, let $Y = \Reals^{d_{Y}}$, and let $\mathcal{G} \colon \dataspace \to Y$ be a (possibly nonlinear) observation operator.
Suppose that $y \in Y$ is given by the model
\begin{equation} \label{eq:BIP}
    y = \mathcal{G}(u) + \xi,\qquad u \sim \mu_{0} \in \prob{\dataspace},\qquad \xi \sim N(0, \Sigma) \in \prob{Y}, 
\end{equation}
with noise covariance $\Sigma \in \spd{Y}$, and with prior distribution $\mu_{0} = N(0, C)$ having covariance operator $C \colon \dataspace \to \dataspace$.
Models of this type arise, for example, in both Eulerian and Lagrangian data assimilation problems in oceanography \citep{CotterDashtiStuart2010}.
Given an observation $y \in Y$ from \eqref{eq:BIP}, the Bayesian approach seeks to infer $u \in \dataspace$ by computing the posterior distribution $\mu^{y} \in \prob{\dataspace}$ representing the distribution of $u \mid y$.
In the setting of \eqref{eq:BIP}, $\mu^{y}$ has a density with respect to $\mu_{0}$ thanks to Bayes' rule \citep[Theorem~14]{DashtiStuart2017}, taking the form
\begin{equation*}
    \frac{\rd \mu^{y}}{\rd \mu_{0}}(u) = \frac{1}{Z(y)} \exp\bigl(-\Phi(u; y)\bigr),\quad \Phi(u; y) = \frac{1}{2} \norm{\mathcal{G}(u) - y}_{\Sigma}^{2},\quad \norm{\quark}_{\Sigma} = \norm{\Sigma^{-1/2} \quark}_{2},
\end{equation*}
where $Z(y) \in (0, 1]$ owing to the nonnegativity of $\Phi$.
A simple calculation then reveals 
\begin{equation}
	\label{eq:posterior--prior_KL_divergence}
    \KLdiv{\mu^{y}}{\mu_{0}} = \EE_{u \sim \mu^{y}}\left[\log \frac{\rd \mu^{y}}{\rd \mu_{0}}(u) \right] = \EE_{u \sim \datameas} \bigl[ -\log Z(y) -\Phi(u) \bigr] \leq -\log Z(y) < \infty.
\end{equation}
Similar arguments apply quite generally for observation models other than \eqref{eq:BIP}, provided the resulting log-density $\log \rd \mu^{y} / \rd \mu_{0}$ satisfies suitable boundedness or integrability conditions.

We now assume that the data distribution $\datameas$ to be learned is the posterior $\mu^{y}$, and that we have samples from it.
This setting could arise, for example, when attempting to generate further approximate samples from the posterior $\mu^{y}$, taking as data the output of a function-space MCMC method \citep{CotterRobertsStuartWhite2013}, with the ambition of faster sampling under FVAE than under MCMC.
Recall the definition of the decoder distribution $\decodermeas^{\decoderparam}$ in \eqref{eq:decoder_infinite_dimensions_1}.
We take $\decodernoisemeas$ to be the prior $\mu_{0}$;
this is a natural choice as \eqref{eq:posterior--prior_KL_divergence} shows that $\KLdiv{\datameas}{\decodernoisemeas} < \infty$.
We next discuss the choice of shift $g$, and the per-sample loss that results from these choices.

\paragraph{Per-Sample Loss.}
Since $\decodernoisemeas$ and $\decodermeas^{\decoderparam}$ are Gaussian, we can use the Cameron--Martin theorem \citep[Corollary~2.4.3]{Bogachev1998} to derive conditions for their mutual absolute continuity.
For this to be the case, the shift $\decodermap(z; \decoderparam)$ must lie 
in the \defterm{Cameron--Martin space} $\CMspace{\decodernoisemeas} \subset \dataspace$.
Before stating the theorem, we recall the following facts about Gaussian measures.
The space $\CMspace{\decodernoisemeas}$ is Hilbert, and for fixed $h \in \CMspace{\decodernoisemeas}$, the $\CMspace{\decodernoisemeas}$-inner product $\innerprod{h}{\quark}_{\CMspace{\decodernoisemeas}}$ extends uniquely (up to equivalence $\decodernoisemeas$-almost everywhere) to a \defterm{measurable linear functional} \citep[Theorem~2.10.11]{Bogachev1998}, denoted by $\dataspace \ni u \mapsto \Pwint{h}{u}_{\CMspace{\decodernoisemeas}}$.

\begin{proposition}[Cameron--Martin theorem]
	\label[proposition]{prop:CM_theorem}
	Let $\decodernoisemeas \in \prob{\dataspace}$ be a Gaussian measure with Cameron--Martin space $\CMspace{\decodernoisemeas}$.
	Then $\decodermeas^{\decoderparam} = \decodernoisemeas\bigl(\quark - \decodermap(z; \decoderparam)\bigr)$ is mutually absolutely continuous with $\decodernoisemeas$ if and only if $\decodermap(z; \decoderparam) \in \CMspace{\decodernoisemeas}$, and 
	\begin{equation} \label{eq:Cameron--Martin_formula}
		\frac{\rd \decodermeas^{\decoderparam}}{\rd \decodernoisemeas}(u) = \exp\Bigl( \Pwint{\decodermap(z; \decoderparam)}{u}_{\CMspace{\decodernoisemeas}} - \tfrac{1}{2} \norm{\decodermap(z; \decoderparam)}_{\CMspace{\decodernoisemeas}}^{2} \Bigr).
	\end{equation}
\end{proposition}

\begin{remark}
\label{rk:Cameron--Martin_exponent}

The exponent in \eqref{eq:Cameron--Martin_formula} should be viewed as the misfit $\tfrac{1}{2} \norm{\decodermap(z; \decoderparam) - u}_{\CMspace{\decodernoisemeas}}^{2}$ with the almost-surely-infinite term $\tfrac{1}{2} \norm{u}_{\CMspace{\decodernoisemeas}}^{2}$ subtracted \citep[Remark~3.8]{Stuart2010}.
When $\decodernoisemeas$ is Brownian motion on $\Reals$, for example, $\Pwint{\decodermap(z; \decoderparam)}{u}_{\CMspace{\decodernoisemeas}}$ is a stochastic integral and $\CMspace{\decodernoisemeas} = H^{1}([0, T])$; this is implicit
in the calculations underlying \Cref{prop:psl}.
When $\decodernoisemeas$ is $L^{2}$-white noise, $\CMspace{\decodernoisemeas} = L^{2}([0, 1])$. 
\qedremark
\end{remark}

When $\decodermap$ takes values in $\CMspace{\decodernoisemeas}$, we can use the Cameron--Martin theorem to write down the per-sample loss explicitly since $\CMspace{\decodernoisemeas} = C^{\nicefrac{1}{2}} \dataspace$, with $\norm{h}_{\CMspace{\decodernoisemeas}} = \norm{C^{-\nicefrac{1}{2}} h}_{\dataspace}$ and $\Pwint{h}{u}_{\CMspace{\decodernoisemeas}} = \innerprod{C^{-\nicefrac{1}{2}} h}{C^{-\nicefrac{1}{2}} u}_{\dataspace}$ for $h \in \CMspace{\decodernoisemeas}$ and $u \in \dataspace$.

\begin{proposition}[Bayesian inverse problem per-sample loss]
    Suppose that $\decodermap(z; \decoderparam) \in \CMspace{\decodernoisemeas}$ for all $z \in \latentspace$ and $\decoderparam \in \Decoderparam$. Then \Cref{assumption:tractable_objective} holds and
    \begin{equation*}
	\persampleloss^{\text{BIP}}(u; \encoderparam, \decoderparam) = \EE_{z \sim \varmeas^{\encoderparam}} \left[  \tfrac{1}{2} \bignorm{C^{-\frac{1}{2}} \decodermap(z; \decoderparam)}_{\dataspace}^{2} - \biginnerprod{C^{-\frac{1}{2}} \decodermap(z; \decoderparam)}{C^{-\frac{1}{2}} u}_{\dataspace} \right]  + \bigKLdiv{\varmeas^{\encoderparam}}{\latentmeas}.
    \end{equation*}
\end{proposition}

Depending on the choice of $\decodernoisemeas$, the condition $\decodermap(z; \decoderparam) \in \CMspace{\decodernoisemeas}$ may follow immediately, e.g., when $\decodernoisemeas$ is Brownian motion and $\decodermap$ is parametrised by a neural network.

\subsection{Architecture and Algorithms}
\label{subsec:VAE_architecture}

In practice, we do not have access to training data $\{u^{(n)}\}_{n = 1}^{N} \subset \dataspace$;
instead we have access to finite-dimensional discretisations $\mathbf{u}^{(n)}$.
We would like to evaluate the encoder and decoder, and to compute the empirical objective \eqref{eq:empiricalised_FVAE_objective} for training, using only these discretisations.

In our architectures we will assume that $\dataspace$ is a Banach space of functions evaluable pointwise almost everywhere with domain $\Omega \subseteq \Reals^{d}$ and range $\Reals^{m}$;
in this setting we will assume that the discretisation $\mathbf{u}$ of a function $u \in \dataspace$ consists of evaluations at $\{x_{i}\}_{i = 1}^{I} \subset \Omega$:
\begin{equation*}
    \text{(discretisation of function $u \in \dataspace$)}\qquad \mathbf{u} = \Bigl\{ \Bigl(x_{i}, u(x_{i})\Bigr) \Bigr\}_{i = 1}^{I} \subset \Omega \times \Reals^{m}.
\end{equation*}
Crucially, the number and location of mesh points may differ for each discretised sample---our aim is to allow for FVAE to be trained and evaluated across different resolutions, with data provided on sparse and potentially irregular meshes.
We therefore discuss how to discretise the loss, and propose encoder/decoder architectures that can be evaluated on any mesh.

\subsubsection{Encoder Architecture}
\label{subsec:encoder_architecture}

The encoder $\encodermap$ is a map from a function $u \colon \Omega \to \Reals^{m}$ to the parameters of the encoder distribution \eqref{eq:encoder_infinite_dimensions}: the mean $\encodermean(u; \encoderparam) \in \latentspace$ and covariance matrix $\encodercov(u; \encoderparam) \in \spd{\latentspace}$.
We assume $\encodercov(u; \encoderparam)$ is diagonal, so $\encodermap$ need only return two vectors:
the mean $\encodermean(u; \encoderparam)$ and the log-diagonal of $\encodercov(u; \encoderparam)$.
We thus define
\begin{equation} \label{eq:encoder_integral_kernel}
	\encodermap(u; \encoderparam) = \rho\left(\int_{\Omega} \kappa\bigl(x, u(x); \encoderparam\bigr) \,\rd x; \encoderparam \right) \in \latentspace \times \latentspace = \Reals^{2d_{\latentspace}},
\end{equation}
where $\kappa \colon \Omega \times \Reals^{m} \times \Encoderparam \to \Reals^{\ell}$ is parametrised as a neural network with two hidden layers of width 64 and output dimension $\ell = 64$, using GELU activation \citep{HendrycksGimpel2016}, and $\rho \colon \Reals^{\ell} \times \Encoderparam \to \Reals^{2d_{\latentspace}}$ is parametrised as a linear layer $\rho(v; \encoderparam) = W^{\encoderparam}v + b^{\encoderparam}$, with $W^{\encoderparam} \in \Reals^{2d_{\latentspace} \times \ell}$ and $b^{\encoderparam} \in \Reals^{2d_{\latentspace}}$.
We augment $x \in \Omega$ with 16 random Fourier features  (\cref{subsec:details_base_architecture}) to aid learning of high-frequency features \citep{Tanciketal2020}.
After discretisation on data $\mathbf{u} = \{(x_{i}, u(x_{i}))\}_{i = 1}^{I}$, in which we approximate the integral over $\Omega$ by a normalised sum,
our architecture resembles set-to-vector maps such as deep sets \citep{Zaheeretal2017}, PointNet \citep{QiSuMoGuibas2017}, and statistic networks \citep{EdwardsStorkey2017}, which take the form
\begin{equation*}
	\bigset{\bigl(x_{i}, u(x_{i})\bigr)}{i = 1, 2, \dots, I} \mapsto \rho\Bigl(\mathrm{pool}\Bigl(\Set{\kappa\bigl(x_{i}, u(x_{i}); \encoderparam \bigr)}{i = 1, 2, \dots, I}\Bigr); \encoderparam\Bigr),
\end{equation*}
where $\mathrm{pool}$ is a pooling operation invariant to the order of its inputs---in our case, the mean.
Unlike these works we design our architecture for functions and only then discretise;
we believe there is great potential to extend other point-cloud and set architectures similarly.

Many other function-to-vector architectures have been proposed, e.g., the variable-input DeepONet (VIDON; \citealp{PrasthoferDeRyckMishra2022}), the mesh-independent neural operator (MINO; \citealp{Lee2022}) and continuum attention \citep{CalvelloKovachkiLevineStuart2024}, and our proposal is most similar to
the linear-functional layer of Fourier neural mappings \citep{HuangNelsenTrautner2024} and the neural functional of \citet{Rahmanetal2022}.
These differ from our approach by preceding \eqref{eq:encoder_integral_kernel} by a neural operator; on the problems we consider, we find our encoder map to be equally expressive.

\subsubsection{Decoder Architecture}
\label{subsec:decoder_architecture}

The decoder $\decodermap$ is a map from a latent vector $z \in \latentspace$ to a function $\decodermap(z; \decoderparam) \colon \Omega \to \Reals^{m}$, which we parametrise using a coordinate neural network $\gamma \colon \latentspace \times \Omega \times \Decoderparam \to \Reals^{m}$ with 5 hidden layers of width 100 using GELU activation throughout, so that
\begin{equation} \label{eq:nonlinear_decoder}
	\decodermap(z; \decoderparam)(x) = \gamma(z, x; \decoderparam).
\end{equation}
As before, we augment $x \in \Omega$ with 16 random Fourier features (\cref{subsec:details_base_architecture}).
Our proposed architecture allows for discretisation of the decoded function $\decodermap(z; \decoderparam)$ on any mesh, and the cost of evaluating the decoder \eqref{eq:nonlinear_decoder} grows linearly with the number of mesh points. 

There are several related approaches in the literature to parametrise vector-to-function maps.
\Citet{HuangNelsenTrautner2024} lift the input by multiplying with a learnable constant function, then apply an operator architecture such as FNO.
\Citet{SeidmanKissasPappasPerdikaris2023} propose both a DeepONet-inspired decoder using a linear combination of learnable basis functions, and a nonlinear decoder essentially the same as what we propose, which is also similar to the architectures of the nonlinear manifold decoder (NOMAD; \citealp{SeidmanKissasPerdikarisPappas2022}) and PARA-Net \citep{DeHoopHuangQianStuart2022}.
Also related are implicit neural representations \citep{Sitzmannetal2020}, in which one regresses on a fixed image using a coordinate neural network and  treats the resulting weights as a resolution-independent representation of the data.

\subsubsection{Discretisation of Per-Sample Loss}
\label{subsec:discretised_Losses}

To discretise the per-sample losses derived in \cref{subsec:VAE_in_infinite_dimensions}, we make two approximations.
First, we approximate the expectation over $z \sim \varmeas^{\encoderparam}$ by Monte Carlo sampling \citep{KingmaWelling2014}, with the number of samples viewed as a hyperparameter.
Second, we approximate the integrals, norms, and inner products arising in the loss, as we now outline.

\paragraph{Per-Sample Loss $\persampleloss^{\text{SDE}}_{\lambda, \beta}$.}
Since the terms appearing in $\persampleloss^{\text{SDE}}_{\lambda, \beta}$ are integral functionals of the data and decoded functions, we can discretise on any partition $0 = t_{0} < t_{1} < \cdots < t_{I} = T$ and work with data discretised at any time step.
The deterministic integral can be approximated by a normalised sum, and the stochastic integral can be discretised as
\begin{align*} 
	\int_{0}^{T} \bigl\langle \decodermap(z; \decoderparam)'(t),  \,\rd u_{t}\bigr\rangle &\approx \sum_{i = 1}^{I} \Bigl\langle \decodermap(z; \decoderparam)'(t_{i-1}),  u(t_{i}) - u(t_{i - 1})\Bigr\rangle,
\end{align*}
which converges in probability in the limit $I \to \infty$ to the true stochastic integral \citep[eq.~(4.6)]{SaerkkaeSolin2019}. 
Since the decoder will be a differentiable neural network, terms involving the derivative $\decodermap(z; \decoderparam)'(t)$ can be computed using automatic differentiation;
we find this to be much more stable than using a finite-difference approximation of the derivative.

\paragraph{Per-Sample Loss $\persampleloss^{\text{BIP}}$.}
For many Bayesian inverse problems, $\dataspace$ is a function space such as $L^{2}(\Omega)$, and so again the norms and inner products are integral functionals amenable to discretisation on any mesh.
However, applying the operator $C^{-\nicefrac{1}{2}}$ is typically tractable only in special cases. 
One widely used setting is the one in which the eigenbasis of $C$ is known and basis coefficients are readily computable; this arises when $C$ is an inverse power of the Laplacian on a rectangle and a fast Fourier transform may be used.

\subsection{Numerical Experiments}
\label{subsec:FVAE_numerical_examples}

We now apply FVAE on two examples where $\datameas$ is an SDE path distribution. Both examples serve as prototypes for more complex problems such as those arising in molecular dynamics.
For all experiments, we adopt the architecture of \cref{subsec:VAE_architecture}.
A summary of conclusions to be drawn from the numerical experiments with these examples is as follows:
\begin{enumerate}[label=(\alph*)]
	\item FVAE captures properties of individual paths as well as ensemble properties of the data set, with the learned latent variables being physically interpretable (\cref{subsec:Brownian_dynamics});
	\item choosing decoder noise that accurately reflects the stochastic variability in the data is essential to obtain a high-quality generative model (\cref{subsec:Brownian_dynamics});
	\item FVAE is robust to changes of mesh in the encoder and decoder, enabling training with heterogeneous data and generative modelling at any resolution (\cref{subsec:MSM}).
\end{enumerate}

\vspace{0.3em}

We emphasise in these experiments that FVAE does this purely from data, with no knowledge of the data-generating process other than in the choice of decoder noise.

\subsubsection{Brownian Dynamics}
\label{subsec:Brownian_dynamics}

The Brownian dynamics model \citep[see][Chap.~14]{Schlick2010}, also known as the Langevin model, is a stochastic approximation of deterministic Newtonian models for molecular dynamics.
In this model, the configuration $u_{t}$ (in some configuration space $X \subseteq \Reals^{m}$) of a molecule is assumed to follow the gradient flow of a \defterm{potential} $U \colon X \to \Reals$ perturbed by additive thermal noise with \defterm{temperature} $\varepsilon > 0$. 
This leads to the Langevin SDE
\begin{equation} \label{eq:Brownian_dynamics}
	\rd u_{t} = -\nabla U(u_{t}) \,\rd t + \sqrt{\varepsilon} \,\rd w_{t},\qquad t \in [0, T],
\end{equation}
where $(w_{t})_{t \in [0, T]}$ is a Brownian motion on $\Reals^{m}$.
As a prototype for the more sophisticated, high-dimensional potentials arising in molecular dynamics, such as the Lennard--Jones potential \citep{Schlick2010}, we take $X = \Reals$ and consider the asymmetric double-well potential 
\begin{equation} \label{eq:Brownian_dynamics_potential}
	U(x) \propto 3x^{4} + 2 x^{3} - 6x^{2} - 6x.
\end{equation}
This has a local minimum at $x_{1} = -1$ and a global minimum at $x_{2} = +1$ (\Cref{fig:sde_potential_and_sample_paths}(a)).
We take $\datameas$ to be the corresponding path distribution, with temperature $\varepsilon = 1$, final time $T = 5$, and initial condition $u_{0} = x_{1}$. (The preceding developments fixed $u_{0}=0$ but are readily adapted to any fixed initial condition.)
The training data set consists of 8,192 paths with time step $\nicefrac{5}{512}$ in $[0, T]$, and in each path 50\% of time steps are missing (see \cref{subsec:details_Brownian_dynamics}).

\begin{figure}[htb]
	\centering
	\vspace{-0.5em}
	\begin{tikzpicture}
		\node at (0, 0) {\includegraphics[width=13cm, clip, trim=0.4cm 0.45cm 0.4cm 0.4cm]{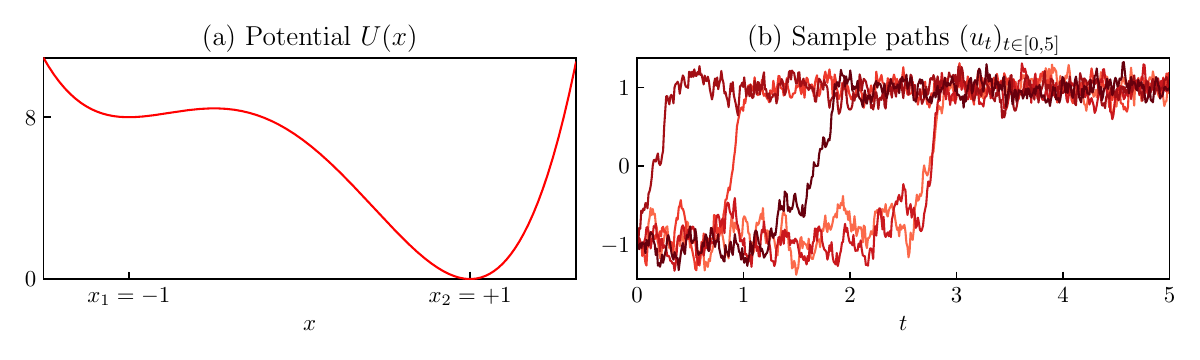}};
	\end{tikzpicture}
	\vspace{-0.9em}
	\caption{(a) Realisations of the SDE \eqref{eq:Brownian_dynamics} follow the gradient flow of the potential $U$. (b) Sample paths $(u_{t})_{t \in [0, T]} \sim \datameas$ begin at $x_{1} = -1$ and transition with high probability to the lower-potential state $x_{2} = +1$ as a result of  the additive thermal noise.}
	\label{fig:sde_potential_and_sample_paths}
\end{figure}

Sample paths drawn from $\datameas$ start at $x_{1}$ and transition with very high probability to the potential-minimising state $x_{2}$;
the time at which the transition begins is determined by the thermal noise, but, once the transition has begun, the manner in which the transition occurs is largely consistent across realisations.
Such universal transition phenomena occur quite generally in the study of random dynamical systems, as a consequence of large-deviation theory \citep[see][]{ERenVandenEijnden2004}.

We train FVAE using the SDE loss (\cref{subsubsec:SDE}) with regularisation parameter $\beta = 1.2$ and zero-penalty scale $\lambda = 10$.
Motivated by the observation that trajectories are determined chiefly by the transition time, we use latent dimension $d_{\latentspace} = 1$.

\paragraph{Choice of Noise Process.}
The choice of decoder noise greatly affects FVAE's performance as an autoencoder and a generative model.
To investigate this, we first train three instances of FVAE with different restoring forces $\kappa$ in the decoder-noise process.
Then, to evaluate autoencoding performance, we draw samples $u \sim \datameas$ from the held-out set and compute the \emph{reconstruction} $\decodermap(\encodermean(u; \encoderparam); \decoderparam)$, which is the mean of the decoder distribution $\decodermeas^{\decoderparam}$ with $z = \encodermean(u; \encoderparam)$ taken to be the mean of the encoder distribution $\varmeas^{\encoderparam}$ (\cref{fig:sde_reconstructions_and_samples}(a)).
To evaluate FVAE as a generative model, we draw samples from the latent distribution $\latentmeas$ and display the mean $\decodermap(z; \decoderparam)$ of the decoder distribution $\decodermeas^{\decoderparam}$ along with a shaded region indicating one standard deviation of the noise process (\cref{fig:sde_reconstructions_and_samples}(b));
moreover we draw samples $\decodermap(z; \decoderparam) + \eta$ to illustrate their qualitative behaviour  (\cref{fig:sde_reconstructions_and_samples}(c)).

\begin{figure}[t]
	\centering
	\begin{tikzpicture}
		\node at (0, 0) {\includegraphics[width=0.98\textwidth, clip, trim=0.4cm 0.5cm 0.4cm 0.9cm]{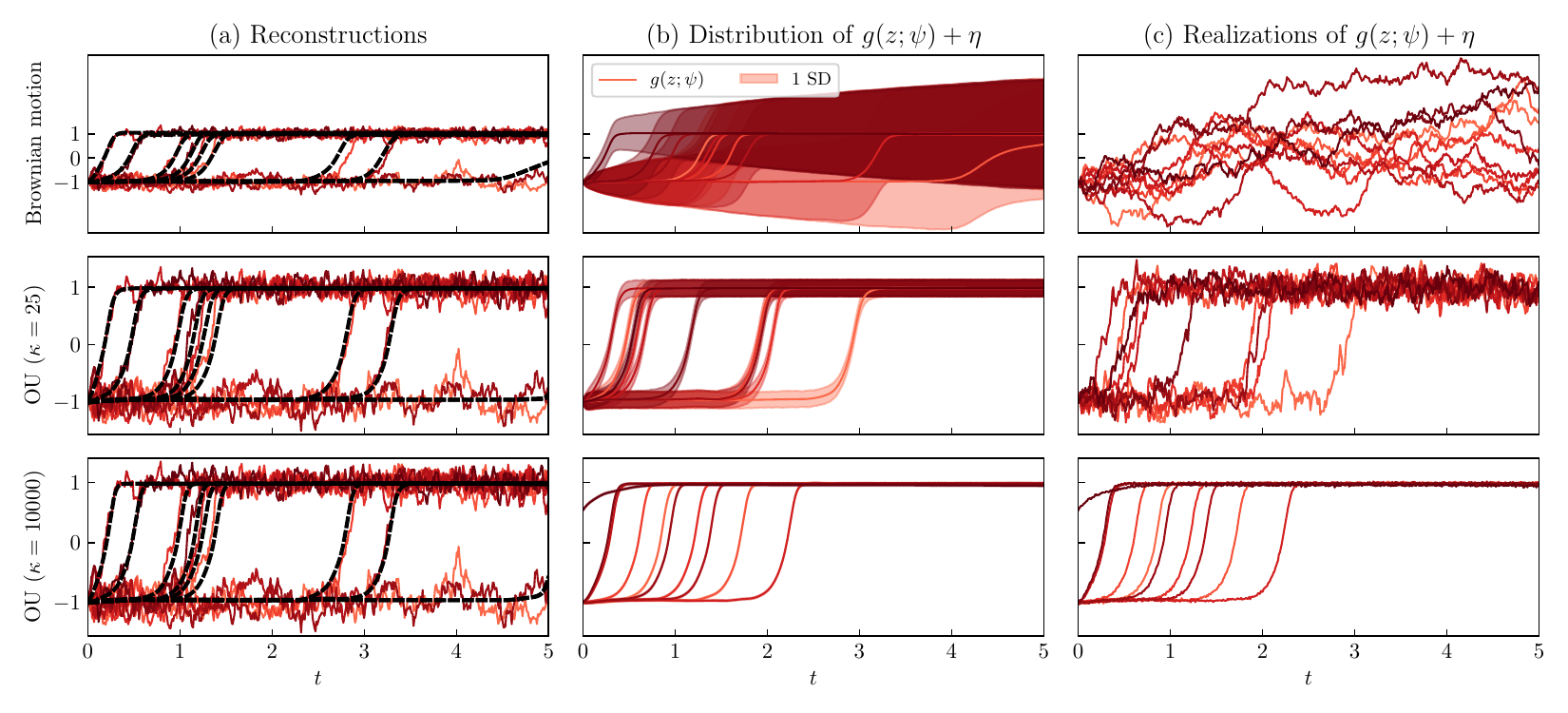}};
		\node at (-4.9, 3.75) {\footnotesize (a) Data $u \sim \datameas$ (red) \& };
		\node at (-4.85, 3.4) {\footnotesize reconstructions $\decodermap(\encodermean(u; \encoderparam); \decoderparam)$ (black)};
		\node at (0.3, 3.75) {\footnotesize (b) Distribution of $\decodermap(z; \decoderparam) + \eta$ with};
		\node at (0.3, 3.4) {\footnotesize 1 standard deviation of noise};
		\node at (5.1, 3.75) {\footnotesize (c) Samples $\decodermap(z;\decoderparam)+\eta$ from};
		\node at (5.1, 3.4) {\footnotesize $\genmeas^{\decoderparam}$, $z \sim \latentmeas$, $\eta \sim \decodernoisemeas$};
	\end{tikzpicture}
	\vspace{-2.0em}
	\caption{The SDE loss gives much freedom in the choice of noise process $(\eta_{t})_{t}$. The top row uses Brownian motion as the decoder noise and the second and third rows use OU processes with different asymptotic variances. While all choices lead to high-quality reconstructions, only the OU process with $\kappa = 25$ gives a generative model that agrees well with the data.}
	\label{fig:sde_reconstructions_and_samples}
    \vspace{-1em}
\end{figure}

Using Brownian motion as the decoder noise leads to excellent reconstructions, but samples from the generative model appear different from the training data.
By using OU-distributed noise with restoring force $\kappa > 0$ we obtain similar reconstructions to those achieved under Brownian motion, but samples from the generative model match the data distribution more closely.
This is because the variance of Brownian motion grows unboundedly with time, while the 
asymptotic variance under the OU process is $\nicefrac{\varepsilon}{2\kappa}$, better reflecting the behaviour of the data for well-chosen $\kappa$.

\begin{wrapfigure}{r}{0.48\textwidth}
	\vspace{-1.8em}
	\begin{tikzpicture}
		\node at (0, 0.2) {\includegraphics[width=0.48\textwidth,height=4.1cm]{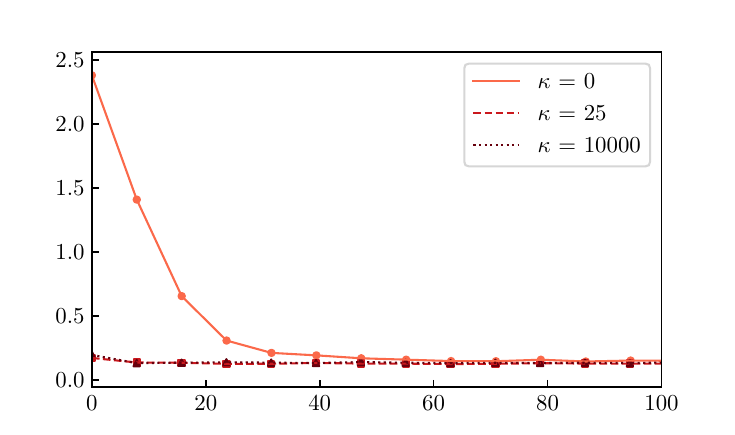}};
		\node at (0, -1.85) {\footnotesize Wall-clock time [s]};
		\node at (0, 2.05) {\footnotesize Reconstruction MSE on held-out set};
	\end{tikzpicture}
	\vspace{-2.2em}
		\caption{Using OU noise ($\kappa > 0$) leads to faster training convergence.}
		\label{fig:sde1d_wallclock}
		\vspace{-1.3em}
\end{wrapfigure}
On this data set, choosing a suitable noise process $(\eta_{t})_{t}$ has the added benefit of significantly accelerating training:
autoencoding mean-squared error (MSE) decreases much faster under OU noise ($\kappa > 0$) than under Brownian motion noise (\cref{fig:sde1d_wallclock}).
We expect the choice of noise should depend in general on properties of the data distribution, with the OU process being particularly suited to this data set;
in the discussion that follows, we use an OU process with restoring force $\kappa = 25$.

\paragraph{Unsupervised Learning of Physically Relevant Quantities.}
Our choice of latent dimension $d_{\latentspace} = 1$ was motivated by the heuristic that the time of the transition from $x_{1} = -1$ to $x_{2} = +1$ essentially determines the SDE trajectory.
FVAE identifies this purely from data, with the learned latent variable $z \in \latentspace$ being in correspondence with the transition time: larger values of $z$ map to paths $\decodermap(z; \decoderparam)$ transitioning later in time (\cref{fig:sde_latent_variable}(a)).
\begin{figure}[t]
	\centering
	\begin{tikzpicture}
		\node at (0, 0) {
			\includegraphics[width=0.85\textwidth, clip, trim=0.8cm 0.5cm 0.4cm 0.9cm]{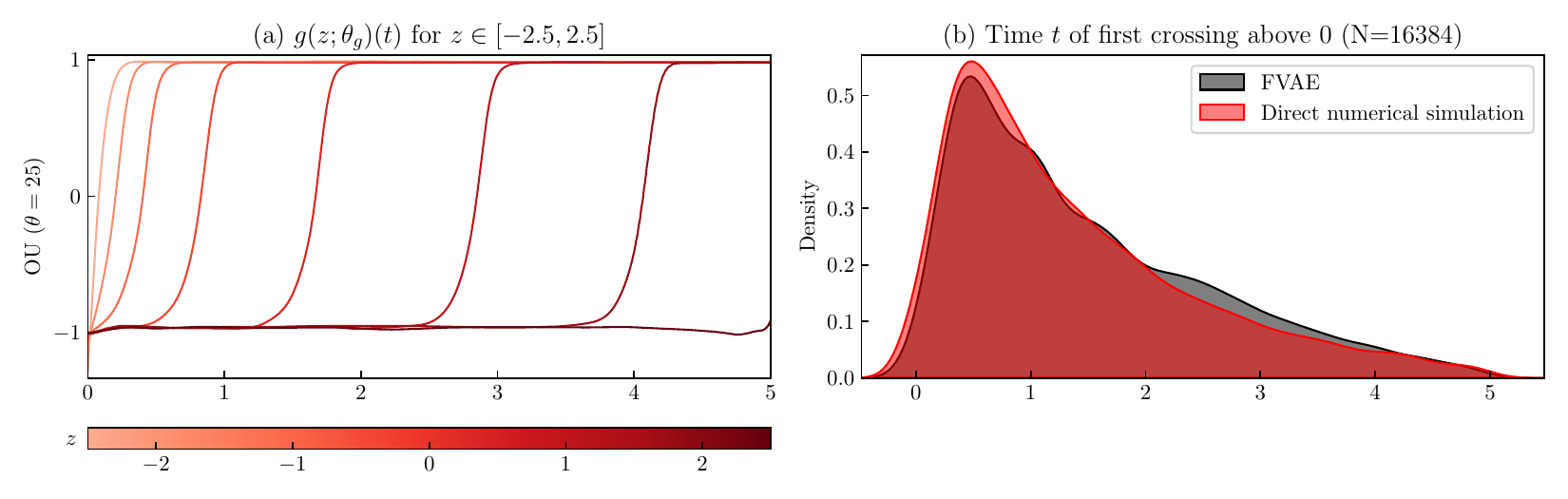}};
		\node at (-3.2, 2) {\footnotesize (a) Decoded paths $\decodermap(z; \decoderparam)(t)$, $z \in [-2.5, 2.5]$};
		\node at (3.55, 2) {\footnotesize (b) Distribution of $T_{0}(u)$ [16,384 samples]};
		\node at (-3.2, -1.27) {\tiny $t$};
		\node at (3.55, -1.45) {\footnotesize $T_{0}(u)$};
		\end{tikzpicture}
		\vspace{-0.8em}
		\caption{(a) The latent variable $z$ identified by FVAE corresponds to the first-crossing time $T_{0}$ of the decoded path $\decodermap(z; \decoderparam)$. (b) Kernel density estimates of the distributions of $T_{0}$ under the FVAE generative model, and when computed using direct simulations, closely agree.}
	\label{fig:sde_latent_variable}
\end{figure}

To understand whether FVAE  captures ensemble statistical properties of the data distribution, we compare the distributions of the first-crossing time $T_{0}(u) = \inf \bigset{t > 0}{u_{t} \geq 0}$
estimated using 16,384 paths from the generative model and 16,384 direct simulations, using kernel density estimates based on Gaussian kernels with bandwidths selected by Scott's rule \citep[eq.~(6.44)]{Scott2015};
we find that the two distributions closely agree (\cref{fig:sde_latent_variable}(b)).

\subsubsection{Estimation of Markov State Models}
\label{subsec:MSM}

In practical applications of molecular dynamics, one is often interested in the evolution of large molecules on long timescales.
For example in the study of protein folding \citep{Konovalovetal2021}, it is of interest to capture the complex, multistage transitions of proteins between configurations.
Moving beyond the toy one-dimensional problem in \cref{subsec:Brownian_dynamics}, the chief difficulty is the very high dimension of such systems, which makes simulations possible only on timescales orders of magnitudes shorter than those of physical interest.

Markov state models (MSMs) offer one method of distilling many simulations  on short timescales into a statistical model permitting much longer simulations \citep{HusicPande2018}.
Assuming that the dynamics are given by a random process $(u_{t})_{t \geq 0}$ taking values in the configuration space $X$, an MSM can be constructed by partitioning $X$ into disjoint state sets $X = X_{1} \cup \cdots \cup X_{p}$, and, for some \defterm{lag time} $\tau > 0$, considering the discrete-time process $(U_{k})_{k \in \Naturals}$ for which $U_{k} = i$ if and only if $u_{k\tau} \in X_{i}$.
One hopes that, if $\tau$ is sufficiently large, the process $(U_{k})_{k \in \Naturals}$ is approximately Markov, and thus its distribution can be characterised by learning the probabilities of transitioning in time $\tau$ from one state to another.
These probabilities can be determined using the short-run simulations---which can be generated in parallel---and the resulting MSM can be used to simulate on longer timescales.

Motivated by this application, we consider the problem of constructing an MSM from data provided at sparse or irregular intervals that do not necessarily align with the lag time $\tau$;
in this case, computing the probability of transition in time $\tau$ directly may not be possible.
We show the power of FVAE in this problem by first learning a generative model from the heterogeneously sampled data and then using the generative model to draw paths sampled at the regular time step $\tau$;
constructing an MSM from these paths is then straightforward.

We give an example based on the Brownian dynamics model \eqref{eq:Brownian_dynamics} on $X = \Reals^{2}$ using a multiwell potential $U$ (\cref{fig:MSM_transition_probabilities}(a)), stated precisely in \cref{subsec:details_MSM}, which we take to be a quadratic bowl,  perturbed by a linear function to break the symmetry, and by six Gaussian densities to act as potential wells with minima at $(0, 0)$, $(0.2, 0.2)$, $(-0.2, -0.2)$, $(0.2, -0.2)$, $(0, 0.2)$ and $(-0.2, 0)$.
We take $\dataspace = C_{0}([0, T], X)$ and let $\datameas \in \prob{\dataspace}$ be the path distribution, with temperature $\varepsilon = 0.1$, final time $T = 3$ and initial condition $u_{0} = 0$.
The training data set consists of 16,384 paths discretised with time step $\nicefrac{3}{512}$, where, for each sample, it is assumed that 50\% of steps are missing (details in \cref{subsec:details_MSM}).
We train FVAE using the SDE loss (\cref{subsubsec:SDE}) with $\kappa = 100$, $\lambda = 50$, $\beta = 0.02$, and latent dimension $d_{\latentspace} = 16$.

\begin{figure}[t]
	\vspace{-2em}
	\begin{tikzpicture}[trim left=8.2cm]
		\node at (0.3, 0.3) {\scalebox{0.72}{
		\begin{tikzpicture}
			\node[anchor=west, minimum width=10cm, minimum height=8cm, inner sep=0, outer sep=0] at (38.65, -16.65) {\includegraphics[width=5.2cm,height=5.2cm,clip, trim=2.4cm 0.95cm 3.8cm 0cm]{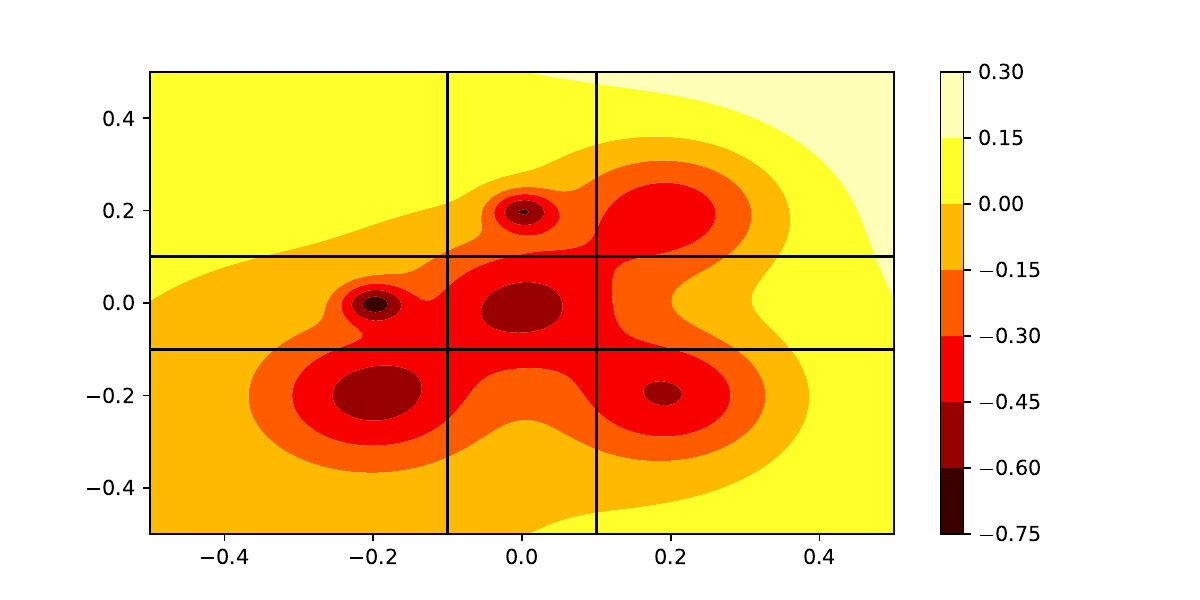}};
	\node[anchor=west] at (46.2, -14.77) {0.3};
	\node[anchor=west] at (46.2, -15.37) {0.15};
	\node[anchor=west] at (46.2, -15.97) {0};
	\node[anchor=west] at (46.2, -16.6) {-0.15};
	\node[anchor=west] at (46.2, -17.22) {-0.3};
	\node[anchor=west] at (46.2, -17.85) {-0.45};
	\node[anchor=west] at (46.2, -18.5) {-0.6};
	\node[anchor=west] at (46.2, -19.11) {-0.75};
	\node [anchor=east] at (41.1, -18.7) {-0.4};
	\node [anchor=east] at (41.1, -17.8) {-0.2};
	\node [anchor=east] at (41.1, -16.95) {0.0};
	\node [anchor=east] at (40.3, -16.95) {$x_{2}$};
	\node [anchor=east] at (41.1, -16.05) {0.2};
	\node [anchor=east] at (41.1, -15.15) {0.4};
	\node [anchor=east] at (41.95, -19.4) {-0.4};
	\node [anchor=east] at (42.85, -19.4) {-0.2};
	\node [anchor=east] at (43.78, -19.4) {0.0};
	\node [anchor=east] at (43.78, -19.8) {$x_{1}$};
	\node [anchor=east] at (44.69, -19.4) {0.2};
	\node [anchor=east] at (45.59, -19.4) {0.4};
	\node[anchor=west] at (41.75, -18.5) {\small \textbf{1}};
	\node[anchor=west] at (43.2, -18.5) {\small \textbf{2}};
	\node[anchor=west] at (44.6, -18.5) {\small \textbf{3}};
	\node[anchor=west] at (41.75, -17) {\small \textbf{4}};
	\node[anchor=west,white] at (43.2, -17) {\small \textbf{5}};
	\node[anchor=west] at (44.6, -17) {\small \textbf{6}};
	\node[anchor=west] at (41.75, -15.5) {\small \textbf{7}};
	\node[anchor=west] at (43.2, -15.5) {\small \textbf{8}};
	\node[anchor=west] at (44.6, -15.5) {\small \textbf{9}};

\end{tikzpicture}}};

\node at (1.7, -0.2) {\scalebox{0.8}{\begin{tikzpicture}[trim left=26cm]
	\node[anchor=west, inner sep=0, outer sep=0] at (68.9, -17.1) {\includegraphics[clip, width=3.7cm, trim=1.6cm 1.4cm 1.3cm 1.4cm]{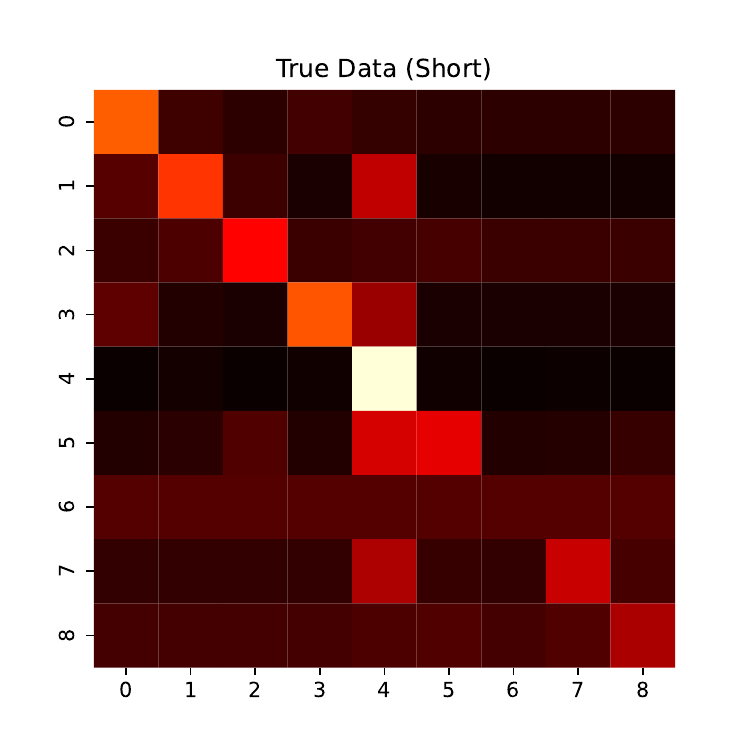}};
	\node[anchor=west, inner sep=0, outer sep=0] at (72.8, -17.1) {\includegraphics[clip, width=3.7cm, trim=1.6cm 1.4cm 1.3cm 1.4cm]{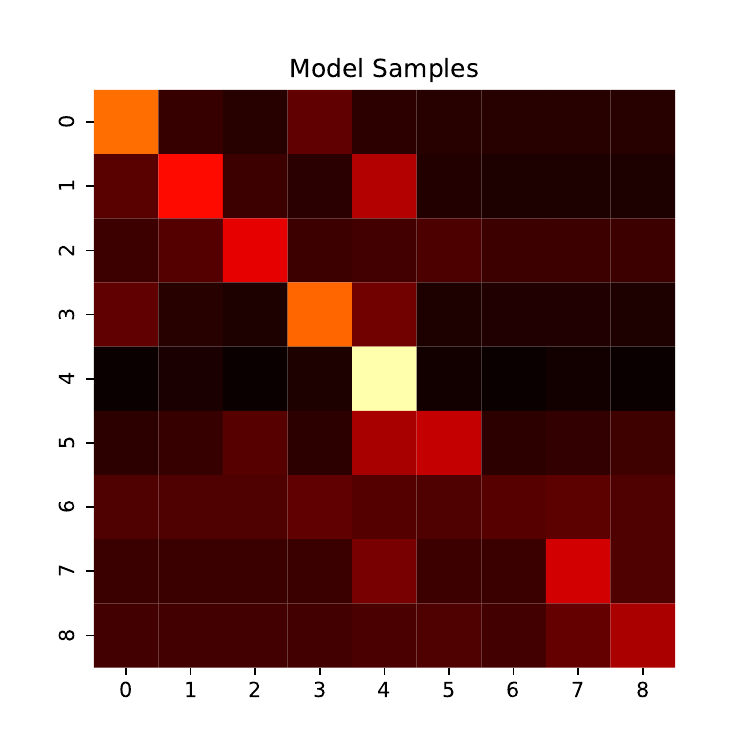}};
	\node[anchor=west, inner sep=0, outer sep=0] at (76.7, -17.27) {\includegraphics[clip, height=4cm, trim=17.15cm 0.7cm 2.7cm 0.72cm]{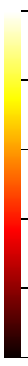}};	
\end{tikzpicture}}};
			\node at (11.2, 2.3) {\footnotesize (a) Potential $U \colon X \to \Reals$ };
			\node at (11.2, 2) {\footnotesize and states $X_{i}$, $i = 1,\dots, 9$ };
			\node at (18.8, 2.3) {\footnotesize (b) Maximum-likelihood transition matrices};
			\node at (17.1, 1.95) {\footnotesize (i) Direct numerical };
			\node at (17.1, 1.62) {\footnotesize simulation $\mathsf{T}^{\text{DNS}}(\tau)$};
			\node at (20.2, 1.95) {\footnotesize (ii) FVAE};
			\node at (20.2, 1.60) {\footnotesize $\mathsf{T}^{\text{FVAE}}(\tau)$};
			\node at (22.35, -1.74) {\footnotesize 0.0};
			\node at (22.35, -1.74+0.6) {\footnotesize 0.2};
			\node at (22.35, -1.74+0.6+0.6) {\footnotesize 0.4};
			\node at (22.35, -1.74+0.6+1.2) {\footnotesize 0.6};
			\node at (22.35, -1.74+1.2+1.2) {\footnotesize 0.8};
			\node at (22.35, -1.74+3) {\footnotesize 1.0};
		\end{tikzpicture}
	\vspace{-2.5em}
	\caption{(a) Contour plot of the potential $U$ and the division of the state space $X$. All paths start at $t = 0$ at the origin (in state $5$).  (b) Transition matrices with lag $\tau = \nicefrac{3}{512}$ computed using FVAE and through direct simulation, both on the time interval $[0, 3]$.}
	\label{fig:MSM_transition_probabilities}
\end{figure}

\paragraph{Partitioning the Configuration Space.}
The states of an MSM can be selected manually using expert knowledge or automatically using variational or machine-learning methods \citep{MardtPasqualiWuNoe2018}.
For simplicity, we choose the states by hand, partitioning $X = \Reals^{2}$ into $p = 9$ disjoint regions (\cref{fig:MSM_transition_probabilities}(a)) divided by the four lines $x_{1} = \pm 0.1$ and $x_{2} = \pm 0.1$.

\paragraph{Estimating Transition Probabilities with FVAE.} 
After training FVAE with irregularly sampled data, we draw samples from the generative model with regular time step $\tau$ and use these samples to compute the MSM transition probabilities.
Setting aside the question of Markovianity for simplicity, we draw from the generative model $M = \text{2,048}$ paths $\{v^{(m)}\}_{m = 1}^{M}$ discretised on a mesh of $K = 513$ equally spaced points with time step $\tau = \nicefrac{3}{512}$ on $[0, T]$, and compute the count matrix
\begin{equation*}
	\mathsf{C}^{\text{FVAE}}(\tau) = \bigl(\mathsf{C}^{\text{FVAE}}_{ij}(\tau)\bigr)_{i, j \in \{1, \dots, p\}},\quad \mathsf{C}^{\text{FVAE}}_{ij}(\tau) = \sum_{k = 0}^{K} \sum_{m = 1}^{M} \one \left[ v^{(m)}_{k\tau} \in X_{i} \text{~and~} v^{(m)}_{(k + 1)\tau} \in X_{j} \right].
\end{equation*}
We then derive the corresponding maximum-likelihood transition matrix $\mathsf{T}^{\text{FVAE}}(\tau)$ by normalising each row of $\mathsf{C}^{\text{FVAE}}(\tau)$ to sum to one; 
for simplicity we do not constrain the transition matrix to satisfy the detailed-balance condition \citep[see][Sec.~IV.D]{Prinzetal2011}.
The resulting transition matrix $\mathsf{T}^{\text{FVAE}}(\tau)$ agrees closely with the matrix $\mathsf{T}^{\text{DNS}}(\tau)$ computed analogously using 2,048 direct numerical simulations on the regular time step $\tau$ (\cref{fig:MSM_transition_probabilities}(b)).

\section{Problems with VAEs in Infinite Dimensions}
\label{sec:problems}

As we have seen in \cref{subsec:VAE_objective}, the empirical FVAE objective $\objectiveFVAE_{N}(\encoderparam, \decoderparam)$ for the data set $\{u^{(n)}\}_{n = 1}^{N}$ is based on a sequence of approximations and equalities:
\[
    \underbrace{ \frac{1}{N} \sum_{n = 1}^{N} \persampleloss(u^{(n)};\encoderparam, \decoderparam) }_{\qefed \objectiveFVAE_{N}(\encoderparam, \decoderparam)}
    \underset{\text{(A)}}{~~\approx~~}
    \underbrace{ \EE_{u \sim \datameas} [ \persampleloss(u;\encoderparam, \decoderparam) ] }_{\qefed \objectiveFVAE(\encoderparam, \decoderparam)}
    \underset{\text{(B)}}{~~=~~}
    \KLdiv{ \encoderjointmeas^{\encoderparam} }{ \decoderjointmeas^{\decoderparam} } - \underset{\text{``constant''}}{\underbrace{\KLdiv{\datameas}{\referencemeas}}}.
\]
The approximation (A) is based on the law of large numbers and is an equality almost surely in the limit $N \to \infty$.
The equality (B) is true by \cref{thm:FVAE_tractable_objective} with a finite constant $\KLdiv{\datameas}{\referencemeas}$ provided \cref{assumption:tractable_objective} holds---but if the assumption does not hold, this ``constant'' may well be infinite.
So, while it is tempting to apply the empirical objective $\objectiveFVAE_{N}$ without first checking the validity of (A) and (B), this strategy is fraught with pitfalls.

To illustrate this we apply FVAE in the white-noise setting of \Cref{ex:white-noise_on_L2};
this coincides with the setting of the VANO model \citep{SeidmanKissasPappasPerdikaris2023}, which we discuss in detail in the related work (\cref{sec:related_work}).
In this example we derive the per-sample loss $\persampleloss$ and apply the resulting empirical objective $\objectiveFVAE_{N}$ for training;
but both the joint divergence $\KLdiv{\encoderjointmeas^{\encoderparam}}{\decoderjointmeas^{\decoderparam}}$ and the constant $\KLdiv{\datameas}{\referencemeas}$ turn out to be infinite.
Consequently, the approximations (A) and (B) break down, which we see numerically:
discretisations of $\objectiveFVAE_{N}$ appear to diverge as resolution is refined, suggesting that they have no continuum limit.

\begin{example}
    \label[example]{ex:infinite_ELBO_Dirac}
	Take $\dataspace = H^{-1}([0, 1])$ and assume that $\datameas$ is the distribution of $u$ in the model
    \begin{align*}
        \xi &\sim \Uniform[0, 1]\\
        u \mid \xi &= \dirac{\xi}.
    \end{align*}
	Realisations of $u$ in this model lie in $H^{s}([0, 1])$, $s < -\nicefrac{1}{2}$, with probability one, so $\datameas \in \prob{\dataspace}$.
	This data can be viewed as a prototype for rough behaviour such as the derivative of shock profiles arising in hyperbolic PDEs with random initial data. 
    It will serve as an extreme example allowing us to isolate the numerical issues associated with using FVAE or VANO in the misspecified setting.

    Take the model \eqref{eq:encoder_infinite_dimensions}--\eqref{eq:latent_distribution_infinite_dimensions} for real-valued functions on $[0, 1]$, with $\decodernoisemeas$ taken to be $L^{2}$-white noise.
	As discussed in \Cref{prop:white_noise}, $\decodernoisemeas \in \prob{H^{s}([0, 1])}$, $s < -\nicefrac{1}{2}$, and $\CMspace{\decodernoisemeas} = L^{2}([0, 1])$.
	Fixing the reference distribution $\referencemeas = \decodernoisemeas$ and assuming $\decodermap$ takes values in $L^{2}([0, 1])$,
	the Cameron--Martin theorem (\Cref{prop:CM_theorem}) ensures that the density $\rd \decodermeas^{\decoderparam} / \rd \referencemeas$ exists, and consequently
	\begin{equation} \label{eq:L2_white_noise_per-sample_loss} 
		\persampleloss(u; \encoderparam, \decoderparam)  = \EE_{z \sim \varmeas^{\encoderparam}}\Bigl[ \tfrac{1}{2} \norm{\decodermap(z; \decoderparam)}_{L^{2}}^{2}  - \Pwint{\decodermap(z; \decoderparam)}{u}_{L^{2}} \Bigr] + \bigKLdiv{\varmeas^{\encoderparam}}{\latentmeas}.
    \end{equation}
	At this stage, we have not verified that the joint divergence $\KLdiv{\encoderjointmeas^{\encoderparam}}{\decoderjointmeas^{\decoderparam}}$ has finite infimum, nor that \Cref{assumption:tractable_objective} holds;
    indeed we will see that both of these conditions fail.
	Nevertheless we can attempt to train FVAE with the empirical objective resulting from \eqref{eq:L2_white_noise_per-sample_loss}.
	To do this we choose $\latentspace = \Reals$ and use an encoder map $\encodermap$ and a decoder map $\decodermap$ tailored to this problem, since the architectures of \cref{subsec:VAE_architecture} are not equipped to deal with functions of negative Sobolev regularity.
	We parametrise $\encodermap$ as 
\begin{equation*}
	\encodermap(u; \encoderparam) = \rho\biggl( \argmax_{x \in [0, 1]} (\varphi \ast u)(x); \encoderparam \biggr) \in \latentspace \times \latentspace = \Reals^{2},
\end{equation*}
where $\rho$ is a neural network and $\varphi$ is a compactly supported smooth mollifier, chosen such that $\varphi \ast u$ has well-defined maximum, and we parametrise $\decodermap$ to return the Gaussian density
\begin{align*}
	g\bigl(z; \decoderparam\bigr)(x) = N\bigl(\mu(z; \decoderparam), \sigma(z; \decoderparam)^{2} ; x\bigr),
\end{align*}
with mean $\mu(z; \decoderparam) \in [0, 1]$ and standard deviation $\sigma(z; \decoderparam) > 0$ computed from a neural network (see \cref{subsec:details_random_dirac}).
In other words, $\encodermap$ applies a neural network to the location of the maximum of $u$, while $\decodermap$ returns a Gaussian density with learned mean and variance (\cref{fig:Dirac_variable-centre_results}(a)).
To investigate the behaviour of discretisations of $\objectiveFVAE_{N}$ as resolution is refined, we generate a sequence of data sets in which we discretise on a mesh of $I \in \{8, 16, 32, 64, 128\}$ equally spaced points $\{\nicefrac{i}{I + 1}\}_{i = 1, \dots, I} \subset [0, 1]$.
At each resolution we generate $I$ training samples: 
one discretised Dirac function at each mesh point, normalised to have unit $L^{1}$-norm.
We train 50 independent instances of FVAE at each resolution and record the value of the empirical objective at convergence (\cref{fig:Dirac_variable-centre_results}(a)).
Notably, the empirical objective appears to diverge as resolution is refined.
\begin{figure}[t]
	\centering
	\begin{tikzpicture}
        \node at (-4.8, 2.95) {\footnotesize (a) Schematic diagram of };
        \node at (-4.8, 2.6) {\footnotesize \cref{ex:infinite_ELBO_Dirac}};
        \node at (-5.0, 0.21) {\includegraphics[width=5cm, height=4.4cm, clip, trim=0.5cm 0.5cm 19.2cm 0.67cm]{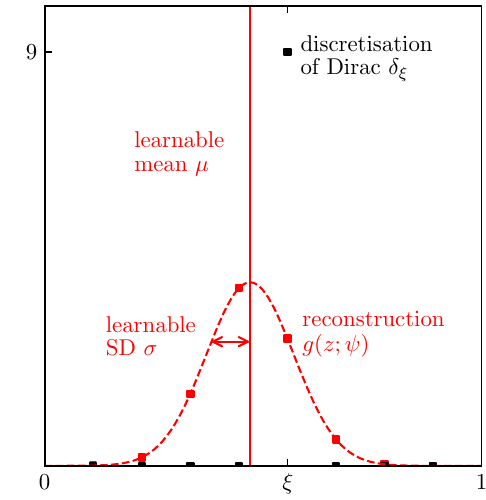}};
    
		\node at (0.2, 2.95) {\footnotesize (b) FVAE empirical objective};
		\node at (0.2, 2.60) {\footnotesize $\decodernoisemeas = N(0, I)$; 50 runs};
		\node at (0, 0.21) {\includegraphics[width=5cm, height=4.4cm, clip, trim=18.9cm 0.5cm 0.5cm 0.67cm]{images/dirac_moving_centre_results}};
		\node at (0.1, -2.08) {\footnotesize Resolution};
		\node[red] at (0.9, 0.59) {\footnotesize median};
		\node at (5.2, 2.95) {\footnotesize (c) FAE empirical objective};
		\node at (5.25, 2.60) {\footnotesize $\dataspace = H^{-1}([0, 1])$; 50 runs};
		\node at (5.0, 0.21) {\includegraphics[width=5cm, height=4.4cm, clip, trim=9.1cm 0.5cm 9.9cm 0.67cm]{images/dirac_moving_centre_results}};
		\node[red] at (4.3, 0.92) {\footnotesize median};
		\node at (5.17, -2.08) {\footnotesize Resolution};
	\end{tikzpicture}
	\caption{(a) The discrete representations (squares) of $\delta_{\xi}$ and $\decodermap(z; \decoderparam)$ on a grid of $8$ points. (b) In \Cref{ex:infinite_ELBO_Dirac},  the FVAE empirical objective at the minimising parameters diverges as resolution is increased. (c) To overcome this, we propose a regularised autoencoder, FAE, in \cref{sec:regularised_AEs}. Repeating the experiment with this objective suggests that the FAE empirical objective has a well-defined continuum limit.}
	\label{fig:Dirac_variable-centre_results}
\end{figure}
This has two major causes:
\begin{enumerate}[label=(\alph*)]
	\item 
	$\datameas$ is not absolutely continuous with respect to $\referencemeas = \decodernoisemeas$:
	the set $\set{\dirac{x}}{x \in [0, 1]}$ has probability zero under $\referencemeas$ but probability one under $\datameas$;
	thus $\KLdiv{\datameas}{\referencemeas} = \infty$.

	\item
	$\datameas$ is not absolutely continuous with respect to 
	$\genmeas^{\decoderparam}$, meaning that  $\KLdiv{\encoderjointmeas^{\encoderparam}}{\decoderjointmeas^{\decoderparam}} = \infty$ for all $\encoderparam$ and $\decoderparam$.
	To see this, we again note that $\set{\dirac{x}}{x \in [0, 1]}$ has probability one under $\datameas$, but, as a consequence of the Cameron--Martin theorem, $\decodermeas^{\decoderparam}$ and $\decodernoisemeas$ are mutually absolutely continuous, and thus as $\decodernoisemeas(\set{\dirac{x}}{x \in [0, 1]}) = 0$,
	\begin{equation*}
		\genmeas^{\decoderparam}\bigl( \set{\dirac{x}}{x \in [0, 1]} \bigr) = \int_{\latentspace} \decodermeas^{\decoderparam}\bigl(\set{\dirac{x}}{x \in [0, 1]}\bigr) \,\latentmeas(\rd z) = 0.
	\end{equation*}
\end{enumerate}
The problematic term in the per-sample loss is the measurable linear functional $\Pwint{\decodermap(z; \decoderparam)}{u}_{L^{2}}$;
this is defined only up to modification on $\decodernoisemeas$-probability zero sets, and yet we evaluate it on just such sets---namely, the $\decodernoisemeas$-probability zero set $\set{\dirac{x}}{x \in [0, 1]}$.
\qedremark
\end{example}

\begin{remark}
    \label[remark]{rk:problems_L2_data}
    The joint divergence $\KLdiv{\encoderjointmeas^{\encoderparam}}{\decoderjointmeas^{\decoderparam}}$ would also be infinite if $\datameas$ was supported on $L^{2}([0, 1])$, as seen in \cref{ex:white-noise_on_L2}, but it is harder to observe any numerical issue in training.
    This is because the measurable linear functional $\Pwint{\decodermap(z; \decoderparam)}{u}_{L^{2}}$ reduces to the usual $L^{2}$-inner product, and so, even though the FVAE objective is not well defined, the per-sample loss can still be viewed as a regularised misfit (see \cref{rk:Cameron--Martin_exponent}):
    \begin{equation*}
        \persampleloss(u; \encoderparam, \decoderparam) = \EE_{z \sim \varmeas^{\encoderparam}} \Bigl[ \tfrac{1}{2} \norm{\decodermap(z; \decoderparam) - u}^{2}_{L^{2}} \Bigr] + \KLdiv{\varmeas^{\encoderparam}}{\latentmeas} - \tfrac{1}{2} \norm{u}_{L^{2}}^{2}.
    \end{equation*}
    As a result one can reinterpret the objective as that of a regularised autoencoder (see \cref{rk:VAE_regularised_autoencoder}).
    This motivates our use of a regularised autoencoder, FAE, in \cref{sec:regularised_AEs}.
    \qedremark
\end{remark}

\section{Regularised Autoencoders on Function Space}
\label{sec:regularised_AEs}

To overcome the issues in applying VAEs in infinite dimensions, we set aside the probabilistic motivation for FVAE and define a regularised autoencoder in function space, the \defterm{functional autoencoder (FAE)}, avoiding the need for onerous conditions on the data distribution.
In \cref{subsec:FAE_objective}, we state the FAE objective and make connections to the FVAE objective.
\Cref{subsec:FAE_architecture} outlines minor adaptations to the FVAE encoder and decoder for use with FAE.
\Cref{subsec:FAE_numerical_examples} demonstrates FAE on two examples from the sciences: incompressible
fluid flows governed by the Navier--Stokes equation; and porous-medium flows governed by Darcy's law.

\subsection{Training Objective}
\label{subsec:FAE_objective}

Throughout \cref{sec:regularised_AEs} we make the following assumption on the data,  postponing discussion of discretisation to \cref{subsec:FAE_architecture}.
Unlike in \cref{sec:VAE} we do not need $\dataspace$ to be separable, allowing us to consider data from an even wider variety of spaces, such as the (non-separable) space $\mathrm{BV}(\Omega)$ of bounded-variation functions on $\Omega \subset \Reals^{d}$.

\begin{assumption}
	Let $(\dataspace, \norm{\quark})$ be a Banach space. 
	There exists a data distribution $\datameas \in \prob{\dataspace}$ from which we have access to $N$ independent and identically distributed samples $\{u^{(n)}\}_{n = 1}^{N} \subset \dataspace$. 
    \qedremark
\end{assumption}

To define our regularised autoencoder, we fix a latent space $\latentspace = \Reals^{d_{\latentspace}}$ and define encoder and decoder transformations $\encodermean$ and $\decodermap$, which, unlike in FVAE, return points rather than probability distributions:
\begin{subequations}
\begin{align}
	\text{(encoder)}~~~~~~\dataspace \ni u &\mapsto \encodermean(u; \encoderparam) \in \latentspace, \label{eq:FAE_encoder}
	\\
	\text{(decoder)}~~~~~~\latentspace \ni z &\mapsto \decodermap(z; \decoderparam) \in \dataspace. \label{eq:FAE_decoder}
\end{align}
\end{subequations}
We then take as our objective the sum of a misfit term between the data and its reconstruction, and a regularisation term with \defterm{regularisation parameter} $\beta > 0$ on the encoded vectors:
\begin{equation} \label{eq:FAE_objective}
	\text{(FAE objective)}~~~~\objectiveFAE_{\beta}(\encoderparam, \decoderparam) = \EE_{u \sim \datameas} \Bigl[ \tfrac{1}{2} \bignorm{\decodermap\bigl(\encodermean(u; \encoderparam); \decoderparam\bigr) - u}^{2} + \beta \bignorm{\encodermean(u; \encoderparam)}_{2}^{2} \Bigr].
\end{equation}
As in \cref{subsec:VAE_objective}, the expectation over $u \sim \datameas$ is approximated by an average over the training data, resulting in the empirical objective $\objectiveFAE_{\beta, N}$.
There is great flexibility in the choice of regularisation term;
we adopt the squared Euclidean norm $\norm{\encodermean(u; \encoderparam)}_{2}^{2}$ as a simplifying choice consistent with using a Gaussian prior $\latentmeas$ in a VAE (\Cref{rk:VAE_regularised_autoencoder}).
While \eqref{eq:FAE_objective} has much in common with the FVAE objective $\objectiveFVAE$, it is not marred by the foundational issues raised in \cref{sec:problems};
indeed, the FAE objective is broadly applicable as the following result shows.

\begin{proposition}
	\label[proposition]{prop:FAE_infimum_finite}
    
	Suppose $\datameas$ has finite second moment, i.e., $\EE_{u \sim \datameas} \bigl[\norm{u}^{2} \bigr] < \infty$.
    If there exist $\encoderparam^{\star} \in \Encoderparam$ and $\decoderparam^{\star} \in \Decoderparam$ such that $\encodermean(u; \encoderparam^{\star}) = 0$ and $\decodermap(z; \decoderparam^{\star}) = 0$,
    then \eqref{eq:FAE_objective} has finite infimum.
\end{proposition}

\begin{proof}
	This follows immediately from evaluating $\objectiveFAE_{\beta}$ at $\encoderparam^{\star}$ and $\decoderparam^{\star}$, since
	$\objectiveFAE_{\beta}(\encoderparam^{\star}, \decoderparam^{\star}) = \EE_{u \sim \datameas} \bigl[ \tfrac{1}{2} \Norm{u}^{2} \bigr]$,
	and the expectation is finite by hypothesis.
\end{proof}

\subsection{Architecture and Algorithms}
\label{subsec:FAE_architecture}

To train FAE we must discretise the objective $\objectiveFAE_{\beta}$ and parametrise the encoder $\encodermean$ and decoder $\decodermap$ with learnable maps. 
We moreover propose a masked training scheme that appears to be new to the operator-learning literature;
as we will show in \cref{subsec:FAE_numerical_examples}, this scheme greatly improves the robustness of FAE to changes of mesh.

\paragraph{Encoder and Decoder Architecture.}
As in \cref{subsec:VAE_architecture}, we will construct encoder and decoder architectures under the assumption that $\dataspace$ is a Banach space of functions evaluable pointwise almost everywhere with domain $\Omega \subseteq \Reals^{d}$ and range $\Reals^{m}$, and that we have access to discretisations $\mathbf{u}^{(n)}$ of the data $u^{(n)}$ comprised of evaluations at finitely many mesh points.
We adopt an architecture near identical to that used for FVAE.
More precisely, we parametrise the encoder as
\begin{equation*}
    \encodermean(u; \encoderparam) = \rho\left( \int_{\Omega} \kappa\bigl(x, u(x); \encoderparam \bigr) \,\rd x; \encoderparam \right) \in \latentspace,
\end{equation*}
with $\kappa \colon \Omega \times \Reals^{m} \times \Encoderparam \to \Reals^{\ell}$ parametrised as a neural network with two hidden layers of width 64, output dimension $\ell = 64$, and with $\rho \colon \Reals^{\ell} \times \Encoderparam \to \Reals^{d_{\latentspace}}$ parametrised as the linear layer $\rho(v; \encoderparam) = W^{\encoderparam} v + b^{\encoderparam}$.
We parametrise the decoder as the coordinate neural network $\gamma \colon \latentspace \times \Omega \times \Decoderparam \to \Reals^{m}$ with 5 hidden layers of width 100, so that
\begin{equation*} 
	\decodermap(z; \decoderparam)(x) = \gamma(z, x; \decoderparam) \in \Reals^{m}.
\end{equation*}
In both cases we use GELU activation and augment $x$ with 16 random Fourier features (\cref{subsec:details_base_architecture}).
Relative to the architectures of \cref{subsec:VAE_architecture}, the only change is in the range of $\encodermean$, which now takes values in $\latentspace$ instead of returning distributional parameters for $\varmeas^{\encoderparam}$.

\paragraph{Discretisation of $\objectiveFAE_{\beta}$.}
The FAE objective can be applied whenever $\dataspace$ is Banach, but in this article we take $\dataspace = L^{2}(\Omega)$, where $\Omega = \TT^{d}$ or $\Omega = [0, 1]^{d}$, and discretise the $\dataspace$-norm with a normalised sum.
One can readily imagine other possibilities, e.g., taking $\dataspace$ to be a Sobolev space of order $s \geq 1$ if derivative information is available \citep{Czarneckietal2017}, and approximating the $L^{2}$-norm of the data and its derivatives by sums.
More generally we may use linear functionals as the starting point for approximation.

\subsubsection{Masked Training}

Self-supervised training---learning to predict the missing data from masked inputs---has proven valuable in both language models such as BERT \citep{DevlinChangLeeToutanova2019} and vision models such as the masked autoencoder (MAE; \citealp{Heetal2022}). 
This method has been shown to both reduce training time and to improve generalisation.
We propose two schemes making use of the ability to discretise the encoder and decoder on arbitrary meshes:
\defterm{complement masking} and \defterm{random masking}.
Under both schemes we transform each discretised training sample $\mathbf{u} = \{(x_{i}, u(x_{i}))\}_{i = 1}^{I}$ by subsampling with index sets $\mathcal{I}_{\text{enc}}$ and $\mathcal{I}_{\text{dec}}$, which may change at each training step, to obtain
\begin{equation*}
	\mathbf{u}_{\text{enc}} = \Set{\bigl(x_{i}, u(x_{i})\bigr)}{i \in \mathcal{I}_{\text{enc}}},\qquad
	\mathbf{u}_{\text{dec}} = \Set{\bigl(x_{i}, u(x_{i})\bigr)}{i \in \mathcal{I}_{\text{dec}}}.
\end{equation*}
We supply $\mathbf{u}_{\text{enc}}$ as input to the discretised encoder;
moreover we discretise the decoder on the mesh $\{x_{i}\}_{i \in \mathcal{I}_{\text{dec}}}$ and compare the decoder output to the masked data $\mathbf{u}_{\text{dec}}$.
In both of the strategies we propose $\mathcal{I}_{\text{enc}}$ and $\mathcal{I}_{\text{dec}}$ will be unstructured random subsets of $\{1, \dots, I\}$, but in principle other masks---such as polygons---could be considered.

\paragraph{Complement Masking.}
The chief strategy used in the numerical experiments is to draw $\mathcal{I}_{\text{enc}}$ as a random subset of $\{1, \dots, I\}$ and take $\mathcal{I}_{\text{dec}} = \{1, \dots, I\} \setminus \mathcal{I}_{\text{enc}}$, fixing the \textbf{(encoder) point ratio} $r_{\text{enc}} = |\mathcal{I}_{\text{enc}}| / I$ as a hyperparameter.
This is similar to the strategy adopted by MAE---though our approach differs by masking individual mesh points instead of patches.
We explore the tradeoffs in the choice of point ratio in \cref{subsec:Navier--Stokes}.

\paragraph{Random Masking.}
A second strategy we consider is to independently draw $\mathcal{I}_{\text{enc}}$ and $\mathcal{I}_{\text{dec}}$ as random subsets of $\{1, \dots, I\}$, fixing both the encoder point ratio $r_{\text{enc}} = |\mathcal{I}_{\text{enc}}|/I$ and the decoder point ratio $r_{\text{dec}} = |\mathcal{I}_{\text{dec}}|/I$.
This gives greater control of the cost of evaluating the encoder and decoder:
by taking both $r_{\text{enc}}$ and $r_{\text{dec}}$ to be small, we significantly reduce the cost of each training step, which is useful when the number of mesh points $I$ is large.

\subsection{Numerical Experiments}
\label{subsec:FAE_numerical_examples}

In \cref{subsec:Navier--Stokes,subsec:Darcy}, we apply FAE as an out-of-the-box method to discover a low-dimensional latent space for solutions to the incompressible 
Navier--Stokes equations and the  Darcy model for flow in a porous medium, respectively. We find that for these data sets:
\begin{enumerate}[label=(\alph*)]
    \item FAE's mesh-invariant architecture autoencodes with performance comparable to convolutional neural network (CNN) architectures of similar size (\cref{subsec:Navier--Stokes});
	\item the ability to discretise the encoder and decoder on different meshes enables new applications to inpainting and data-driven superresolution (\cref{subsec:Navier--Stokes}), as well as extensions of existing zero-shot superresolution as proposed for VANO by \cite{SeidmanKissasPappasPerdikaris2023};
	\item masked training significantly improves performance under mesh changes (\cref{subsec:Navier--Stokes}) and can accelerate training while reducing memory demand (\cref{subsec:Darcy});
	\item training a generative model on the FAE latent space leads to a resolution-invariant generative model which accurately captures distributional properties (\cref{subsec:Darcy}).
\end{enumerate}

\subsubsection{Incompressible Navier--Stokes Equations}
\label{subsec:Navier--Stokes}

We first illustrate how FAE can be used to learn a low-dimensional representation for snapshots of the vorticity of a fluid flow in two spatial dimensions governed by the 
incompressible Navier--Stokes equations, and illustrate some of the benefits of our mesh-invariant model.

Let $\Omega = \TT^{2}$ be the torus, viewed as the square $[0, 1]^{2}$ with opposing edges identified and with unit normal $\hat{z}$, and let $\dataspace = L^{2}(\Omega)$.
While the incompressible Navier--Stokes equations are typically formulated in terms of the primitive variables of velocity $u$ and pressure $p$, it is more natural in this case to work with the 
vorticity--streamfunction formulation \citep[see][eq.~(2.6)]{ChandlerKerswell2013}. 
In particular the vorticity 
$\nabla \times u$ is zero except in the out-of-plane component $\hat{z} \omega$. The scalar component
of the vorticity, $\omega$, then satisfies
\begin{equation} \label{eq:Navier--Stokes_velocity--vorticity}
	\begin{alignedat}{2}
		\partial_{t} \omega &= \hat{z} \cdot \bigl( \nabla \times (u \times \omega \hat{z})  \bigr) + \nu \Delta \omega + \varphi, &&\text{~~~~ $(x,t) \in \Omega \times (0, T]$}, \\
		\omega(x, 0) &= \omega_{0}(x), &&\text{~~~~ $x \in \Omega.$}
	\end{alignedat}
\end{equation}
In this setting the velocity is given by $u = \nabla \times (\psi_{s} \hat{z})$, where the streamfunction $\psi_{s}$ satisfies $\omega = \Delta \psi_{s}$.\footnote{While it is typical in the literature \citep[e.g.,][]{ChandlerKerswell2013} to denote the streamfunction by $\psi$, we instead denote it by $\psi_{s}$ to avoid ambiguity with the decoder parameter $\psi$.}
Thus, using periodicity, $\psi_{s}$ is uniquely defined, up to an irrelevant constant, in
terms of $\omega$, and \eqref{eq:Navier--Stokes_velocity--vorticity} defines a closed evolution equation for $\omega$.
We suppose that $\datameas \in \prob{\dataspace}$ is the distribution of the scalar vorticity $\omega(\quark, T = 50)$ given by \eqref{eq:Navier--Stokes_velocity--vorticity}, with
viscosity $\nu = 10^{-4}$ and forcing 
\begin{equation*}
 \varphi(x) = \tfrac{1}{10} \sin(2\pi x_{1} + 2\pi x_{2}) + \tfrac{1}{10} \cos(2\pi x_{1} + 2 \pi x_{2}).
\end{equation*}
We assume that the initial condition $\omega_{0}$ has distribution $N(0, C)$ with covariance operator $C = 7^{3/2}(49I - \Delta)^{-5/2}$, where $\Delta$  is the Laplacian operator for spatially-mean-zero functions on the torus $\TT^{2}$.
The training data set, based on that of \citet{Lietal2021}, consists of 8,000 samples from $\datameas$ generated on a $64 \times 64$ grid using a pseudospectral solver, with a further 2,000 independent samples held out as an evaluation set.
The data are scaled so that $\omega(x, T) \in [0, 1]$ for all $x \in \Omega$.
Further details are provided in \Cref{subsec:details_Navier--Stokes}.

We train FAE with latent dimension $d_{\latentspace} = 64$ and regularisation parameter $\beta = 10^{-3}$, and train with complement masking using a point ratio $r_{\text{enc}}$ of 70\%.

\paragraph{Performance at Fixed Resolution.}
\begin{table}[th]
            \vspace{-0.4em}
	\centering
            {\footnotesize
		\begin{tabular}{lcc}
			\multicolumn{3}{c}{\footnotesize Autoencoding on evaluation set ($64 \times 64$ grid)} \\
		\toprule
		& MSE & Parameters \\\midrule
			FAE architecture & $4.82 \times 10^{-4}$ & 64,857 \vspace{-0.25em}\\
				   & \tiny $\pm 2.57 \times 10^{-5}$ & \\
			CNN architecture & $2.38 \times 10^{-4}$ & 71,553 \vspace{-0.25em}\\
		 & \tiny $\pm 9.43 \times 10^{-6} $ & \\
		\bottomrule
		\end{tabular}
  
            \vspace{0.2em}
			\footnotesize ~~~Mean $\pm$ 1 standard deviation; 5 training runs}
			\vspace{-0.6em}
		\caption{Our resolution-invariant architecture performs comparably to CNNs with similar parameter counts.}
	\label{fig:Navier--Stokes_reconstruction}

\end{table}
We compare the autoencoding performance of our mesh-invariant FAE architecture to a standard fixed-resolution CNN architecture, both trained using the FAE objective with all other hyperparameters the same.
Our goal is to understand whether our architecture is competitive even without the inductive bias of CNNs.

To do this we fix a class of FAE and CNN architectures for $64 \times 64$ unmasked data, and perform a search to select the best-performing models with similar parameter counts (details in \Cref{subsec:details_Navier--Stokes}).
The FAE architecture achieves reconstruction MSE slightly greater than the CNN on the held-out data, with a comparable number of parameters (\Cref{fig:Navier--Stokes_reconstruction}). 
It is reasonable to expect that mesh-invariance of the FAE architecture comes at some cost to performance at fixed resolution, especially as the CNN benefits from a strong inductive bias, but the results of \Cref{fig:Navier--Stokes_reconstruction} suggest that the cost is modest.
Further research is desirable to close this gap through better mesh-invariant architectures.

\paragraph{Inpainting.}
Methods for inpainting---inferring missing parts of an input based on observations and prior knowledge from training data---and related inverse problems are of great interest in computer vision and in scientific applications \citep{Quanetal2024}.
We exploit the ability of FAE to encode on any mesh, and decode on any other mesh, to solve a variety of inpainting tasks.
More precisely, after training FAE, we take data from the held-out set on a $64 \times 64$ grid and, for each discretised sample, we apply one of three possible masks:
\begin{enumerate}[label=(\roman*)]
    \item random masking with point ratio 5\%, i.e., masking 95\% of mesh points; or
	\item masking of all mesh points lying in a square with random centre and side length; or
	\item masking of all mesh points in the $-0.05$-superlevel set of a draw from the Gaussian random field $N(0, (30^{2} I - \Delta)^{-1.2})$, where $\Delta$ is the Laplacian for functions on the torus.
\end{enumerate}
\vspace{0.3em}

\begin{figure}[ht]
	\centering
	\begin{tikzpicture}
		\node[minimum width=5.5cm, minimum height=5.2cm, inner sep=0, outer sep=0] at (9.9, 2.0) {\includegraphics[width=5.9cm, clip, trim=0.3cm 0.0cm 0cm 0cm]{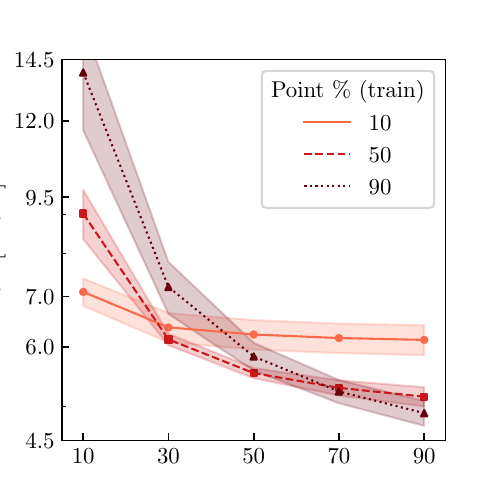}};
		\node at (9.7, -1.02) {\footnotesize Point \% (evaluation)};
		\node at (9.75, 5.4) {\footnotesize (b) Reconstruction MSE [$\times 10^{-4}$] };
        \node at (9.75, 5.0) {\footnotesize Mean $\pm$ 1 standard deviation;};
        \node at (9.75, 4.68) {\footnotesize 5 training runs};

		\node at (2.5, 5.4) {\footnotesize (a) Inpainting with missing mesh points};
		\node[anchor=north] at (0, 5.35) {\footnotesize input};
		\node[anchor=north] at (2.5, 5.33) {\footnotesize reconstruction};
		\node[anchor=north] at (5, 5.35) {\footnotesize ground truth};

		\node[outer sep=0, inner sep=0, minimum width=1.8cm, minimum height=1.8cm] at (0,0) {\includegraphics[height=1.8cm, clip, trim=6cm 0.8cm 6cm 0.8cm]{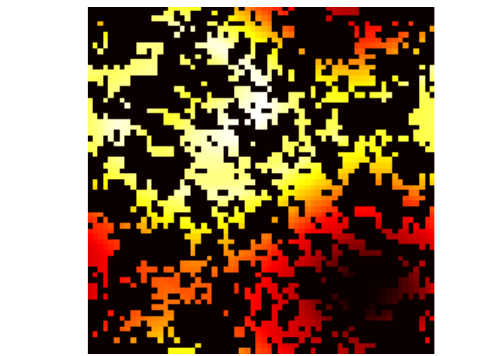}};
		\node[outer sep=0, inner sep=0, minimum width=1.8cm, minimum height=1.8cm] at (2.5,0) {\includegraphics[height=1.8cm, clip, trim=8.9cm 0.9cm 4.5cm 0.9cm]{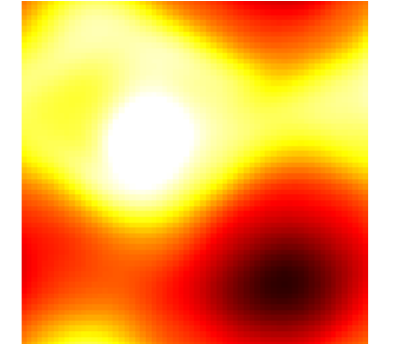}};
		\node[outer sep=0, inner sep=0, minimum width=1.8cm, minimum height=1.8cm] at (5,0) {\includegraphics[height=1.8cm, clip, trim=15.4cm 1cm 14.4cm 1cm]{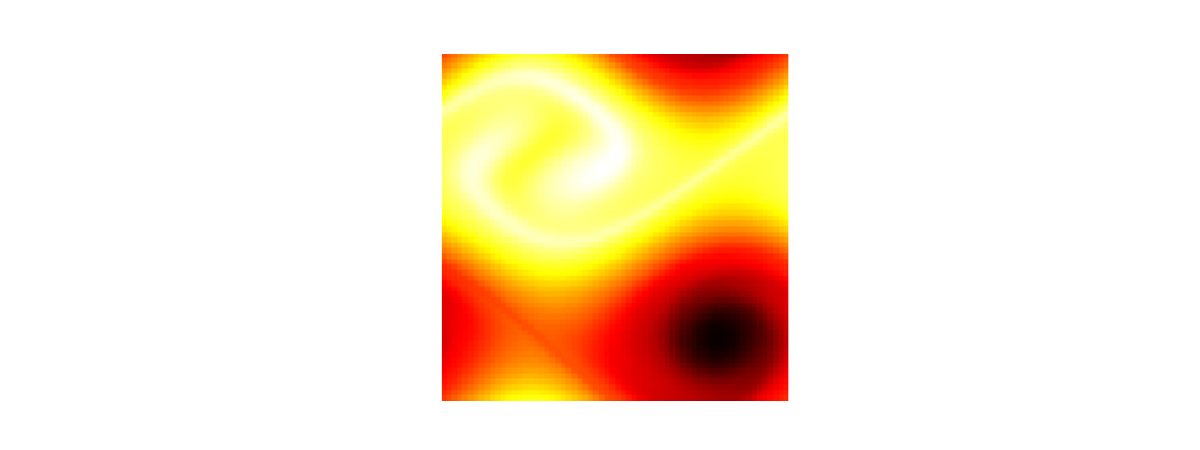}};
		\node[outer sep=0, inner sep=0, minimum width=1.8cm, minimum height=1.8cm] at (6.04,0) {\includegraphics[height=1.8cm, clip, trim=15.5cm 0.9cm 4cm 0.9cm]{images/ns_inpainting/grf/reconstruction}};
		\node at (-1.35, 0) {\footnotesize (iii)};
		\node at (6.35, -0.90) {\tiny 0.2};
		\node at (6.35, -0.90+0.37) {\tiny 0.3};
		\node at (6.35, -0.90+0.37+0.37) {\tiny 0.4};
		\node at (6.35, -0.90+0.37+0.37+0.37) {\tiny 0.5};
		\node at (6.35, -0.90+0.37+0.37+0.37+0.37) {\tiny 0.6};

		\node[outer sep=0, inner sep=0, minimum width=1.8cm, minimum height=1.8cm] at (0,2) {\includegraphics[height=1.8cm, clip, trim=6cm 0.8cm 6cm 0.8cm]{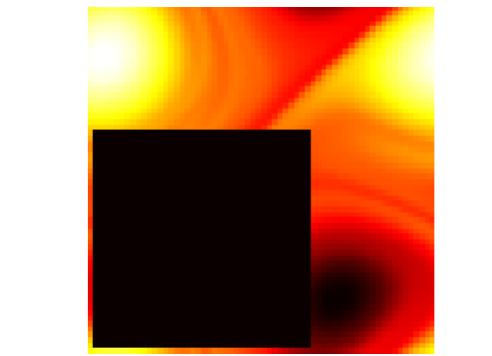}};
		\node[outer sep=0, inner sep=0, minimum width=1.8cm, minimum height=1.8cm] at (2.5,2) {\includegraphics[height=1.8cm, clip, trim=8.9cm 0.9cm 4.5cm 0.9cm]{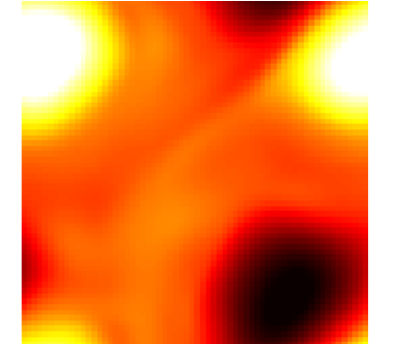}};
		\node[outer sep=0, inner sep=0, minimum width=1.8cm, minimum height=1.8cm] at (5,2) {\includegraphics[height=1.8cm, clip, trim=15.4cm 1cm 14.4cm 1cm]{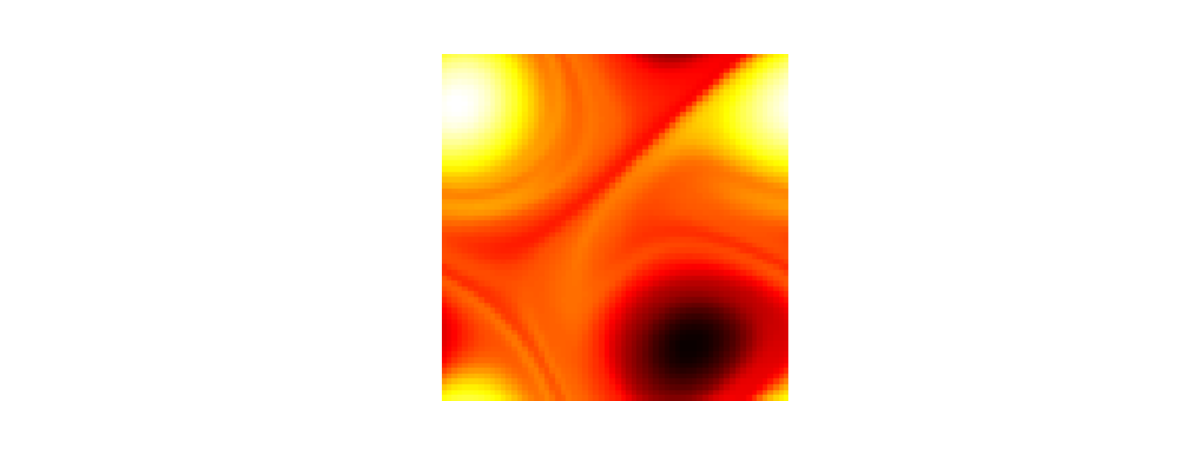}};
		\node[outer sep=0, inner sep=0, minimum width=1.8cm, minimum height=1.8cm] at (6.04,2) {\includegraphics[height=1.8cm, clip, trim=15.5cm 0.9cm 4cm 0.9cm]{images/ns_inpainting/square/reconstruction}};

		\node at (-1.35, 2) {\footnotesize (ii)};
		\node at (6.35, 1.37) {\tiny 0.2};
		\node at (6.35, 1.37+0.42) {\tiny 0.4};
		\node at (6.35, 1.37+0.84) {\tiny 0.6};
		\node at (6.35, 1.37+1.25) {\tiny 0.8};

		\node[outer sep=0, inner sep=0, minimum width=1.8cm, minimum height=1.8cm] at (0,4) {\includegraphics[height=1.8cm, clip, trim=25cm 0.8cm 25cm 0.8cm]{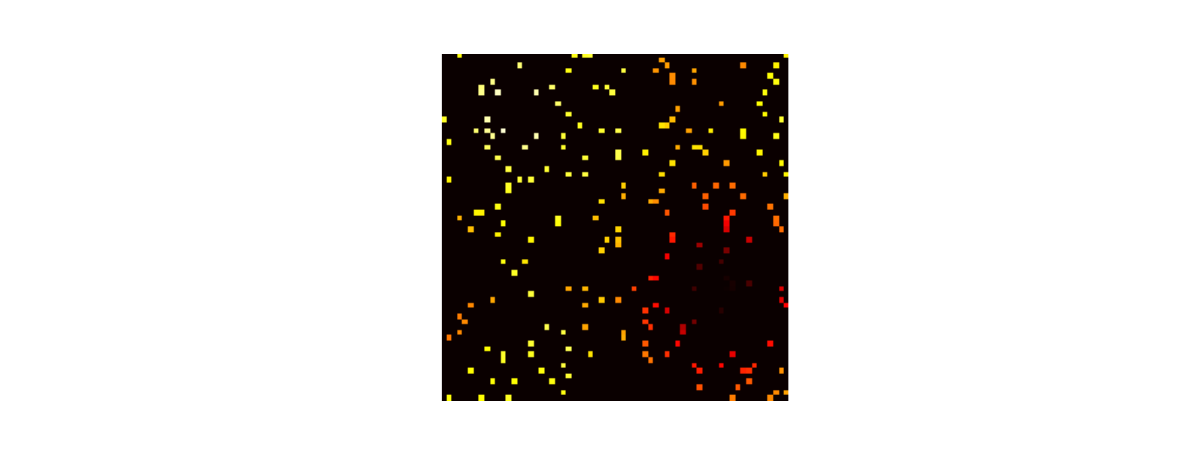}};
		\node[outer sep=0, inner sep=0, minimum width=1.8cm, minimum height=1.8cm] at (2.4,4) {\includegraphics[height=1.8cm, clip, trim=8.5cm 0.9cm 5cm 0.9cm]{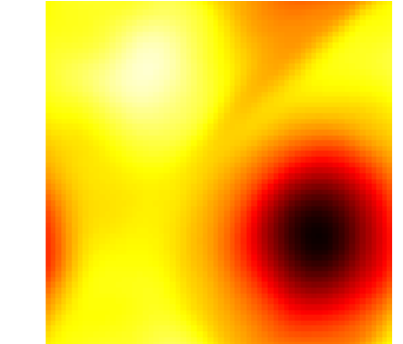}};
		\node[outer sep=0, inner sep=0, minimum width=1.8cm, minimum height=1.8cm] at (6.05, 4) {\includegraphics[height=1.8cm, clip, trim=15.5cm 0.9cm 3.9cm 0.9cm]{images/ns_inpainting/uniform/reconstruction}};
		\node[outer sep=0, inner sep=0, minimum width=1.8cm, minimum height=1.8cm] at (4.87,4) {\includegraphics[height=1.8cm, clip, trim=7cm 1cm 7cm 1cm]{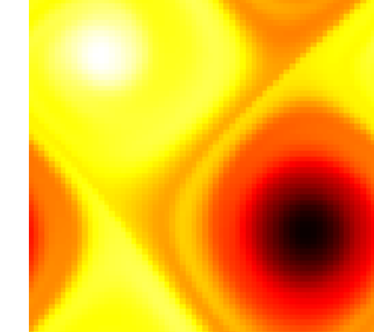}};
		\node at (6.39, 3.14) {\tiny 0.1};
		\node at (6.39, 3.14+0.27) {\tiny 0.2};
		\node at (6.39, 3.14+0.54) {\tiny 0.3};
		\node at (6.39, 3.14+0.54+0.27) {\tiny 0.4};
		\node at (6.39, 3.14+0.54+0.54) {\tiny 0.5};
		\node at (6.39, 3.16+0.54+0.54+0.27) {\tiny 0.6};
		\node at (6.39, 3.14+0.54+0.54+0.54) {\tiny 0.7};
		\node at (-1.3, 4) {\footnotesize (i)};
	\end{tikzpicture}
	\vspace{-0.6em}
	\caption{(a) FAE can solve a variety of inpainting tasks; further samples in \cref{subsec:details_Navier--Stokes}. (b) Passing the encoder a sparse mesh during training is helpful when evaluating on sparse meshes and vice versa.} 
	\label{fig:Navier--Stokes_inpainting}
	\vspace{-0.4em}
\end{figure}

Decoding these samples on a $64 \times 64$ grid (\cref{fig:Navier--Stokes_inpainting}) leads to reconstructions that agree well with the ground truth, and we find FAE to be robust even with a significant amount of the original mesh missing.
As a consequence of the autoencoding procedure, the observed region of the input may also be modified, an effect most pronounced in (ii), where some features in the input are oversmoothed in the reconstruction. 
We hypothesise that the failure to capture fine-scale features could be mitigated with better neural-operator architectures.

To understand the effect of the training point ratio on inpainting quality, we first train instances of FAE with complement masking with point ratios 10\%, 50\%, and 90\%.
Then, for each model, we evaluate its autoencoding performance by applying random masking to each sample from the held-out set with point ratio $r_{\text{eval}} \in \{10\%, 30\%, 50\%, 70\%, 90\%\}$, reconstructing on the full $64 \times 64$ grid, and computing the reconstruction MSE averaged over the held-out set (\cref{fig:Navier--Stokes_inpainting}(b)).
We observe that the best choice of $r_{\text{enc}}$ depends on the point ratio $r_{\text{eval}}$ of the input.
Unsurprisingly, when $r_{\text{eval}}$ is small, it is better to use a model trained with $r_{\text{enc}}$ small, and vice versa.
Further analysis is provided in \cref{subsec:details_Navier--Stokes}.

\paragraph{Superresolution.}
The ability to encode and decode on different meshes also enables the use of FAE for single-image superresolution: high-resolution reconstruction of a low-resolution input.
Superresolution methods based on deep learning have found use in varied applications including imaging \citep{Lietal2024} and fluid dynamics, e.g., in increasing the resolution (``downscaling'') of numerical simulations \citep{Kochkovetal2021,BischoffDeck2024}. 
Generalising other continuous superresolution models such as the Local Implicit Image Function \citep{Chenetal2021}, a single trained FAE model can be applied with any upsampling factor, and FAE has the further advantage of accepting inputs at any resolution.

\begin{figure}[ht]
	\centering
	\begin{tikzpicture}
        \node[anchor=center] at (10.3, 2.9) {\footnotesize (b)(ii) Number of decoder network evaluations}; \node[anchor=center] at (9.55, 2.6) {\footnotesize for superresolution on:};
        \node at (11.0, 2) {\begin{tikzpicture}

        \node[anchor=west] at (4.3, 1.25) {\footnotesize full grid};
        \node[anchor=west, minimum width=3cm, minimum height=0.3cm, draw] at (5.9, 1.2) {};
        \node[anchor=west] at (8.9, 1.2) {\footnotesize 160k};

		\node[anchor=west] at (4.3, 0.85) {\footnotesize patch};
        \node[anchor=west, inner sep=0, outer sep=0, minimum width=0.7125cm, text width=0.375cm, minimum height=0.3cm, draw] at (5.9, 0.8) {};
		\node[anchor=west] at (8.9, 0.8) {\footnotesize 38k};
        \end{tikzpicture}
        };
    
		\node at (10.5, 4.8) {
\begin{tikzpicture}
		\node[minimum width=2.6cm, minimum height=2.6cm, anchor=east, outer sep=0, inner sep=0] at (0, 0.42) {\includegraphics[width=2.3cm, clip, trim=14cm 1.2cm 13.65cm 1.2cm]{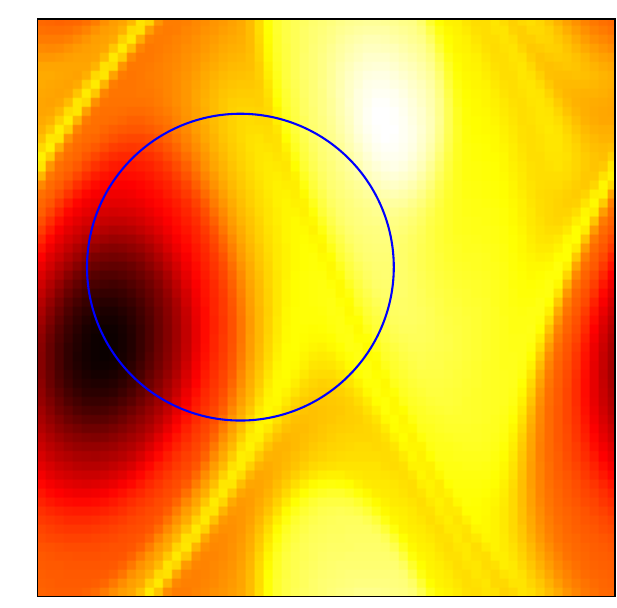}};
        \node at (-1.3, -0.95) {\footnotesize input data on };
        \node at (-1.3, -1.25) {\footnotesize coarse mesh ($64 \times 64$)};

        \draw[-stealth,blue] (-0.8, 0.7) to [bend left=8] (0.6, 0.6);
		\node[outer sep=0, inner sep=0, minimum width=2.4cm, minimum height=2.4cm, anchor=west] at (0.6, 0.45) {\includegraphics[width=2.3cm, clip, trim=18cm 1cm 9.5cm 1cm]{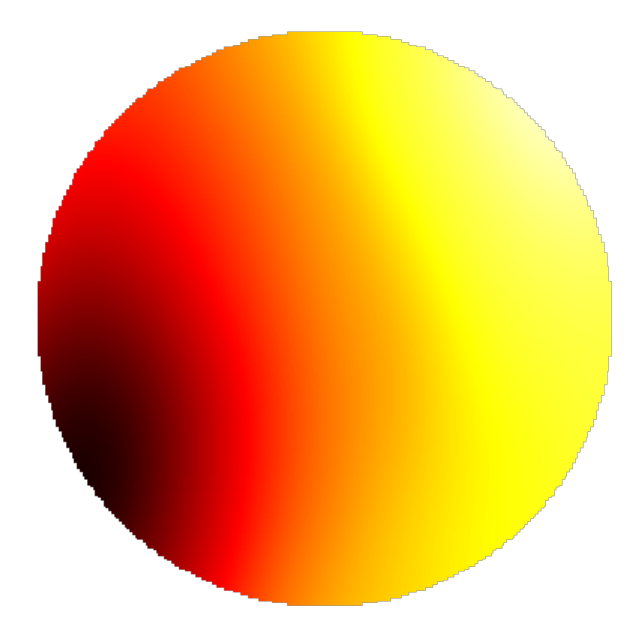}};
		\node[outer sep=0, inner sep=0, minimum width=2.4cm, minimum height=2.4cm, anchor=west] at (2, 0.55) {\includegraphics[height=2.4cm, clip, trim=29.5cm 1cm 7.5cm 1cm]{images/ns_patch_superresolution/reconstructed_patch}};
        \node at (0.68+1.25, -0.91) {\footnotesize refine in specific region};
		\node at (0.68+1.25, -1.21) {\footnotesize (mesh spacing $\nicefrac{1}{400}$)};

        \node at (3.5, -0.55) {\tiny 0.1};
		\node at (3.5, -0.55+0.33) {\tiny 0.2};
		\node at (3.5, -0.55+0.67) {\tiny 0.3};
		\node at (3.5, -0.55+1) {\tiny 0.4};
		\node at (3.5, -0.55+0.99+0.36) {\tiny 0.5};
		\node at (3.5, -0.55+0.99+0.7) {\tiny 0.6};
		\node at (3.5, -0.55+0.99+1.04) {\tiny 0.7};
    \end{tikzpicture}
        
        };

		\node at (2.5, 6.5) {\footnotesize (a) Examples of data-driven superresolution};
        \node[anchor=center] at (10.5, 6.5) {\footnotesize (b)(i) Superresolution on patches};
		\node at (0, 6.15) {\footnotesize input};
		\node at (2.5, 6.15) {\footnotesize reconstruction};
		\node at (5, 6.15) {\footnotesize ground truth};

		\node[minimum width=2cm, minimum height=2cm] at (0,2.65) {\includegraphics[height=2cm, clip, trim=8cm 1cm 7.3cm 1cm]{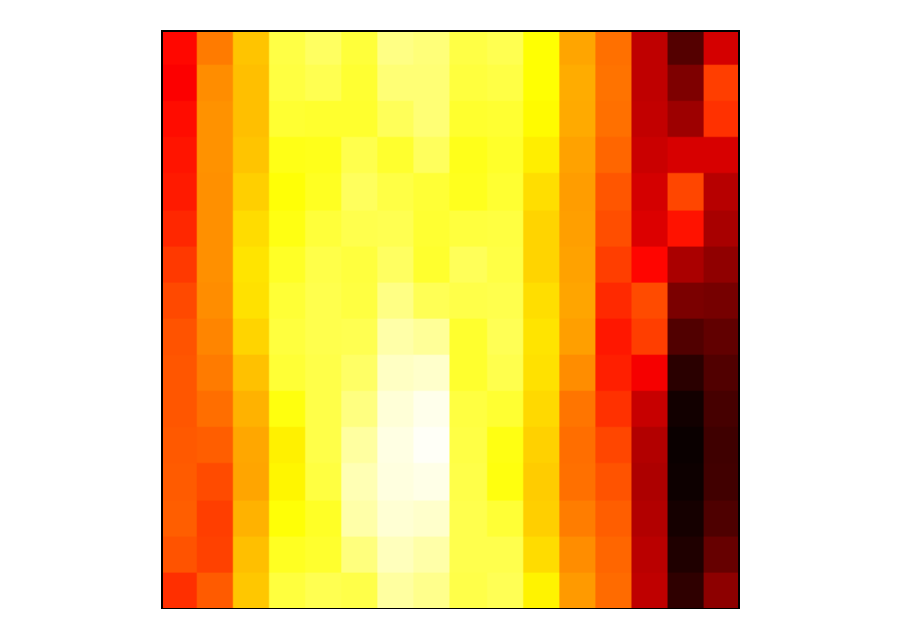}};
		\node[minimum width=2cm, minimum height=2cm] at (2.5,2.65) {\includegraphics[height=2cm, clip, trim=12.5cm 1cm 7.3cm 1cm]{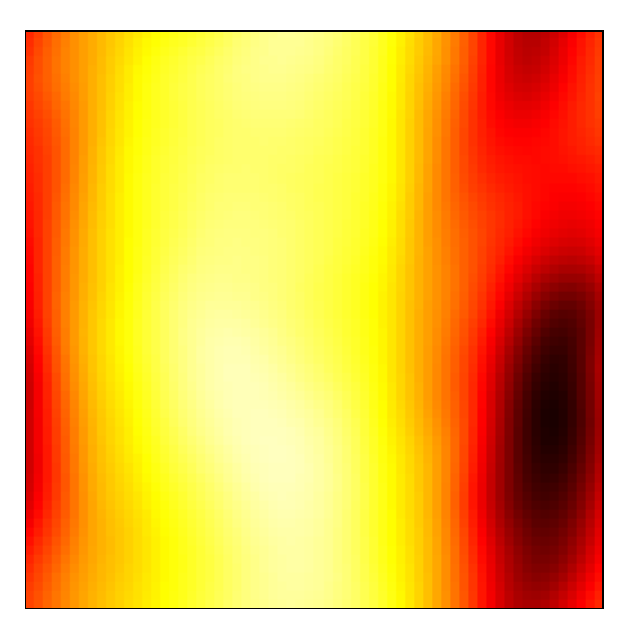}};
		\node[minimum width=2cm, minimum height=2cm] at (5,2.65) {\includegraphics[height=2cm, clip, trim=8cm 1cm 7.3cm 1cm]{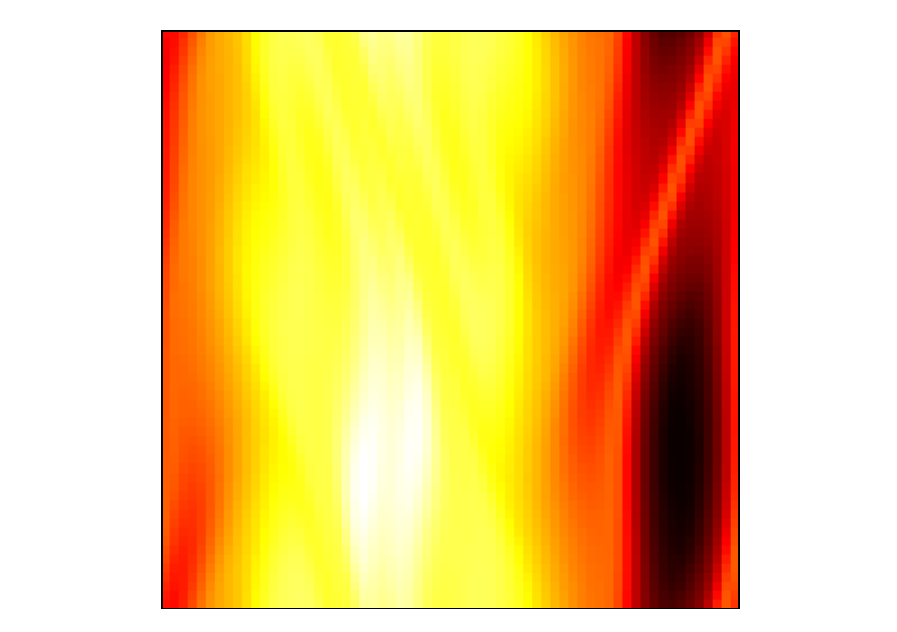}};
		\node[minimum width=2cm, minimum height=2cm] at (6.1,2.65) {\includegraphics[height=2cm, clip, trim=23.7cm 1cm 5.9cm 1cm]{images/ns_superresolution/16/reconstruction}};
		\node at (6.35, 1.86) {\tiny 0.2};
		\node at (6.35, 1.86+0.33) {\tiny 0.3};
		\node at (6.35, 1.86+0.66) {\tiny 0.4};
		\node at (6.35, 1.86+1.01) {\tiny 0.5};
		\node at (6.35, 1.86+1.02+0.33) {\tiny 0.6};
		\node at (6.35, 1.86+1.02+0.66) {\tiny 0.7};
		\node at (0, 1.43) {\footnotesize $16 \times 16$};
		\node at (2.5, 1.43) {\footnotesize $64 \times 64$};
		\node at (5, 1.43) {\footnotesize $64 \times 64$};

		\node[minimum width=2cm, minimum height=2cm] at (0,5) {\includegraphics[height=2cm, clip, trim=8cm 1cm 7.3cm 1cm]{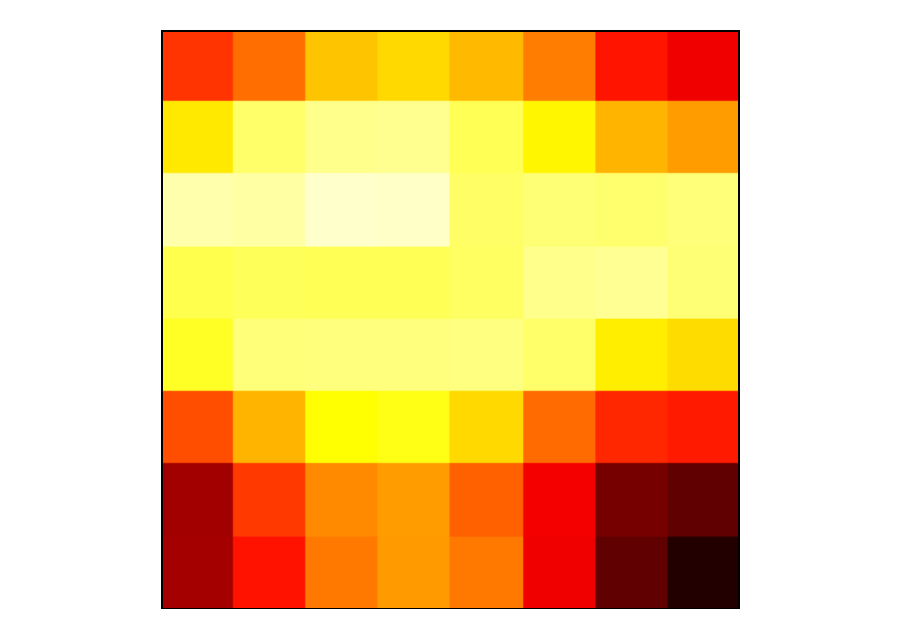}};
		\node[minimum width=2cm, minimum height=2cm] at (2.5,5) {\includegraphics[height=2cm, clip, trim=12.5cm 1cm 7.3cm 1cm]{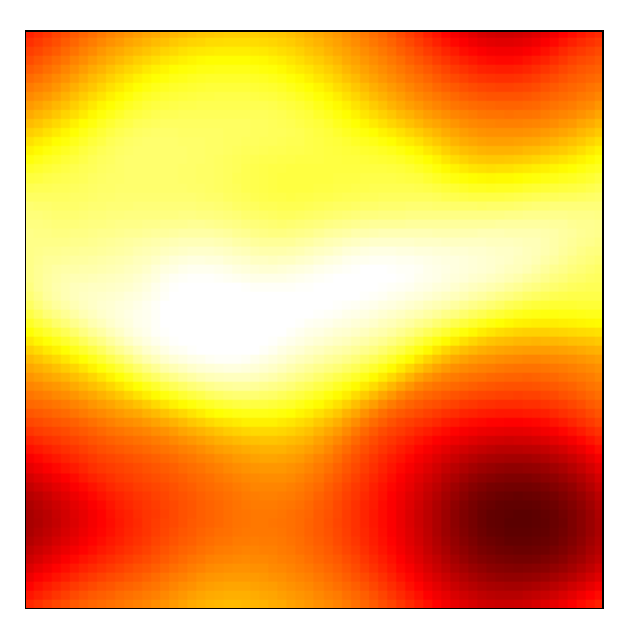}};
		\node[minimum width=2cm, minimum height=2cm] at (5,5) {\includegraphics[height=2cm, clip, trim=8cm 1cm 7.3cm 1cm]{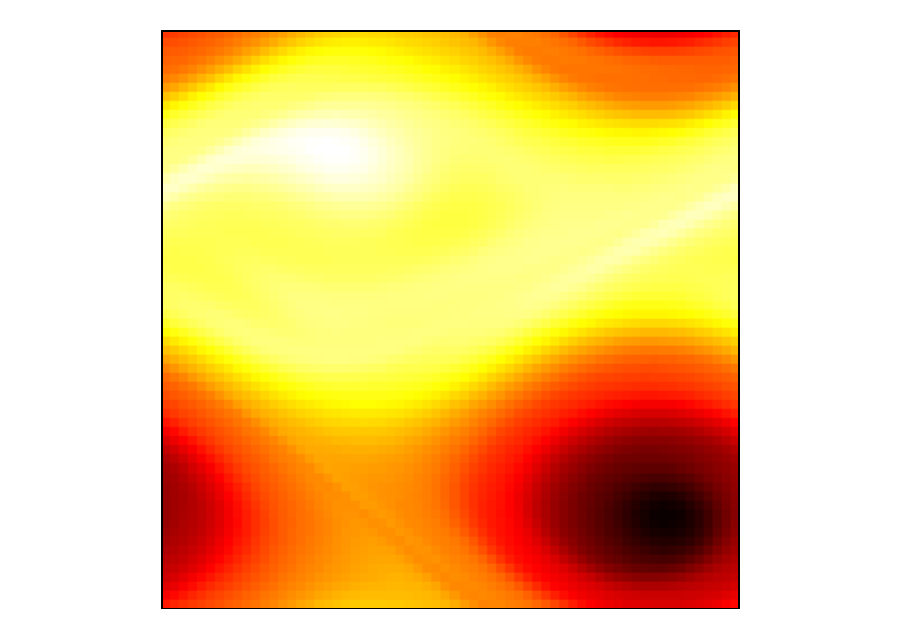}};
		\node[minimum width=2cm, minimum height=2cm] at (6.1,5) {\includegraphics[height=2cm, clip, trim=23.7cm 1cm 5.9cm 1cm]{images/ns_superresolution/8/reconstruction}};
		\node at (0, 3.83) {\footnotesize $8 \times 8$};
		\node at (2.5, 3.83) {\footnotesize $64 \times 64$};
		\node at (5, 3.83) {\footnotesize  $64 \times 64$};
		\node at (6.35, 4.28) {\tiny 0.2};
		\node at (6.35, 4.28+0.33) {\tiny 0.3};
		\node at (6.35, 4.28+0.66) {\tiny 0.4};
		\node at (6.35, 4.28+0.99) {\tiny 0.5};
		\node at (6.35, 4.28+1.32) {\tiny 0.6};

		\node at (-1.22, 5) {\footnotesize (i)};
		\node at (-1.22, 2.65) {\footnotesize (ii)};
	\end{tikzpicture}
	\vspace{-1.5em}
	\caption{(a) FAE can encode low-resolution inputs and decode at higher resolution, recovering fine-scale features using knowledge of the underlying data. Further examples in \cref{subsec:details_Navier--Stokes}. (b) Evaluating the decoder in a specific subregion can lead to significant computational savings compared to performing superresolution on the full grid.}
	\label{fig:Navier--Stokes_superres}
\end{figure}

For \defterm{data-driven superresolution}---where we train a model at high resolution and use it to enhance low-resolution inputs at inference time---FAE is able to resolve unseen features from $8 \times 8$ and $16 \times 16$ inputs on a $64 \times 64$ output grid after training at resolution $64 \times 64$ (\cref{fig:Navier--Stokes_superres}(a)).
As with inpainting, superresolution performance could be further improved with an architecture that is better able to capture the turbulent dynamics in the data.

We also investigate the stability of FAE for \defterm{zero-shot superresolution} \citep{Lietal2021}, where the model is evaluated on higher resolutions than seen during training.
Since FAE is purely data-driven, we view this as a test of the model's mesh-invariance and do not expect to resolve high-frequency features that were not seen during training.
Our architecture proves robust when autoencoding on meshes much finer than the original $64 \times 64$ training grid (\cref{fig:Navier--Stokes_zero-shot_and_latent_interpolation}(a));
moreover our coordinate MLP architecture allows us to decode on extremely fine meshes without exhausting the GPU memory (details in \cref{subsec:details_Navier--Stokes}).
While zero-shot superresolution is possible with VANO when the input is given on the mesh seen during training, FAE can be used for superresolution with any input.

\paragraph{Efficient Superresolution on Regions of Interest.}
Since our decoder can be evaluated on any mesh, we can perform superresolution in a specific subregion without upsampling across the whole domain. 
Doing this can significantly reduce inference time, memory usage, and energy cost.
As an example, we consider the task of reconstructing a circular subregion of interest with target mesh spacing $\nicefrac{1}{400}$ (\cref{fig:Navier--Stokes_superres}(b)(i)). 
Achieving this resolution over the whole domain---corresponding to a $400 \times 400$ grid---would involve 160,000 evaluations of the decoder network;
decoding on the subregion requires just $\nicefrac{1}{4}$ of this (\cref{fig:Navier--Stokes_superres}(b)(ii)).

\paragraph{Applications of the Latent Space $\latentspace$.}

\begin{figure}[ht]
	\centering
	\vspace{-1em}
	\begin{tikzpicture}
		\node[minimum width=2.5cm, minimum height=2.5cm, inner sep=0, outer sep=0] at (-0, -0.35) {\includegraphics[height=2.5cm,clip,trim=3cm 0.9cm 3cm 0.5cm]{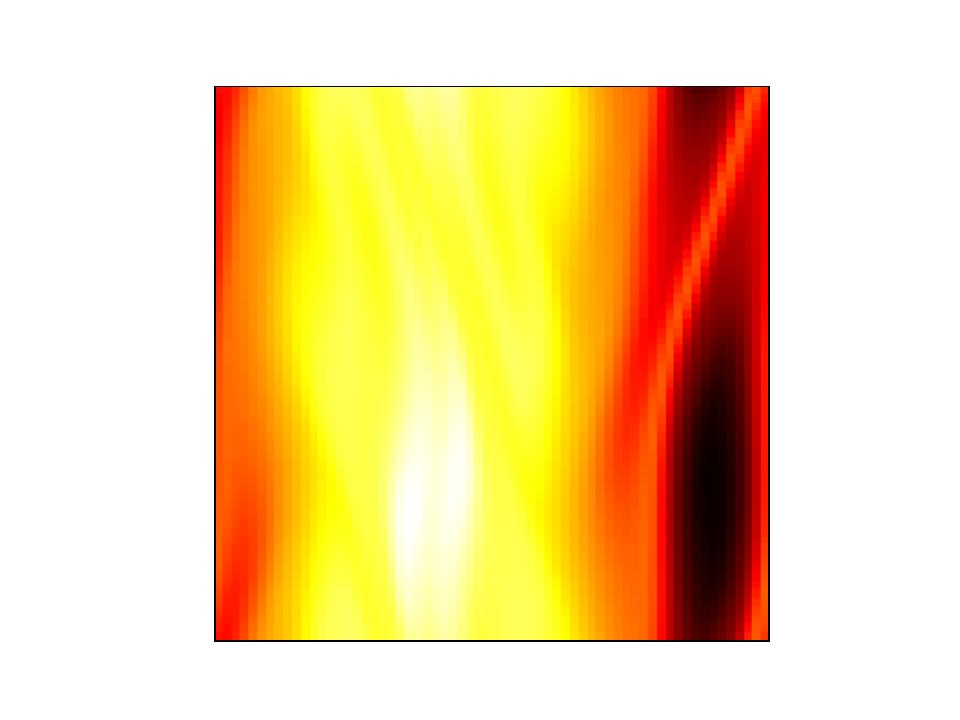}};
            \draw[-stealth] (1.5, -0.35) -- (2.5, -0.35);
            \node at (2, -0.6) {\footnotesize $32\times$};
            \node at (2, -0.88) {\footnotesize increase};
		\node at (-0, -1.85) {\small $64 \times 64$};
		\node[minimum width=2.5cm, minimum height=2.5cm, inner sep=0, outer sep=0] at (4, -0.35) {\includegraphics[height=2.5cm,clip,trim=2.7cm 0.9cm 2.9cm 0.5cm]{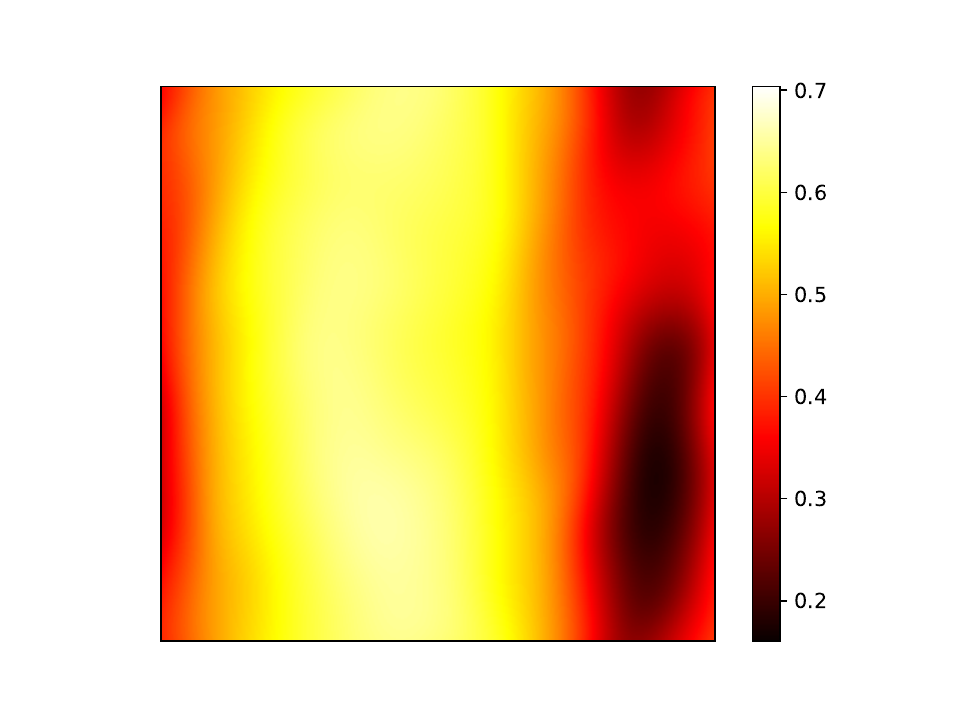}};
		\node at (3.95, -1.85) {\small $\text{2,048} \times \text{2,048}$};
		\node at (5.45, -1.34) {\tiny 0.2};
		\node at (5.45, -1.34+0.4) {\tiny 0.3};
		\node at (5.45, -1.34+0.8) {\tiny 0.4};
		\node at (5.45, -1.34+1.2) {\tiny 0.5};
		\node at (5.45, -1.34+1.6) {\tiny 0.6};
		\node at (5.45, -1.34+2) {\tiny 0.7};
  
		\node[minimum width=2.5cm, minimum height=2.5cm, inner sep=0, outer sep=0] at (7.2, -0.35) {\includegraphics[height=2.5cm,clip,trim=3cm 0.9cm 3cm 0.5cm]{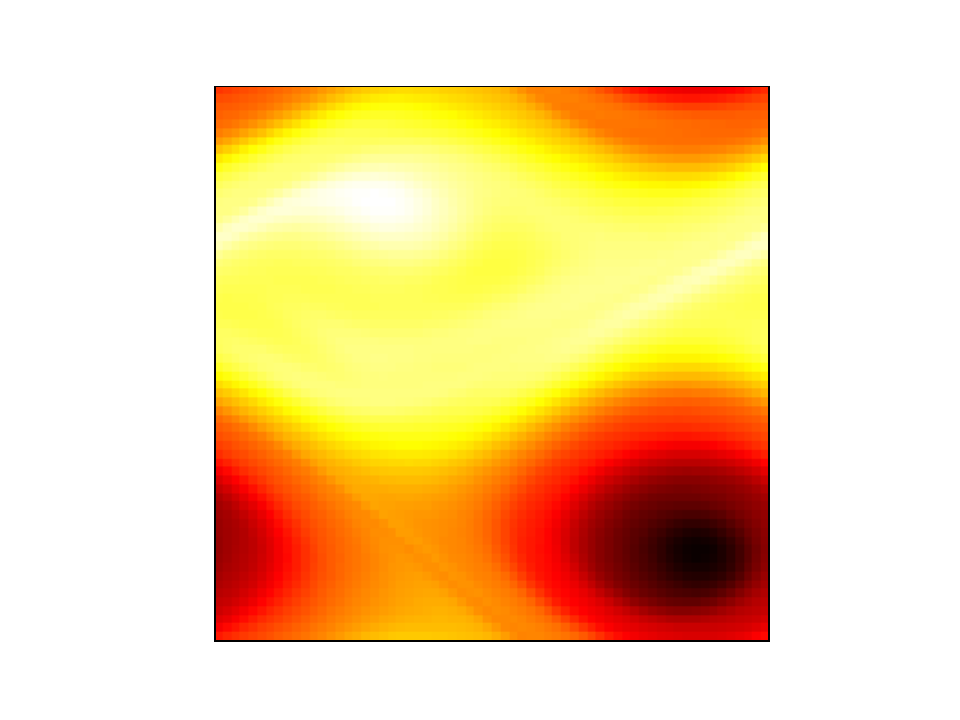}};
		\node at (7.2, -1.85) {\small $64 \times 64$};
             \draw[-stealth] (8.7, -0.35) -- (9.7, -0.35);
            \node at (9.2, -0.6) {\footnotesize $512\times$};
            \node at (9.2, -0.88) {\footnotesize increase};
		\node[minimum width=2.5cm, minimum height=2.5cm, inner sep=0, outer sep=0] at (11.2, -0.35) {\includegraphics[height=2.5cm,clip,trim=2.7cm 0.9cm 2.8cm 0.5cm]{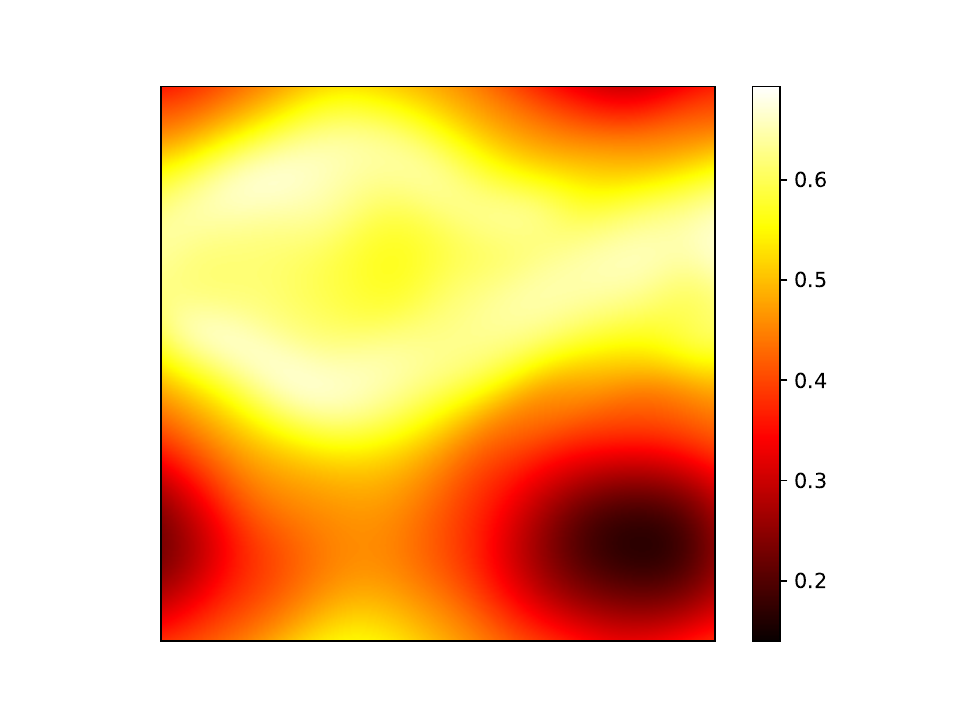}};
		\node at (12.65, -1.275) {\tiny 0.2};
		\node at (12.65, -1.275+0.4) {\tiny 0.3};
		\node at (12.65, -1.275+0.8) {\tiny 0.4};
		\node at (12.65, -1.275+1.2) {\tiny 0.5};
		\node at (12.65, -1.275+1.6) {\tiny 0.6};
		\node at (11.1, -1.85) {\small $\text{32,768} \times \text{32,768}$};

		\node[anchor=center] at (5.5, -2.25) {\small (b) Latent interpolation $\decodermap\bigl(\encodermean(u_{1}; \encoderparam) \alpha + \encodermean(u_{2}; \encoderparam) (1-\alpha); \decoderparam \bigr)$};
		\node[minimum width=1.75cm, minimum height=1.75cm, inner sep=0, outer sep=0] at (-0.3, -3.45) {\includegraphics[height=1.75cm, clip, trim=22.5cm 1cm 21.5cm 1cm]{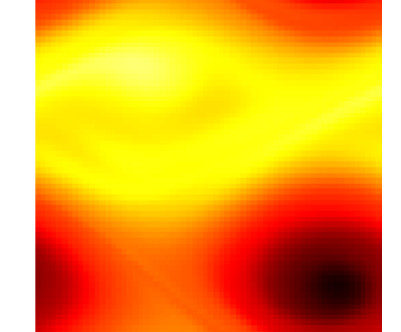}};
		\node[anchor=center] at (-0.3, -4.6) {$u_{1}$};
		\node[minimum width=1.75cm, minimum height=1.75cm, inner sep=0, outer sep=0] at (1.8, -3.45) {\includegraphics[height=1.75cm, clip, trim=22.5cm 1cm 21.5cm 1cm]{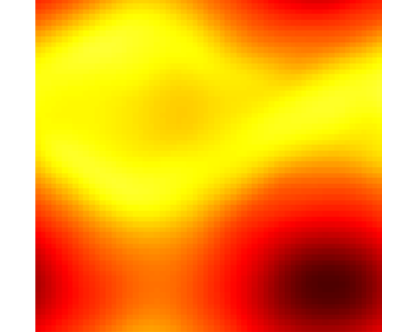}};
		\node[anchor=center] at (1.8, -4.55) {$\alpha = 0.1$};
		\node[minimum width=1.75cm, minimum height=1.75cm, inner sep=0, outer sep=0] at (3.7, -3.45) {\includegraphics[height=1.75cm, clip, trim=22.5cm 1cm 21.5cm 1cm]{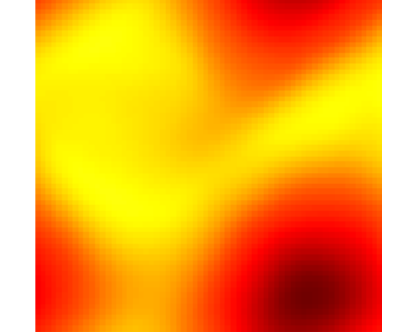}};
		\node[anchor=center] at (3.7, -4.55) {$\alpha = 0.3$};
		\node[minimum width=1.75cm, minimum height=1.75cm, inner sep=0, outer sep=0] at (5.6, -3.45) {\includegraphics[height=1.75cm, clip, trim=22.5cm 1cm 21.5cm 1cm]{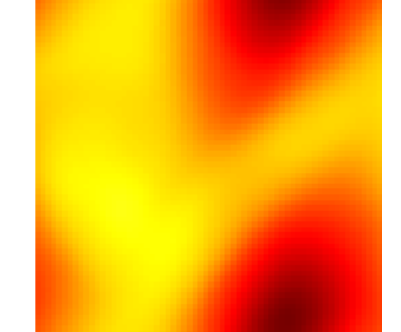}};
		\node[anchor=center] at (5.6, -4.55) {$\alpha = 0.5$};
		\node[minimum width=1.75cm, minimum height=1.75cm, inner sep=0, outer sep=0] at (7.5, -3.45) {\includegraphics[height=1.75cm, clip, trim=22.5cm 1cm 21.5cm 1cm]{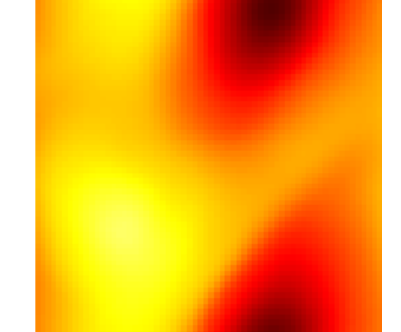}};
		\node[anchor=center] at (7.5, -4.55) {$\alpha = 0.7$};
		\node[minimum width=1.75cm, minimum height=1.75cm, inner sep=0, outer sep=0] at (9.4, -3.45) {\includegraphics[height=1.75cm, clip, trim=22.5cm 1cm 21.5cm 1cm]{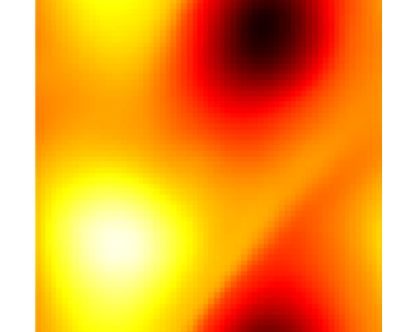}};
		\node[anchor=center] at (9.4, -4.55) {$\alpha = 0.9$};

		\node[minimum width=1.75cm, minimum height=1.75cm, inner sep=0, outer sep=0] at (11.5, -3.45) {\includegraphics[height=1.75cm, clip, trim=31.5cm 1cm 12.5cm 1cm]{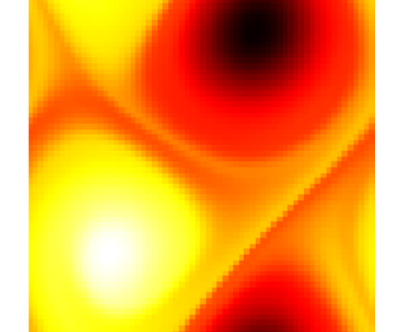}};
		\node[minimum width=1.75cm, minimum height=1.75cm, inner sep=0, outer sep=0] at (12.6, -3.44) {\includegraphics[height=1.88cm, clip, trim=39.5cm 0.6cm 10.5cm 0.6cm]{images/ns_latent_interpolation/end_true}};
		\node[anchor=center] at (11.5, -4.6) {$u_{2}$};
		\node at (12.9, -4.12) {\tiny 0.2};
		\node at (12.9, -4.12+0.26) {\tiny 0.3};
		\node at (12.9, -4.12+0.52) {\tiny 0.4};
		\node at (12.9, -4.12+0.78) {\tiny 0.5};
		\node at (12.9, -4.12+0.78+0.26) {\tiny 0.6};
		\node at (12.9, -4.12+0.78+0.52) {\tiny 0.7};
		\node[anchor=center] at (5.5, 0.95) {\small (a) Autoencoding beyond training resolution with zero-shot superresolution};
	\end{tikzpicture}
	\vspace{-0.7em}
	\caption{(a) FAE can stably decode at resolutions much higher than the training resolution (best viewed digitally). (b) The regularised latent space $\latentspace$ allows for meaningful interpolation between samples. Further examples are given in \cref{subsec:details_Navier--Stokes}.}
	\label{fig:Navier--Stokes_zero-shot_and_latent_interpolation}
	\vspace{-0.5em}
\end{figure}

The regularised FAE latent space gives a well-structured finite representation of the infinite-dimensional data $u \in \dataspace$.
We expect there to be benefit in using this representation as a building block for applications such as supervised learning and generative modelling on functional data, similar in spirit to other supervised operator-learning methods with encoder--decoder structure \citep{SeidmanKissasPerdikarisPappas2022}.

As a first step towards verifying that the latent space does indeed capture useful structure beyond mere memorisation of the training data, we draw $u_{1}$ and $u_{2}$ from the held-out set, compute latent vectors $z_{1} = \encodermean(u_{1}; \encoderparam)$ and $z_{2} = \encodermean(u_{2}; \encoderparam) \in \latentspace$, and evaluate the decoder $\decodermap$ along the convex combination $z_{1} \alpha +  z_{2} (1-\alpha)$.
This leads to a sensible interpolation in $\dataspace$, suggesting that the latent representation is robust and well-regularised (\cref{fig:Navier--Stokes_zero-shot_and_latent_interpolation}(b)).

\subsubsection{Darcy Flow}
\label{subsec:Darcy}

Darcy flow is a model of steady-state flow in a porous medium, derivable from first principles using homogenisation;
see, e.g., \citet[Sec.~2.11]{FreezeCherry1979} and \citet{Keller1980}.
We restrict attention to the two-dimensional domain $\Omega = [0, 1]^{2}$ and suppose that, for some permeability field $k \colon \Omega \to \Reals$ and forcing $\varphi \colon \Omega \to \Reals$, the pressure field $p \colon \Omega \to \Reals$ satisfies
\begin{equation} \label{eq:Darcy_flow}
	\begin{alignedat}{2}
		- \nabla \cdot \bigl( k \nabla p \bigr) &= \varphi &&\text{~~~on $\Omega$,} \\
		p &= 0 &&\text{~~~on $\partial \Omega$.}
	\end{alignedat}
\end{equation}
We assume $\varphi = 1$ and that $k$ is distributed as the pushforward of the distribution $N\bigl(0, (-\Delta + 9I)^{-2}\bigr)$, where $\Delta$ is the Laplacian restricted to functions defined on $\Omega$  with zero Neumann data on $\partial \Omega$, under the map $x \mapsto 3 + 9 \cdot \one \bigl[x \geq 0\bigr]$.
We take $\dataspace = L^{2}(\Omega)$ and define $\datameas \in \prob{\dataspace}$ to be the distribution of pressure fields $p$ solving \eqref{eq:Darcy_flow} with permeability $k$.
While solutions to this elliptic PDE can be expected to have greater smoothness \citep[Sec.~6.3]{Evans2010}, we assume only that $p \in L^{2}(\Omega)$ and use the $L^{2}$-norm in the FAE objective \eqref{eq:FAE_objective}.

The training data set is based on that of \citet{Lietal2021} and consists of 1,024 samples from $\datameas$ on a $421\times421$ grid, with a further 1,024 samples held out as an evaluation set.
Data are scaled so that $p(x) \in [0, 1]$ for all $x \in \Omega$ and, where specified, we downsample as described in \cref{subsec:details_Darcy_flow}.
We train FAE with $d_{\latentspace} = 64$ and $\beta = 10^{-3}$, and use complement masking with a point ratio $r_{\text{enc}}$ of 70\%.

\begin{wrapfigure}{r}{0.46\textwidth}
    \centering
    \vspace{0em}
		\begin{tikzpicture}
			\node at (0.2, 0.02) {\includegraphics[clip, trim=0.3cm 0.14cm 1.1cm 0.6cm, width=7.1cm]{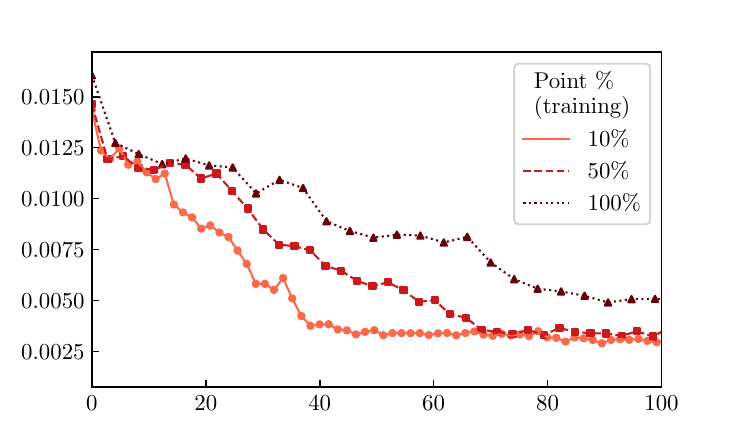}};
			\node at (0.4, 2.6) {\footnotesize Reconstruction MSE on held-out set ($211 \times 211$)};
			\node at (0.4, 2.25) {\footnotesize (mean over $5$ training runs)};
			\node at (0.4, -2.18) {\footnotesize Wall-clock time [s]};
			\node[fill=white, minimum width=1cm, minimum height=0.5cm] at (3.95, -2) {};
		\end{tikzpicture}
	\vspace{-2.3em}
	\caption{Training with masking in the encoder and decoder reduces training time.}
	\label{fig:Darcy_wallclock}
	\vspace{-0.7em}
\end{wrapfigure}

\paragraph{Accelerating Training Using Masking.}
As well as improving reconstructions and robustness to mesh changes, masked training can greatly reduce the cost of training.
To illustrate this we compare the training dynamics of FAE on data downsampled to resolution $211 \times 211$, using random masking with point ratio $r_{\text{enc}} = r_{\text{dec}} \in \{10\%, 50\%, 90\%\}$.
Since the evaluation cost of the encoder and decoder scales linearly with the number of mesh points, we expect significant computational gains when using low point ratios.
We perform five training runs for each model and compute the average reconstruction MSE over time on held-out data at resolution $211 \times 211$.
The models trained with masking converge faster as the smaller data tensors allow for better use of the GPU parallelism (\cref{fig:Darcy_wallclock}).
At higher resolutions, memory constraints may preclude training on the full grid, making masking vital.

Related ideas are used in the adaptive-subsampling training scheme for FNOs proposed by \citet{LanthalerStuartTrautner2024}, which involves training first on a coarse grid and refining the mesh each time the evaluation metric plateaus;
our approach differs by dropping mesh points randomly, which would not be possible with FNO.
One can readily imagine training FAE with a combination of adaptive subsampling and masking.

\paragraph{Generative Modelling.}
While FAE is not itself a generative model, it can be made so by training a fixed-dimension generative model on the latent space $\latentspace$ \citep{Ghoshetal2020,VahdatKreisKautz2021}.
More precisely, we know that applying the FAE encoder $\encodermean$ to data induces a distribution $\pushforwardmeas^{\encoderparam} \in \prob{\latentspace}$ for $z$ given by
\begin{equation} \label{eq:FAE_latent_distribution}
	z \mid u = \encodermean(u; \encoderparam),\quad u \sim \datameas.
\end{equation}
Unlike with FVAE, there is no reason that this should be close to Gaussian.
However, we can approximate $\pushforwardmeas^{\encoderparam}$
with a fixed-resolution generative model $\latentmeas^{\latentgenparam} \in \prob{\latentspace}$ parametrised by $\latentgenparam \in \Latentgenparam$, and 
define the FAE generative model $\genmeas^{\decoderparam, \latentgenparam}$ for data $u$ by
\begin{equation} \label{eq:FAE_generative_model}
	\text{(FAE generative model)}~~~~~~u \mid z = \decodermap(z; \decoderparam),\quad z \sim \latentmeas^{\latentgenparam}.
\end{equation}
Since applying the decoder to $\pushforwardmeas^{\encoderparam}$ should approximately recover the data distribution if $\decodermap(\encodermean(u; \encoderparam); \decoderparam) \approx u$ for $u \sim \datameas$, we hope that when $\pushforwardmeas^{\encoderparam} \approx \latentmeas^{\latentgenparam}$, samples from \eqref{eq:FAE_generative_model} will be approximately distributed according to $\datameas$.
As a simple illustration, we train FAE at resolution $47 \times 47$ and fit a Gaussian mixture model $\latentmeas^{\latentgenparam}$ with 10 components to $\pushforwardmeas^{\encoderparam}$ using the expectation-maximisation algorithm \citep[see][Sec.~9.2.2]{Bishop2006}. 
Samples from \eqref{eq:FAE_generative_model} closely resemble those from the held-out data set (\cref{fig:Darcy_generative}(a)), and as a result of our mesh-invariant architecture, it is possible to generate new samples on any mesh.

\begin{figure}[thb]
	\vspace{-0.5em}
		\begin{tikzpicture}
                \node at (0, 0) {\begin{tikzpicture}
			\node[anchor=center] at (2.63, 1) {\footnotesize (a)(i) FAE samples $\decodermap(z; \decoderparam)$, $z \sim \latentmeas$ ($47 \times 47$ grid)};
			\node[minimum width=1.5cm, minimum height=1.5cm, outer sep=0, inner sep=0] at (0, 0) {\includegraphics[width=1.5cm, clip, trim=13cm 1.2cm 12.3cm 1.2cm]{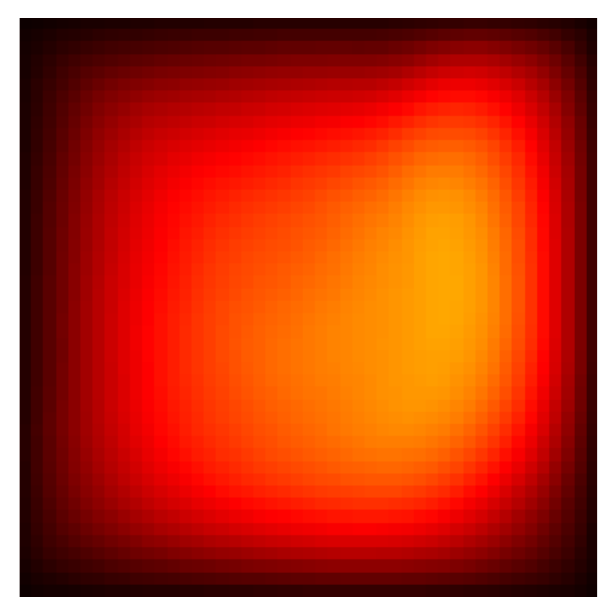}};
			\node[minimum width=1.5cm, minimum height=1.5cm, outer sep=0, inner sep=0] at (1.7, 0) {\includegraphics[width=1.5cm, clip, trim=13cm 1.2cm 12.3cm 1.2cm]{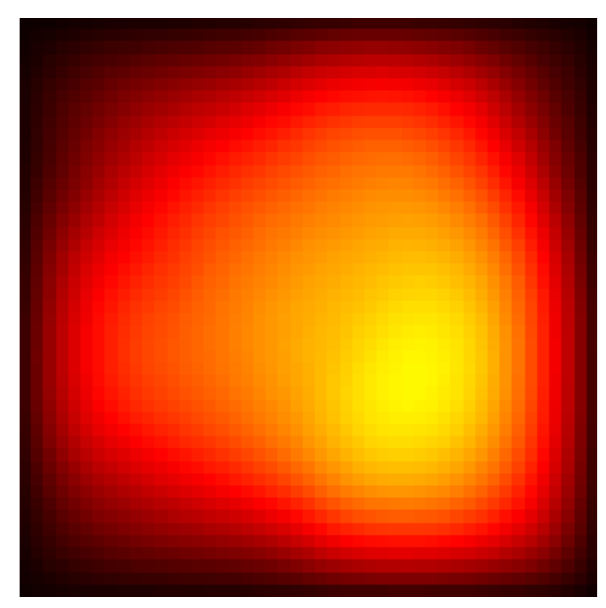}};
			\node[minimum width=1.5cm, minimum height=1.5cm, outer sep=0, inner sep=0] at (3.4, 0) {\includegraphics[width=1.5cm, clip, trim=13cm 1.2cm 12.3cm 1.2cm]{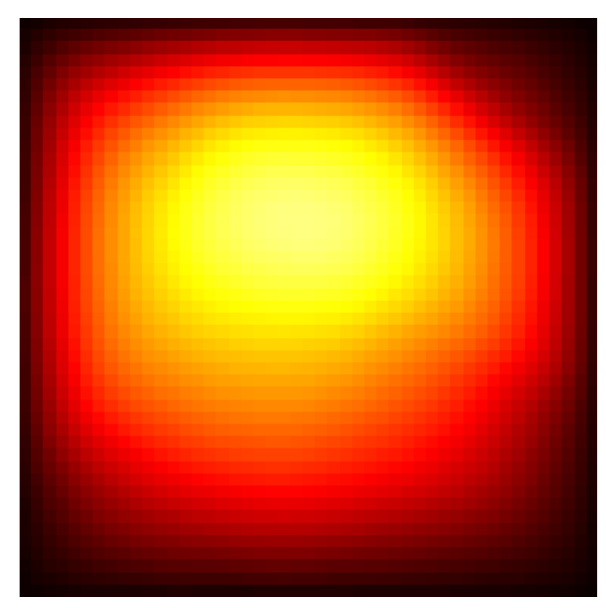}};
			\node[minimum width=1.5cm, minimum height=1.5cm, outer sep=0, inner sep=0] at (5.1, 0) {\includegraphics[width=1.5cm, clip, trim=13cm 1.2cm 12.3cm 1.2cm]{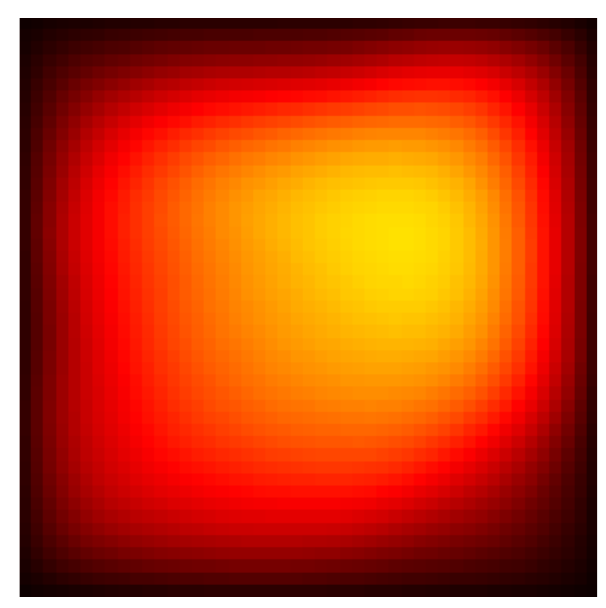}};
			\node at (6.1, -0.02){\begin{tikzpicture}
					\node at (6.1, 0) {\includegraphics[height=1.5cm, clip, trim=27.5cm 1.2cm 7cm 1.2cm]{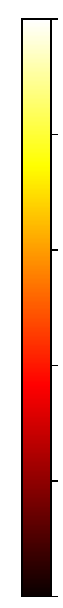}};
					\node at (6.35, -0.71) {\tiny 0.0};
					\node at (6.35, -0.71+0.27) {\tiny 0.2};
					\node at (6.35, -0.71+0.27+0.3) {\tiny 0.4};
					\node at (6.35, -0.71+0.27+0.3+0.27) {\tiny 0.6};
					\node at (6.35, -0.71+0.27+0.31+0.27+0.27) {\tiny 0.8};
					\node at (6.35, -0.71+0.27+0.31+0.27+0.33+0.22) {\tiny 1.0};
			\end{tikzpicture}};\end{tikzpicture}
};
            \node at (7.4, 0) {\begin{tikzpicture}
			\node[anchor=center] at (2.55, -1.05) {\footnotesize (a)(ii) Samples $p \sim \datameas$ ($47 \times 47$ grid)};
			\node[minimum width=1.5cm, minimum height=1.5cm, outer sep=0, inner sep=0] at (0, -2) {\includegraphics[width=1.5cm, clip, trim=13cm 1.2cm 12.3cm 1.2cm]{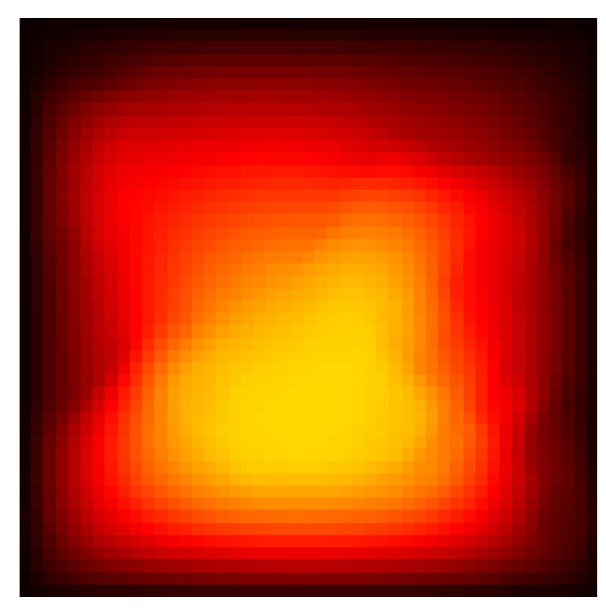}};
			\node[minimum width=1.5cm, minimum height=1.5cm, outer sep=0, inner sep=0] at (1.7, -2) {\includegraphics[width=1.5cm, clip, trim=13cm 1.2cm 12.3cm 1.2cm]{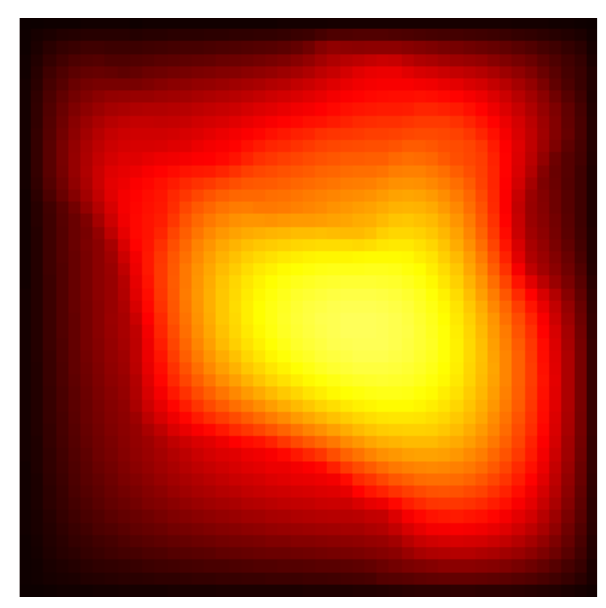}};
			\node[minimum width=1.5cm, minimum height=1.5cm, outer sep=0, inner sep=0] at (3.4, -2) {\includegraphics[width=1.5cm, clip, trim=13cm 1.2cm 12.3cm 1.2cm]{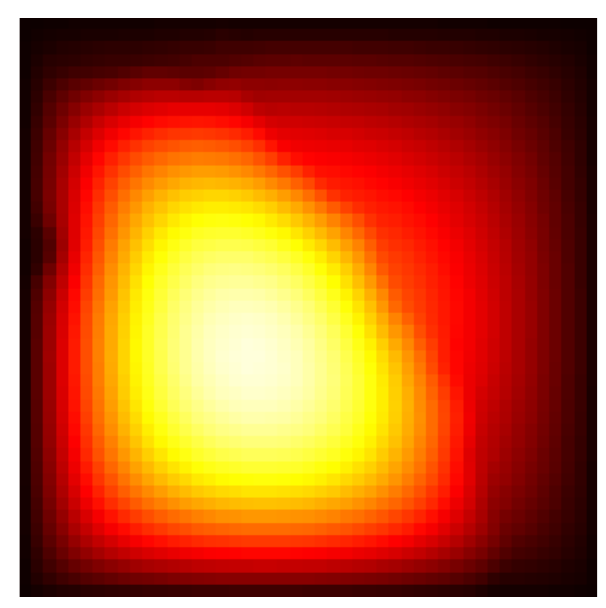}};
			\node[minimum width=1.5cm, minimum height=1.5cm, outer sep=0, inner sep=0] at (5.1, -2) {\includegraphics[width=1.5cm, clip, trim=13cm 1.2cm 12.3cm 1.2cm]{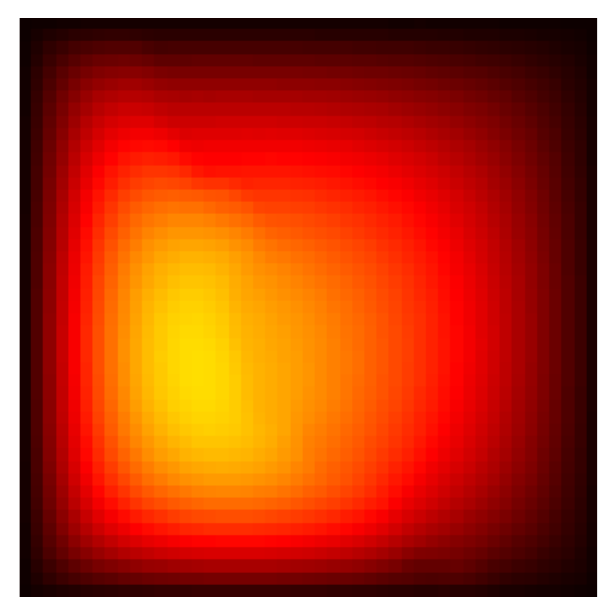}};
			\node at (6.1, -2.02){\begin{tikzpicture}
					\node at (6.1, 0) {\includegraphics[height=1.5cm, clip, trim=27.5cm 1.2cm 7cm 1.2cm]{images/darcy_generative/colorbar}};
					\node at (6.35, -0.71) {\tiny 0.0};
					\node at (6.35, -0.71+0.27) {\tiny 0.2};
					\node at (6.35, -0.71+0.27+0.3) {\tiny 0.4};
					\node at (6.35, -0.71+0.27+0.3+0.27) {\tiny 0.6};
					\node at (6.35, -0.71+0.27+0.31+0.27+0.27) {\tiny 0.8};
					\node at (6.35, -0.71+0.27+0.31+0.27+0.33+0.22) {\tiny 1.0};
			\end{tikzpicture}};
            \end{tikzpicture}};
            \node at (3.8, -1.2) {\footnotesize (b) Kernel density estimates for quantities of interest (based on 1,024 samples from generative model)};
            
            \node at (-0.2, -1.65) {\footnotesize (i) $\mathsf{Q}_{1}(p) = \displaystyle\max_{x \in \Omega} p(x)$};
            \node at (-0.2, -3.35) {
            \includegraphics[width=5.6cm, height=3cm, clip, trim=3.4cm 0.4cm 15.5cm 0.8cm]{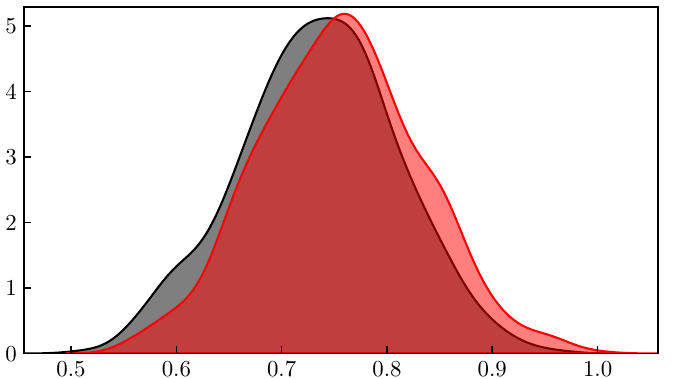}
            };
            \node at (7.3, -1.65) {\footnotesize (ii) $\mathsf{Q}_{2}(p) = \norm{p}_{L^{2}}$};
            \node at (7.2, -3.35) {
            \includegraphics[width=5.6cm, height=3cm, clip, trim=15.8cm 0.4cm 3.0cm 0.8cm]{images/Darcy_kde_plots}
            };
		\end{tikzpicture}
  
	\vspace{-1em}
	\caption{(a) Uncurated samples from the FAE generative model for the pressure field $p$. Further samples are provided in \cref{subsec:details_Darcy_flow}. (b) The distributions of quantities of interest computed using the FAE generative model closely agree with the ground truth.}
	\label{fig:Darcy_generative}
\end{figure}

To measure generative performance, we approximate the distributions of physically relevant quantities of interest $\mathsf{Q}_{i}(p)$ depending on the data $p \in \dataspace$, comparing 1,024 samples from the generative model to the held-out data.
Using kernel density estimates as in \cref{fig:sde_latent_variable}(b), we see close agreement between the distributions (\cref{fig:Darcy_generative}(b)).
While we could also evaluate the generative model using distances such as maximum mean discrepancy (MMD; \citealp{Borgwardtetal2006}), we focus on interpretable quantities relevant to the physical system at hand.

Though we adopt the convention of training the autoencoder and generative model separately \citep{Rombachetal2022} here, the models could also be trained jointly;
we leave this, and an investigation of generative models on the FAE latent space, to future work.

\section{Related Work}
\label{sec:related_work}

\paragraph{Variational Autoencoding Neural Operators.}
The VANO model \citep{SeidmanKissasPappasPerdikaris2023} was the first to attempt systematic extension of the VAE objective to function space. The paper uses ideas from operator learning to construct a model that can decode---but not encode---at any resolution. Our approach differs in both training objective and practical implementation, as we now outline.

The most significant difference between what is proposed in this paper and in VANO is the objective on function space: the VANO objective coincides with a specific case of our model \eqref{eq:encoder_infinite_dimensions}--\eqref{eq:latent_distribution_infinite_dimensions} with the decoder noise $\decodernoisemeas$ being white noise on $L^{2}([0, 1]^{d})$.
As a consequence the generative model for VANO takes values in $\dataspace = H^{s}([0, 1]^{d})$ if and only if $s <- \nicefrac{d}{2}$; in particular generated draws are not in $L^2([0,1]^{d}).$
Unlike our approach, VANO aims to maximise an extension of the ELBO \eqref{eq:VAE_ELBO}, in which a regularisation parameter $\beta > 0$ is chosen as a hyperparameter, and the ELBO takes the form
\begin{align*}
	\ELBO^{\mathrm{VANO}}_{\beta}(u; \encoderparam, \decoderparam) 
	&= \EE_{z \sim \varmeas^{\encoderparam}} \Biggl[\log \frac{\rd \decodermeas^{\decoderparam}}{\rd \decodernoisemeas}(u) \Biggr] - \beta \bigKLdiv{\varmeas^{\encoderparam}}{\latentmeas} \\
	&= \EE_{z \sim \varmeas^{\encoderparam}} \left[ \Pwint{\decodermap(z; \decoderparam)}{u}_{L^{2}} - \tfrac{1}{2} \norm{\decodermap(z; \decoderparam)}_{L^{2}}^{2} \right] - \beta \bigKLdiv{\varmeas^{\encoderparam}}{\latentmeas},
\end{align*}
where the second equality comes from the Cameron--Martin theorem as in \Cref{ex:infinite_ELBO_Dirac}.
Maximising $\ELBO^{\text{VANO}}_{\beta}$ with $\beta = 1$ is precisely equivalent to minimising the per-sample loss \eqref{eq:L2_white_noise_per-sample_loss} from \Cref{ex:infinite_ELBO_Dirac}, so naive application of $\ELBO^{\text{VANO}}_{\beta}$ will result in the same issues seen there.
In particular, 
discretisations of the ELBO may diverge as resolution is refined;
moreover, for data with $L^{2}$-regularity, the generative model $\genmeas^{\decoderparam}$ is greatly misspecified, with draws $\decodermap(z;\decoderparam) + \eta$, $z \sim \latentmeas$, $\eta \sim \decodernoisemeas$, lying in a Sobolev space of lower regularity than the data.
This issue is obscured by the convention in the VAE literature of considering only the decoder mean $\decodermap(z; \decoderparam)$;
considering the full generative model with draws $\decodermap(z; \decoderparam) + \eta$ reveals the incompatibility more clearly.
We argue that the empirical success of VANO in autoencoding is because the objective can be seen as that of a regularised autoencoder (\Cref{rk:problems_L2_data}).

Along with the differences in the training objective and its interpretation, FVAE differs greatly from VANO in architecture.
While the VANO decoders can be discretised on any mesh---and our decoder closely resembles VANO's nonlinear decoder---its encoders assume a fixed mesh for training and inference.
In contrast, our encoder can be discretised on any mesh,
enabling many of our contributions, such as masked training, inpainting, and superresolution, which are not possible within VANO.

\paragraph{Generative Models on Function Space.}
Aside from VANO, recent years have seen significant interest in the development of generative models on function space.
Several extensions of score-based \citep{Songetal2021} and denoising diffusion models \citep{SohlDicksteinetal2015,HoJainAbbeel2020} to function space have been proposed \citep[e.g.,][]{PidstrigachMarzoukReichWang2023,Hagemannetal2023,Limetal2023,KerriganLeySmyth2023,Franzeseetal2023,ZhangWonka2024}.
\Citet{Rahmanetal2022} propose the generative adversarial neural operator (GANO), extending Wasserstein generative adversarial neural networks \citep{ArjovskyChintalaBottou2017} to function space with FNOs in the generator and discriminator to achieve resolution-invariance.

\paragraph{Variational Inference on Function Space.}
In machine learning, variational inference on function space also arises in the context of Bayesian neural networks \citep{SunZhangShiGrosse2019,Burtetal2021,CinquinBamler2024}. 
In this setting one wishes to minimise the KL divergence between the posterior in function space and a computationally tractable approximation---but, as in our study, this divergence may be infinite owing to a lack of absolute continuity between the two distributions.

\paragraph{Learning on Point Clouds.}
Our architecture takes inspiration from the literature on machine learning on point clouds, where data are viewed as sets of points with arbitrary cardinality.
Several models for autoencoding and generative modelling with point clouds have been proposed, such as energy-based processes \citep{YangDaiDaiSchuurmans2020} and SetVAE \citep{KimYooLeeHong2021};
our work differs by defining a loss in function space, ensuring that our model converges to a continuum limit as the mesh is refined.
Continuum limits of semisupervised algorithms for graphs and point clouds have also been studied \citep[e.g.,][]{DunlopSlepcevStuartThorpe2020}.

\section{Outlook}
\label{sec:outlook}

Our study of autoencoders on function space has led to FVAE, an extension of VAEs which imposes stringent requirements on the data distribution in infinite dimensions but benefits from firm probabilistic foundations; it has also led to the non-probabilistic FAE, a regularised autoencoder which can be applied much more broadly to functional data.

\paragraph{Benefits.} 
Both FVAE and FAE offer significant benefits when working with functional data, such as enabling training with data across resolutions, inpainting, superresolution, and generative modelling.
These benefits are possible only through our pairing of a well-defined objective in function space with mesh-invariant encoder and decoder architectures.

\paragraph{Limitations.}
FVAE can be applied only when the generative model is sufficiently compatible with the data distribution---a condition that is difficult to satisfy in infinite dimensions, and restricts the applicability to FVAE to specific problem classes.
FAE overcomes this restriction, but does not share the probabilistic foundations of FVAE.

The desire to discretise the encoder and decoder on arbitrary meshes rules out many high-performing grid-based architectures, including convolutional networks and FNOs.
We believe this is a limiting factor in the numerical experiments, and that combining our work with more complex operator architectures \citep[e.g.,][]{Kovachkietal2023} or continuum extensions of point-cloud CNNs \citep{Lietal2018} would yield further improvements.

\paragraph{Future Work.}

Our work gives new methods for nonlinear dimension reduction in function space, and we expect there to be benefit in building operator-learning methods that make use of the resulting latent space, in the spirit of PCA-Net \citep{BhattacharyaHosseiniKovachkiStuart2021}.
For FAE, which unlike FVAE is not inherently a generative model, we expect particular benefit in the use of more sophisticated generative models on the latent space, for example diffusion models, analogous to Stable Diffusion \citep{Rombachetal2022}.

While our focus has been on scientific problems with synthetic data, our methods could also be applied to real-world data, for example in computer vision;
for these challenging data sets, further research on improved mesh-invariant architectures will be vital.
Our study has also focussed on the typical machine-learning setting of a fixed dataset of size $N$; 
research into the behaviour of FVAE and FAE in the infinite-data limits using tools from statistical learning theory would also be of interest.

\acks{%
JB is supported by Splunk Inc.
MG is supported by a Royal Academy of Engineering Research Chair, and Engineering and Physical Sciences Research Council (EPSRC) grants EP/T000414/1, EP/W005816/1, EP/V056441/1, EP/V056522/1, EP/R018413/2, EP/R034710/1, and EP/R004889/1.
HL is supported by the Warwick Mathematics Institute Centre for Doctoral Training and gratefully acknowledges funding from the University of Warwick and the EPSRC (grant EP/W524645/1).
AMS is supported by a Department of Defense Vannevar Bush Faculty Fellowship and by the SciAI Center, funded by the Office of Naval Research (ONR), under grant N00014-23-1-2729.
For the purpose of open access, the authors have applied a Creative Commons Attribution (CC BY) licence to any Author Accepted Manuscript version arising.
}

\clearpage
\appendix

\section{Supporting Results}

\label{sec:supporting_results}

 In the following proof we use the fact that the norm of the Sobolev space $H^{s}([0, 1])$ can be written as a weighted sum of frequencies \citep[see][Sec.~9]{KreinPetunin1966}:
\begin{equation} \label{eq:Sobolev_norm}
    \norm{u}_{H^{s}([0, 1])}^{2} = \sum_{j \in \Naturals} \bigl(1 + j^{2}\bigr)^{s} |\alpha_{j}|^{2},\quad u = \sum_{j \in \Naturals} \alpha_{j} e_{j}, \quad e_{j}(x) = \sqrt{2} \sin(\pi j x).
\end{equation}

\begin{proofof}{\Cref{prop:white_noise}}
    Let $\eta$ be an $L^{2}$-white noise, let $h = \sum_{j \in \Naturals} h_{j} e_{j} \in L^{2}([0, 1])$, and note that $\decodernoisemeas(\quark - h)$ is the distribution of the random variable $\eta + h$.
    Thus, writing out the $H^{s}$-norm using the Karhunen--Lo\`eve expansion of $\eta$, we see that 
    \begin{equation} \label{eq:white_noise_norm_series}
        \norm{\eta + h}_{H^{s}([0, 1])}^{2} = \sum_{j \in \Naturals} \bigl(1+j^{2}\bigr)^{s} \bigabsval{\xi_{j} + h_{j}}^{2}.
    \end{equation}
    First, we show that $\norm{\eta + h}_{H^{s}([0, 1])} < \infty$ almost surely when $s < -\nicefrac{1}{2}$. 
    To do this we apply the Kolmogorov two-series theorem \citep[Theorem~2.5.6]{Durrett2019}, which states that the random series \eqref{eq:white_noise_norm_series} converges almost surely if 
    \begin{equation*}
        \sum_{j \in \Naturals}  \bigl(1 + j^{2}\bigr)^{s} \EE\Bigl[ \bigl(\xi_{j} + h_{j}\bigr)^{2} \Bigr] < \infty \text{~~~~and~~~~} \sum_{j \in \Naturals} \bigl(1 + j^{2}\bigr)^{2s} \Var \Bigl(  \bigl(\xi_{j} + h_{j}\bigr)^{2} \Bigr) < \infty.
    \end{equation*}
    But, since $\xi_{j} \sim N(0, 1)$, we know that $\EE[\xi_{j}^{2}] = 1$ and $\Var(\xi_{j}^{2}) = 2$;
    applying this, the elementary identity $(x + y)^{2} \leq 2x^{2} + 2y^{2}$ for $x, y \in \Reals$, and the fact that $\sum_{j \in \Naturals} j^{\alpha} < \infty$ for $\alpha < -1$ shows that the two series are finite.
    To see that $\decodernoisemeas(\quark - h)$ assigns zero probability to $L^{2}([0, 1])$, suppose for contradiction that $\eta + h$ had finite $L^{2}$-norm; then $\eta$ would also have finite $L^{2}$-norm.
    But as a consequence of the Borel--Cantelli lemma,
\begin{equation*}
	\bignorm{\eta}_{L^{2}([0, 1])}^{2} = \sum_{j \in \Naturals} \xi_{j}^{2} = \infty \text{~~almost surely,}
\end{equation*}
because the summands are independent and identically distributed, and thus for any constant $c > 0$, infinitely many summands exceed $c$ with probability one.
\end{proofof}

\begin{lemma}
    \label[lemma]{lem:SDE_KL}
	Suppose that $\dataspace = C_{0}([0, T], \Reals^{m})$ and that $\mu \in \prob{\dataspace}$ and $\nu \in \prob{\dataspace}$ are the laws of the $\Reals^{m}$-valued diffusions
    \begin{equation*}
	\begin{alignedat}{3}
		\rd u_{t} &= b(u_{t}) \,\rd t + \sqrt{\varepsilon} \,\rd w_{t},\qquad &u_{0} =0,\qquad &t \in [0, T]\\
		\rd v_{t} &= c(v_{t}) \,\rd t + \sqrt{\varepsilon} \,\rd w_{t},\qquad &v_{0} = 0,\qquad &t \in [0, T].
	\end{alignedat}
    \end{equation*}
	where $(w_{t})_{t \in [0, T]}$ is a Brownian motion on $\Reals^{m}$.
	Suppose that the Novikov condition \eqref{eq:Novikov_condition} holds for both processes.
	Then 
    \begin{equation*}
		\KLdiv{\mu}{\nu} = \EE_{u \sim \mu} \left[ \frac{1}{2\varepsilon} \int_{0}^{T} \norm{b(u_{t}) - c(u_{t})}_{2}^{2} \,\rd t  \right].
    \end{equation*}
\end{lemma}

\begin{proof}
	Applying the Girsanov formula \eqref{eq:Girsanov_density} to obtain the density $\rd \mu / \rd \nu$, taking logarithms to evaluate $\KLdiv{\mu}{\nu}$, and noting that under $\mu$ we have
$\rd u_{t} = b(u_{t}) \,\rd t + \sqrt{\varepsilon}\,\rd w_{t}$, we obtain
\begin{equation*}
		\KLdiv{\mu}{\nu} = \EE_{u \sim \mu} \left[  \frac{1}{2\varepsilon} \int_{0}^{T} \norm{b(u_{t}) - c(u_{t})}_{2}^{2} \,\rd t - \frac{1}{\sqrt{\varepsilon}} \int_{0}^{T} \langle b(u_{t}) - c(u_{t}), \,\rd w_{t}\rangle \right].
	\end{equation*}
	Under $\mu$, the process $(w_{t})_{t \in [0, T]}$ is Brownian motion and so the second expectation is zero.
\end{proof}

\section{Experimental Details}

In this section, we provide additional details, training configurations, samples, and analysis for the numerical experiments in \cref{subsec:FVAE_numerical_examples} and \cref{subsec:FAE_numerical_examples}.
All experiments were run on a single NVIDIA GeForce RTX 4090 GPU with 24 GB of VRAM.

\subsection{Base Architecture} 
\label{subsec:details_base_architecture} 
We use the common architecture described in \cref{subsec:VAE_architecture} and \cref{subsec:FAE_architecture} for all experiments, using the Adam optimiser \citep{KingmaBa2015} with the default hyperparameters $\varepsilon$, $\beta_{1}$, and $\beta_{2}$;
we specify the learning rate and learning-rate decay schedule for each experiment in what follows.

\paragraph{Positional Encodings.}
Where specified, both the encoder and decoder will make use of Gaussian random Fourier features \citep{Tanciketal2020}, pairing the query coordinate $x \in \Omega \subset \Reals^{d}$ with a positional encoding $\gamma(x) \in \Reals^{2k}$.
To generate these encodings, a matrix $B \in \Reals^{k \times d}$ with independent $N(0, I)$ entries is sampled and viewed as a hyperparameter of the model to be used in both the encoder and decoder.
The positional encoding $\gamma(x)$ is then given by the concatenated vector $\gamma(x) = \bigl[ \cos(2\pi Bx); \sin(2\pi Bx) \bigr]^{T} \in \Reals^{2k}$
where the sine and cosine functions are applied componentwise to the vector $2\pi Bx$.

\subsection{Brownian Dynamics}
\label{subsec:details_Brownian_dynamics}

The training data consists of 8,192 samples from the path distribution $\datameas$ of the SDE \eqref{eq:Brownian_dynamics},\eqref{eq:Brownian_dynamics_potential} on the time interval $[0, T]$, $T = 5$. Trajectories are generated using the Euler--Maruyama scheme with internal time step $\nicefrac{1}{8,192}$ (unrelated to the choice to take 8,192 training samples), and the resulting paths are then subsampled by a factor of $80$ to obtain the training data.
Thus the data have effective time increment $\nicefrac{5}{512}$; moreover the path information is removed at $50\%$ of the points resulting from these time increments,  chosen uniformly at random.

\paragraph{Experimental Setup.}
We train for 100,000 steps with initial learning rate $10^{-3}$ and an exponential decay of 0.98 applied every 1,000 steps, with batch size 32 and $4$ Monte Carlo samples for $\varmeas^{\encoderparam}$.
We use latent dimension $d_{\latentspace} = 1$, $\beta = 1.2$ and $\lambda = 10$.
The three sets of simulations shown in \cref{fig:sde_reconstructions_and_samples} 
use $\kappa = 0$, 25, and 10,000 respectively.

\subsection{Estimation of Markov State Models}
\label{subsec:details_MSM}

\begin{figure}[htb]
    \centering
    \includegraphics[width=0.85\linewidth]{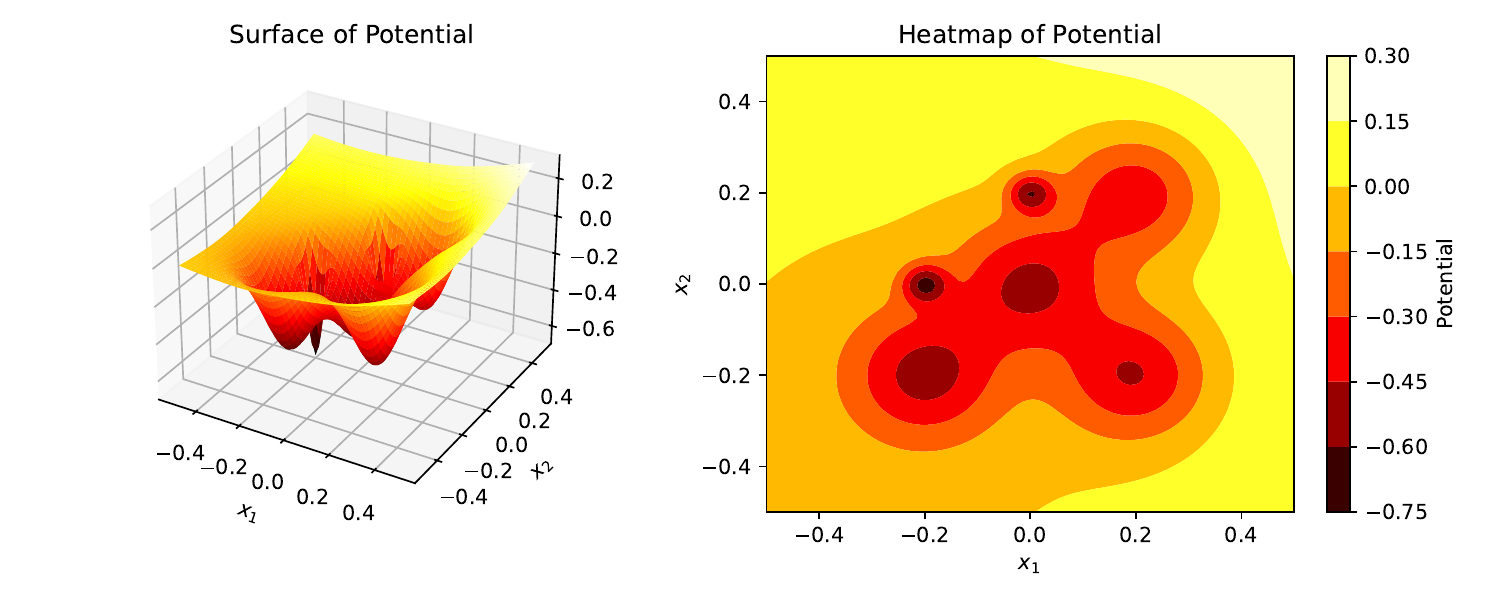}
	\caption{Potential function $U \colon \Reals^{2} \to \Reals$ for \cref{subsec:MSM}.}
    \label{fig:sde2d_potential}
\end{figure}

\paragraph{Data and Discretisation.}

To validate the ability of FVAE to model higher-dimensional SDE trajectories, we specify a simple potential with qualitative features similar to those arising in the complex potential surfaces arising in molecular dynamics.
To this end, define the centres $c_{1}  = (0, 0)$, $c_{2} = (0.2, 0.2)$, $c_{3} = (-0.2, -0.2)$, $c_{4} = (0.2, -0.2)$, $c_{5} = (0, 0.2)$ and $c_{6} = (-0.2, 0)$;
standard deviations $\sigma_{1} = \sigma_{2} = \sigma_{3} = \sigma_{4} = 0.1$ and $\sigma_{5} = \sigma_{6} = 0.03$;
and masses $m_{1} = m_{2} = m_{3} = m_{4} = 0.1$ and $m_{5} = m_{6} = 0.01$.
Then  let
\begin{equation*}
U(x) = 0.3 \Biggl[  0.5(x_1 + x_2) + x_1^2 + x_2^2 - \sum_{i = 1}^{6} m_{i} N\left(x; c_{i}, \sigma_{i}^{2} I_{2} \right) \Biggr].
\end{equation*}
This potential has three key components: a linear term breaking the symmetry, a quadratic term preventing paths from veering too far from the path's starting point, the origin, and negative Gaussian densities---serving as potential wells---positioned at the centres $c_{i}$ (\cref{fig:sde2d_potential}). 
Sample paths of \eqref{eq:Brownian_dynamics} with initial condition $u_{0} = 0$, temperature $\varepsilon = 0.1$ and final time $T = 3$ show significant diversity, with many paths transitioning at least once between different wells (see ground truth in \cref{fig:sde2d_recs_dataset}).

\paragraph{Experimental Setup.}

The training set consists of 16,384 paths generated with an Euler--Maruyama scheme with internal time step $\nicefrac{1}{8,192}$, subsampled by a factor $48$ to obtain an equally spaced mesh of $513$ points.
We take $d_{\latentspace} = 16$, $\beta = 10$, $\kappa = 50$, and $\lambda = 50$, and, as in \cref{subsec:details_MSM}, train on data where 50\% of the points on the path are missing. We also use the same learning rate, learning-rate decay schedule, step limit, and batch size.

\paragraph{Results.}
FVAE's reconstructions closely match the inputs (\cref{fig:sde2d_recs_dataset}), and FVAE produces convincing generative samples capturing qualitative features of the data (\cref{fig:sde2d_samples_dataset}).

\begin{figure}[htb]
    \centering
    \begin{subfigure}[b]{\linewidth}
        \centering
        \includegraphics[width=0.8\linewidth]{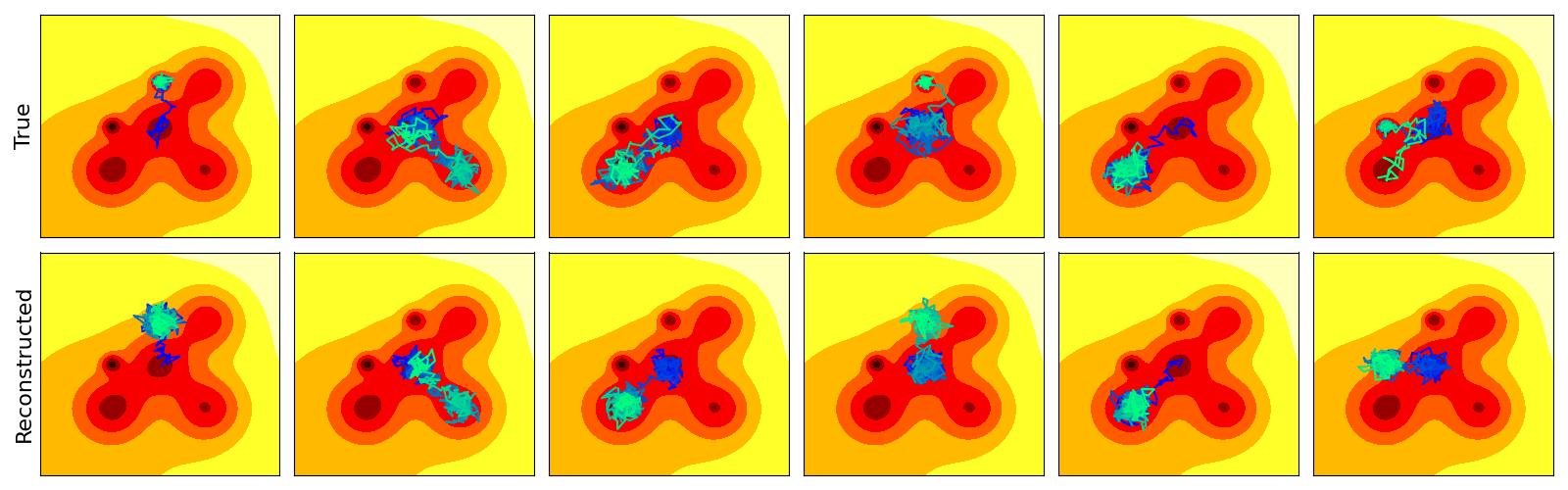}
    \end{subfigure}
    
    \vspace{-1em} 
    \rule{0.85\linewidth}{0.4pt}
    
    \begin{subfigure}[b]{\linewidth}
        \centering
        \includegraphics[width=0.8\linewidth]{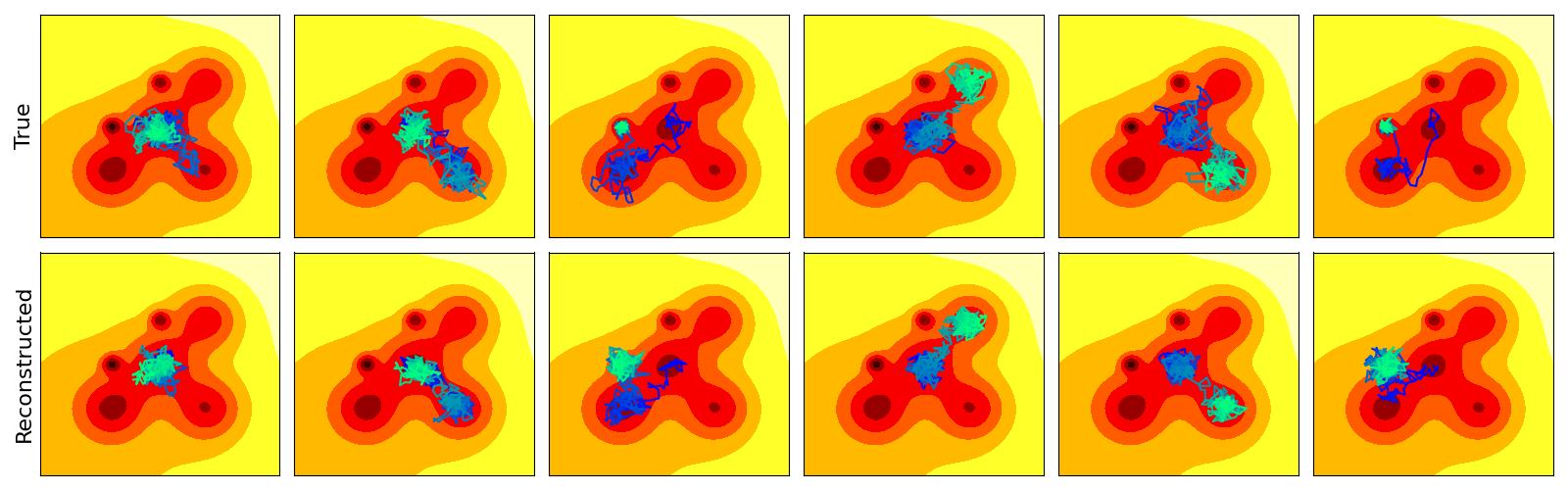}
    \end{subfigure}
    
	\caption{Held-out ground-truth data from the SDE in \cref{subsec:MSM} (``True'' row) and the corresponding FVAE reconstructions of sample paths (``Reconstructed'' row).}
    \label{fig:sde2d_recs_dataset}
\end{figure}

\begin{figure}[htbp]
    \centering
    \begin{subfigure}{0.49\textwidth}
        \centering
        \includegraphics[width=1\textwidth]{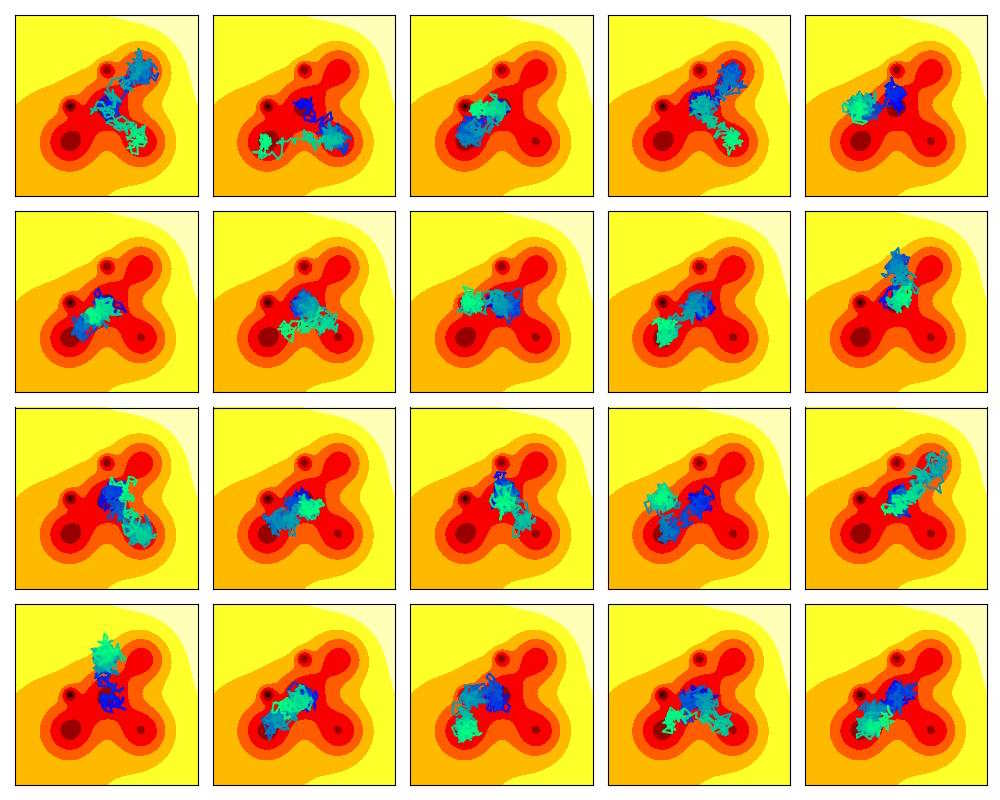}
        \caption{FVAE}
    \end{subfigure}
    \begin{subfigure}{0.49\textwidth}
        \centering
        \includegraphics[width=1\textwidth]{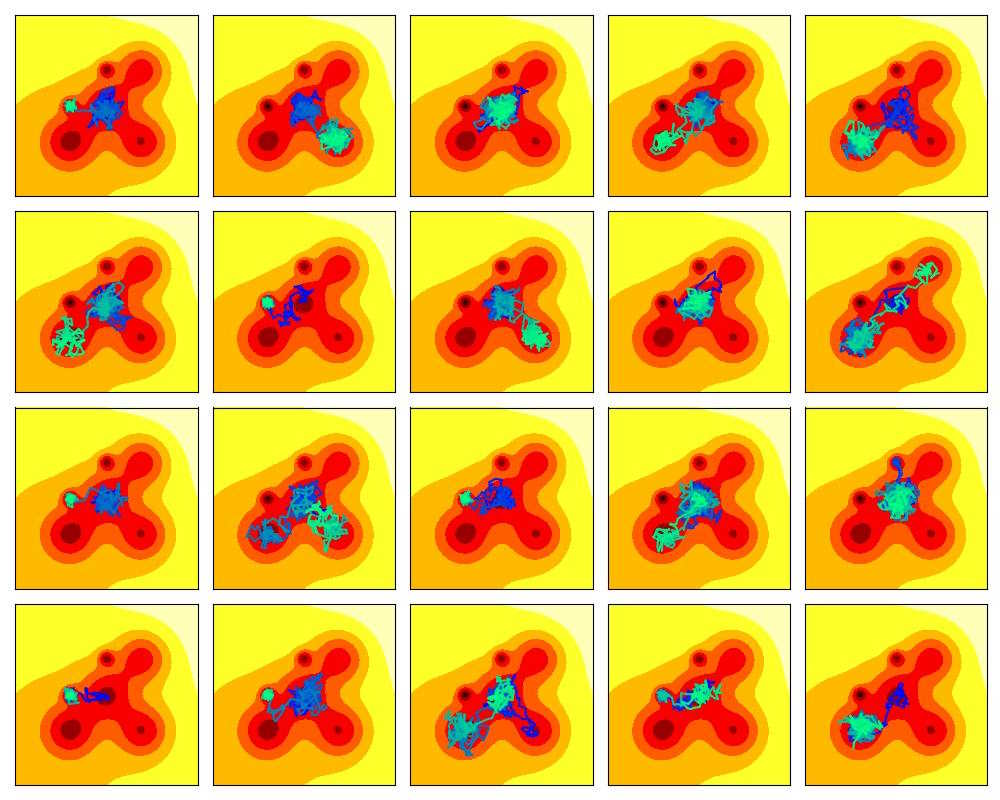}
        \caption{Data set}
    \end{subfigure}
    
	\caption{(a) Samples of the SDE in \cref{subsec:MSM} drawn from the FVAE generative model with randomly drawn latent vector $z \sim \latentmeas$. (b) Ground-truth paths of the SDE in \cref{subsec:MSM} generated using an Euler--Maruyama solver. In both subfigures, the evolution through time $t \in [0, 3]$ is depicted as a transition in colour from blue to green.}
    \label{fig:sde2d_samples_dataset}
\end{figure}

\subsection{Dirac Distributions}
\label{subsec:details_random_dirac}

\paragraph{Data and Discretisation.}
We view  $\datameas$ as a probability distribution on $\dataspace = H^{-1}([0, 1])$. 
At each resolution $I$, we discretise the domain $[0, 1]$ using an evenly spaced mesh of points $\{\nicefrac{i}{I + 1}\}_{i = 1, \dots, I}$ and approximate the Dirac mass $\dirac{\xi}$, $\xi \in [0, 1]$, by the optimal $L^{1}$-approximation:
a discretised function which is zero except at the mesh point closest to $\xi$, normalised to have unit $L^{1}$-norm.
The training data set consists of discretised Dirac functions at each mesh point;
the goal is not to train a practical model for generalisation, but to isolate the effect of the objective.

\paragraph{Experimental Setup.}
We train FVAE and FAE models at resolutions $I \in \{8, 16, 32, 64, 128\}$.
For each model, we perform 50 independent runs of 30,000 steps with batch size 6.

\paragraph{Architecture.}
The neural network $\rho \colon \Reals \times \Encoderparam \to \Reals \times \Reals$ in the encoder map $\encodermap$ is assumed to have 3 hidden layers of width 128, and the mean $\mu(z; \decoderparam)$ and standard deviation $\sigma(z; \decoderparam)$ in the decoder are computed from a 3-layer neural network of width 128.
For numerical stability, we impose a lower bound on $\sigma$ based on the mesh spacing $\Delta x$, given by $\sigma_{\mathrm{min}}(\Delta x) = (2\pi)^{-1/2} \Delta x$.

\paragraph{FVAE Configuration.}
We view data $u \sim \datameas$ as lying in the Sobolev space $\dataspace = H^{-1}([0, 1])$; 
the decoder $\decodermap$ will output functions in $L^{2}([0, 1])$ and we take decoder-noise distribution $\decodernoisemeas = N(0, I)$, noting that $\decodernoisemeas \in \prob{H^{s}([0, 1])}$ if and only if $s < -\nicefrac{1}{2}$; in particular
white-noise samples do not lie in the space $L^{2}([0, 1])$.
We modify the per-sample loss \eqref{eq:per-sample_loss} by reweighting the term $\KLdiv{\varmeas^{\encoderparam}}{\latentmeas}$ by $\beta = 10^{-4}$, and take $16$ Monte Carlo samples for $\varmeas^{\encoderparam}$.
We use an initial learning rate of $10^{-4}$, decaying exponentially by a factor $0.7$ every 1,000 steps. 

\paragraph{FAE Configuration.}
To compute the $H^{-1}$-norm we truncate the series expansion \eqref{eq:Sobolev_norm} and compute coefficients $\alpha_{j}$ from a discretisation of $u$ using the discrete sine transform.
We use initial learning rate $10^{-4}$, decaying exponentially by a factor $0.9$ every 1,000 steps, and take $\beta = 10^{-12}$.
For consistency with the FVAE loss, we subtract the squared data norm $\frac{1}{2}\norm{u}_{H^{-1}}^{2}$ from the FAE loss, yielding the expression
\begin{equation*}
    \frac{1}{2} \bignorm{\decodermap(\encodermean(u; \encoderparam); \decoderparam) - u}_{H^{-1}}^{2} - \frac{1}{2} \bignorm{u}_{H^{-1}}^{2} = \frac{1}{2} \Norm{\decodermap\bigl(\encodermean(u; \encoderparam); \decoderparam\bigr)}_{H^{-1}}^{2} - \Innerprod{\decodermap\bigl(\encodermean(u; \encoderparam); \decoderparam\bigr)}{u}_{H^{-1}}.
\end{equation*}

\paragraph{Results.}
As expected, the final training loss under both models decreases as the resolution is refined, since the lower bound $\sigma_{\mathrm{min}}$ decreases.
However, the FAE loss appears to converge and is stable across runs, while the FVAE loss appears to diverge and becomes increasingly unstable across runs.
This gives convincing empirical evidence that the joint divergence \eqref{eq:joint_Kullback--Leibler_divergence} for FVAE is not defined as a result of the misspecified decoder noise;
the use of FAE with an appropriate data norm alleviates this issue.
Since the FVAE objective with $\decodernoisemeas = N(0, I)$ coincides with the VANO objective, this issue would also be present for VANO.
Under both models, training becomes increasingly unstable at high resolutions:
when $\sigma$ is small, the loss becomes highly sensitive to changes in $\mu$;
this instability is unrelated to the divergence of the FVAE training loss and is a consequence of training through gradient descent.

\subsection{Incompressible Navier--Stokes Equations}
\label{subsec:details_Navier--Stokes}

\paragraph{Data and Discretisation.}

\begin{table}[h!]
\centering

\begin{tabular}{ccccc}
\toprule
\textbf{Viscosity $\nu$} & \textbf{Resolution} & \textbf{Train Samples} & \textbf{Eval.\ Samples} & \textbf{Snapshot Time $T$} \\\midrule
$10^{-3}$ & $64\times64$ & 4,000 & 1,000  & 50 \\
$10^{-4}$ & $64\times64$ & 8,000 & 2,000 & 50 \\
$10^{-5}$ & $64\times64$ & 960 & 240 & 20 \\
\bottomrule
\end{tabular}
\caption{Details of Navier--Stokes data sets.}
\label{table:ns_data}
\end{table}

We use data as provided online by \citet{Lietal2021}. 
Solutions of \eqref{eq:Navier--Stokes_velocity--vorticity} are generated by sampling the initial condition from the Gaussian random field $N(0, C)$, $C = 7^{3/2} (49I - \Delta)^{-5/2}$, and evolving in time using a pseudospectral method.
While the data of \citet{Lietal2021} includes the full time evolution, we use only snapshots of the vorticity at the final time. Every snapshot is a $64\times64$ image, normalised to take values in $[0, 1]$; 
details of this data set are given in \Cref{table:ns_data}.

\paragraph{Effects of Point Ratios.}
Here, we extend the analysis of \cref{fig:Navier--Stokes_inpainting}(b) to understand how the point ratio used during training affects reconstruction performance.
We first train two FAE models on the Navier--Stokes data set with viscosity $\nu = 10^{-4}$, using complement masking with a point ratio $r_{\text{enc}}$ of 10\% and 90\% respectively.
Then, we fix an arbitrary sample from the held-out set and, for each model, generate 1,000 distinct masks with point ratios 10\%, 30\%, 50\%, 70\%, and 90\%.
We then encode on each mesh and decode on the full grid and compute kernel density estimates of the reconstruction MSE (\cref{fig:appendix_points_ratio_study}).
The model trained with $r_{\text{enc}} = 90\%$ is much more sensitive to the location of the evaluation mesh points, especially when the evaluation point ratio is low;
with sufficiently high encoder point ratio at evaluation time, however, the reconstruction MSE of the model trained using $r_{\text{enc}} = 90\%$ is lower that of the model trained at $r_{\text{enc}} = 10\%$.
This suggests a tradeoff whereby a lower training point ratio provides more stability, at the cost of increasing autoencoding MSE, particularly when the point ratio of the evaluation data is high. 
We hypothesise that a lower training point ratio regularises the model to attain a more robust internal representation.

\begin{figure}[htb]
    \centering
    
    \begin{subfigure}[b]{0.49\linewidth}
        \centering
        \includegraphics[width=0.99\linewidth]{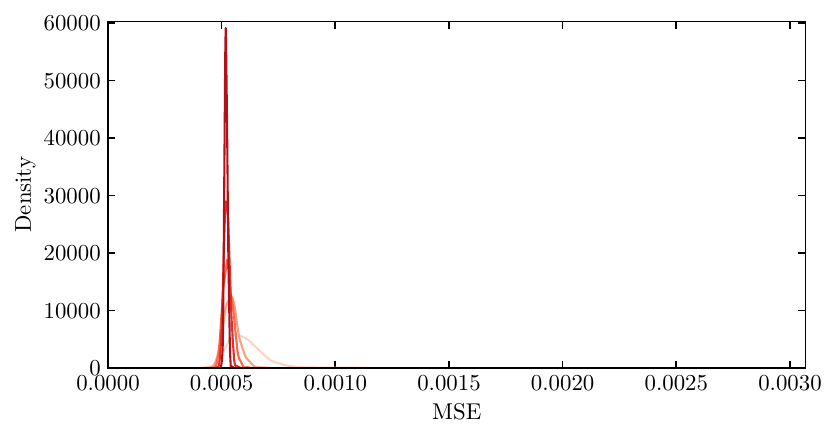}
        \caption{For a model trained with point ratio 10\%.}
    \end{subfigure}
    \begin{subfigure}[b]{0.49\linewidth}
        \centering
        \includegraphics[width=0.99\linewidth]{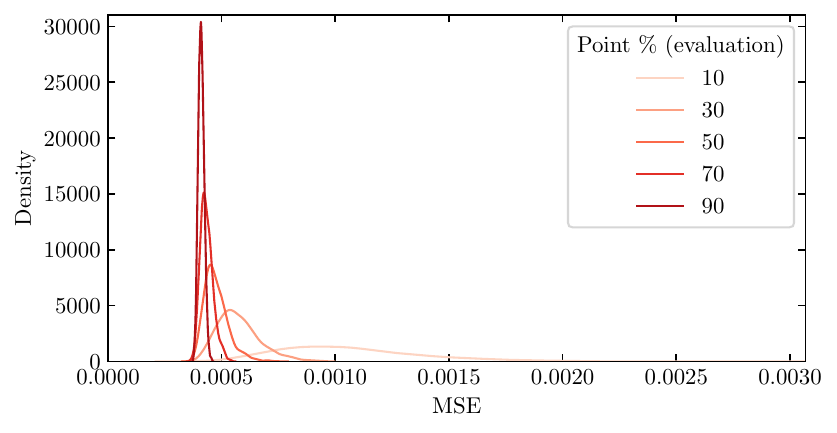}
        \caption{For a model trained with a point ratio 90\%.}
    \end{subfigure}
    
	\caption{Kernel density estimates for full-grid reconstruction MSE on the reference sample across 1,000 randomly chosen meshes. Training with a high point ratio reduces MSE when the evaluation data has a high point ratio, but at the cost of greater variance when evaluating on low point ratios.}
    \label{fig:appendix_points_ratio_study}
\end{figure}

We also investigate the sensitivity of the models to a specific encoder mesh, seeking to understand whether an encoder mesh achieving low MSE on one image leads to low MSE on other images.
The procedure is as follows:
we select an image arbitrarily from the held-out set (the \defterm{reference sample}) and draw 1,000 random meshes with point ratio 10\%;
then, we select the mesh resulting in the lowest reconstruction MSE for each of the two models.
For the nearest neighbours of the chosen sample in the held-out set, the reconstruction error on this MSE-minimising mesh is lower than average (\cref{fig:comparison}(a); dashed lines), suggesting that a good configuration will yield good results on similar samples.
On the other hand, using the MSE-minimising mesh on arbitrary samples from the held-out set yields an MSE somewhat lower than a randomly chosen mesh;
unsurprisingly, however, the arbitrarily chosen samples appear to benefit less than the nearest neighbours (\cref{fig:comparison}(b)).

\begin{figure}[htb]  
    \centering
    \begin{subfigure}[b]{\textwidth}
        \centering
        \includegraphics[clip, trim=0cm 0cm 0cm 0.5cm, width=\textwidth]{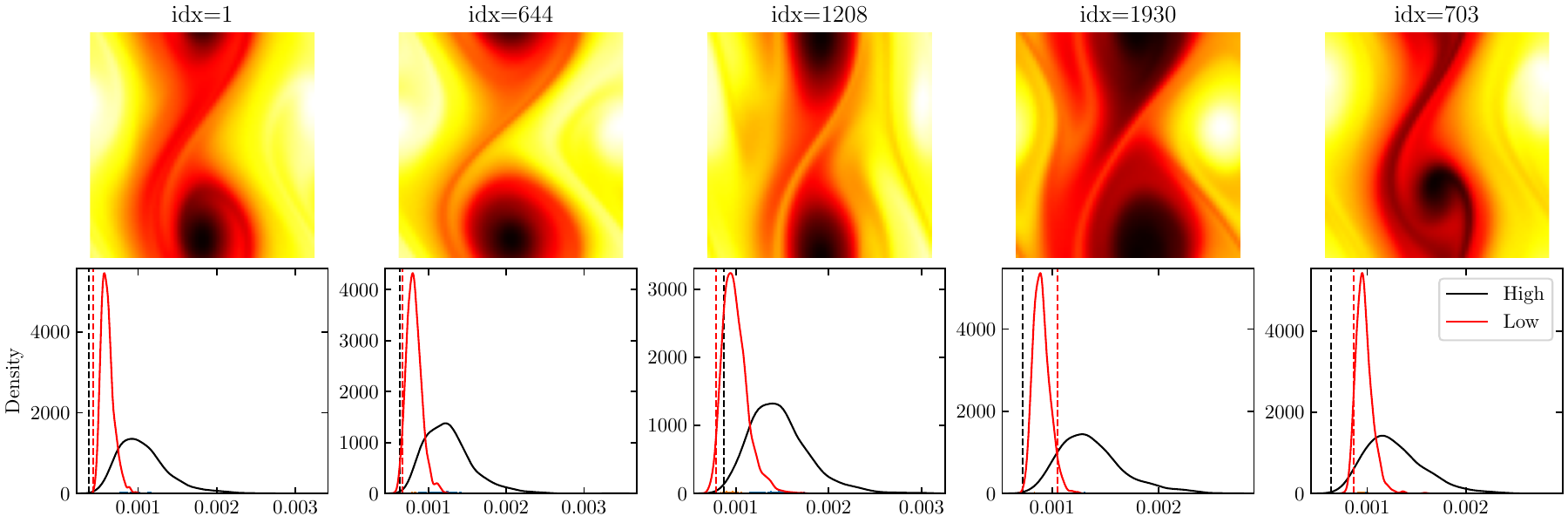}
        \caption{Nearest neighbours of the reference sample in the held-out set.}
    \end{subfigure}
    \vspace{1em}  
    \begin{subfigure}[b]{\textwidth}
        \centering
        \includegraphics[clip, trim=0cm 0cm 0cm 0.5cm, width=\textwidth]{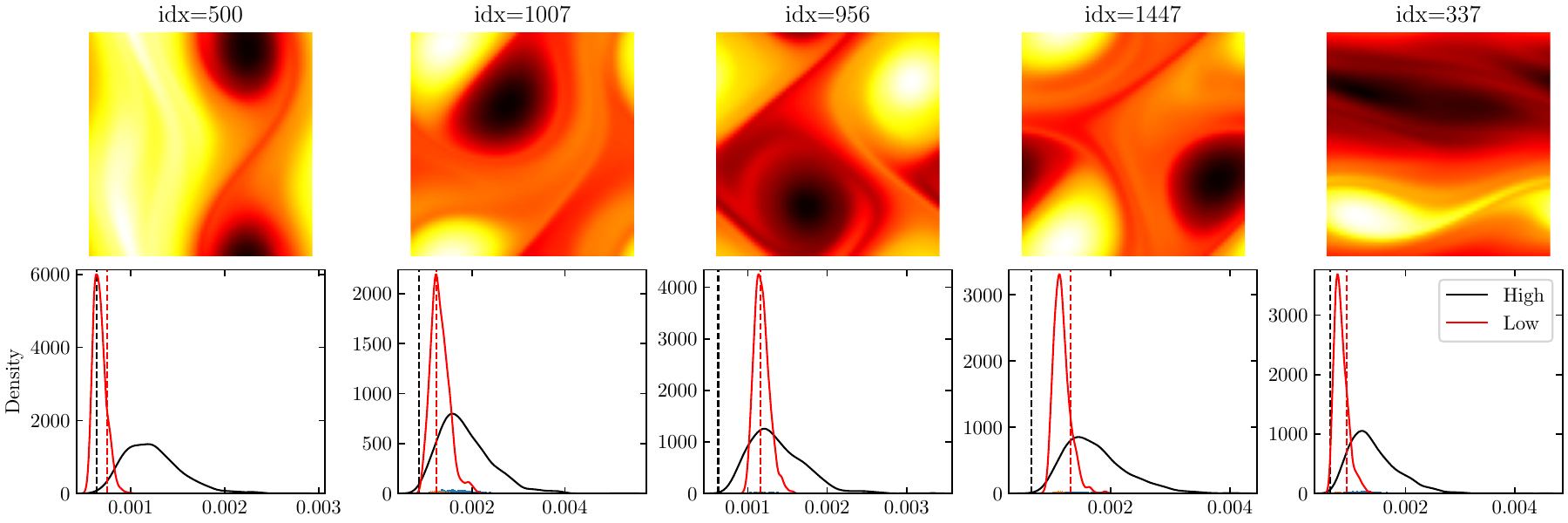}
        \caption{Arbitrarily chosen samples from the held-out set.}
    \end{subfigure}
	\vspace{-2.6em}
    \caption{Kernel density estimates of full-grid reconstruction MSE for models trained at 10\% (Low) and 90\% (High) point ratios on samples from the held-out set. Dashed lines indicate the MSE obtained using the mesh minimising MSE on the reference sample.}
    \label{fig:comparison}
    \vspace{-0.9em}
\end{figure}

\paragraph{Experimental Setup.}

We train for 50,000 steps with batch size 32 and initial learning rate $10^{-3}$, decaying exponentially by a factor 0.98 every 1,000 steps.
We use complement masking with $r_{\text{enc}} = 0.7$, providing a good balance of performance and robustness to masking.

\paragraph{Architecture.}

Both the CNN and FAE architecture use Gaussian random positional encodings with $k=16$.
For the sake of comparison, we use a standard CNN architecture inspired by the VGG model \citep{SimonyanZisserman2015}, gradually contracting/expanding the feature map while increasing/decreasing the channel dimensions. 
The architecture we use was identified using a search over parameters such as the network depth while maintaining a similar parameter count to our baseline FAE model. 
The encoder consists of four CNN layers with output channel dimension 4, 4, 8, and 16 respectively and kernel sizes are 2, 2, 4, and 4 respectively, all with stride 2.
The result is flattened and passed through a single-hidden-layer MLP of width 64 to obtain a vector of dimension 64.
The decoder consists of a single-layer MLP of width 64 and output dimension 512, which is then rearranged to a $4 \times 4$ feature map with channel size 32. 
This feature map is then passed through four layers of transposed convolutions that respectively map to 16, 8, 4, and 4 channel dimensions, with kernel sizes 4, 4, 2, and 2 respectively, and stride 2.
The result is then mapped by two CNN layers with kernel size 3, stride 1, and output channel dimension 8 and 1 respectively.

\paragraph{Uncurated Reconstructions and Samples.}

Reconstructions of randomly selected data from the held-out sets for viscosities $\nu = 10^{-3}$, $10^{-4}$ and $10^{-5}$ are provided in \cref{fig:appendix_Navier--Stokes_rec_1e-3,fig:appendix_Navier--Stokes_rec_1e-4,fig:appendix_Navier--Stokes_rec_1e-5} respectively.
\begin{figure}[htbp]
    \centering
    \begin{subfigure}[b]{0.42\textwidth}
        \centering
        \includegraphics[width=\textwidth]{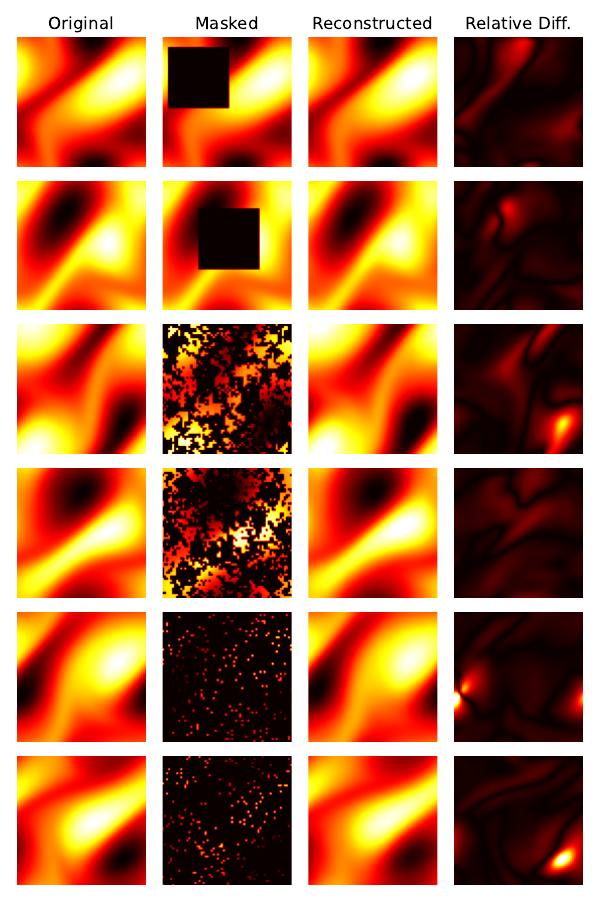}
    \end{subfigure}
    \begin{tikzpicture}
        \draw[thick] (0,0) -- (0,8); 
    \end{tikzpicture}
    \begin{subfigure}[b]{0.42\textwidth}
        \centering
        \includegraphics[width=\textwidth]{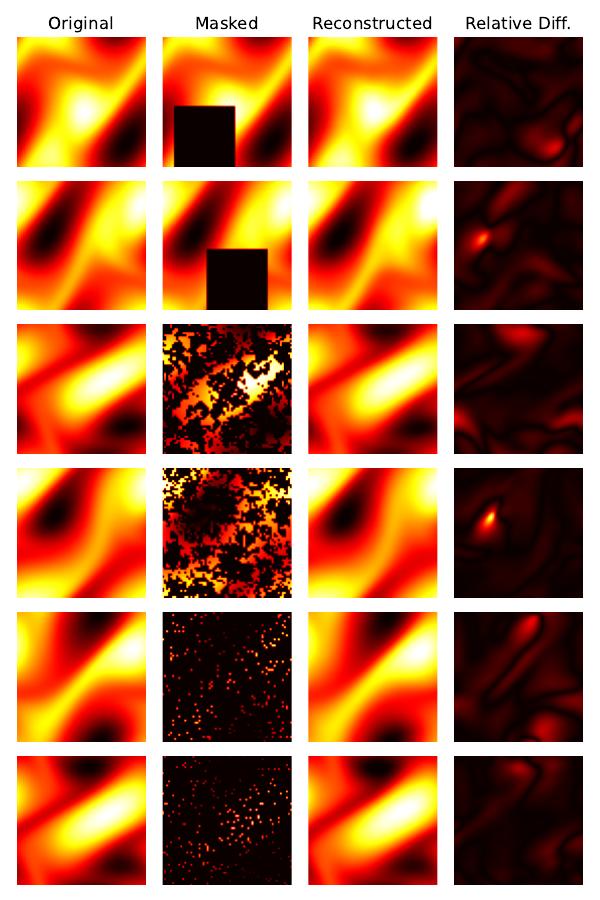}
    \end{subfigure}
    \caption{FAE reconstructions of Navier--Stokes data with viscosity $10^{-3}$.}
    \label{fig:appendix_Navier--Stokes_rec_1e-3}
\end{figure}
\begin{figure}[htbp]
    \centering
    \begin{subfigure}[b]{0.42\textwidth}
        \centering
        \includegraphics[width=\textwidth]{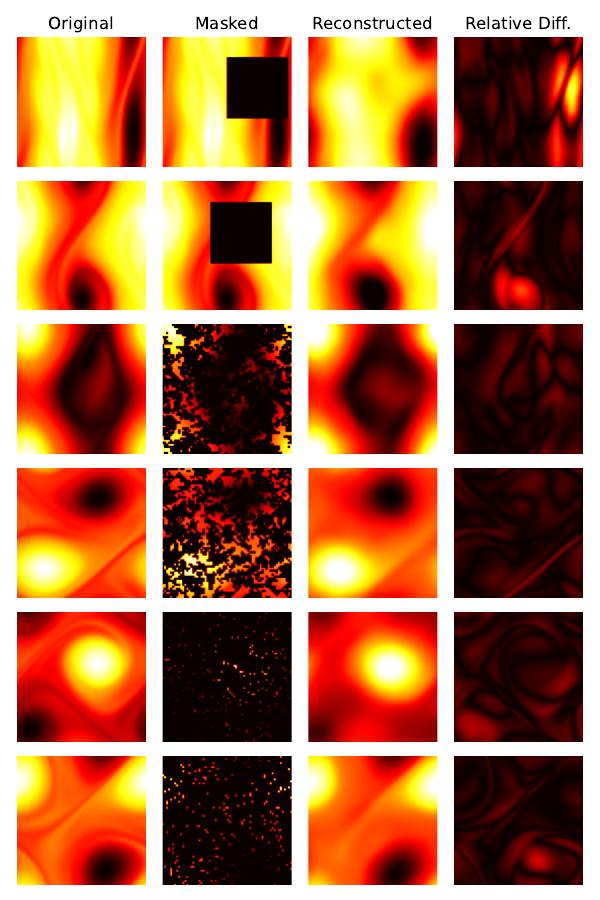}
    \end{subfigure}
    \begin{tikzpicture}
        \draw[thick] (0,0) -- (0,8); 
    \end{tikzpicture}
    \begin{subfigure}[b]{0.42\textwidth}
        \centering
        \includegraphics[width=\textwidth]{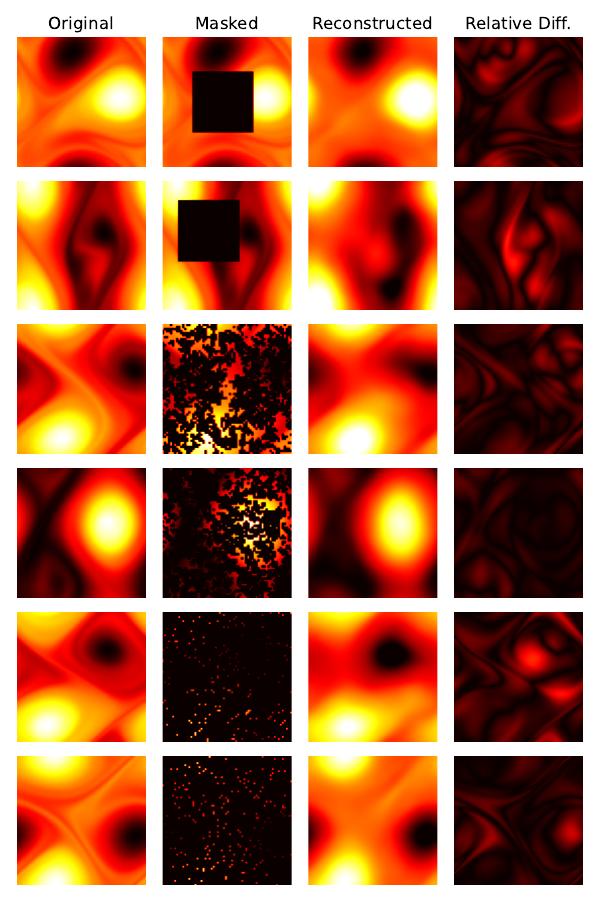}
    \end{subfigure}
    \caption{FAE reconstructions of Navier--Stokes data with viscosity $10^{-4}$.}
        \label{fig:appendix_Navier--Stokes_rec_1e-4}
\end{figure}
\begin{figure}[htbp]
    \centering
    \begin{subfigure}[b]{0.42\textwidth}
        \centering
        \includegraphics[width=\textwidth]{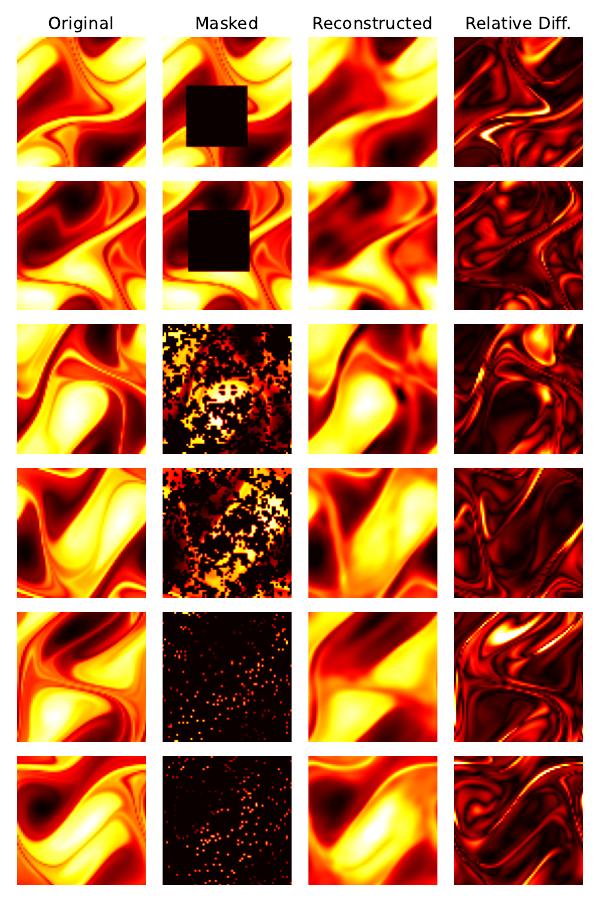}
    \end{subfigure}
    \begin{tikzpicture}
        \draw[thick] (0,0) -- (0,8); 
    \end{tikzpicture}
    \begin{subfigure}[b]{0.42\textwidth}
        \centering
        \includegraphics[width=\textwidth]{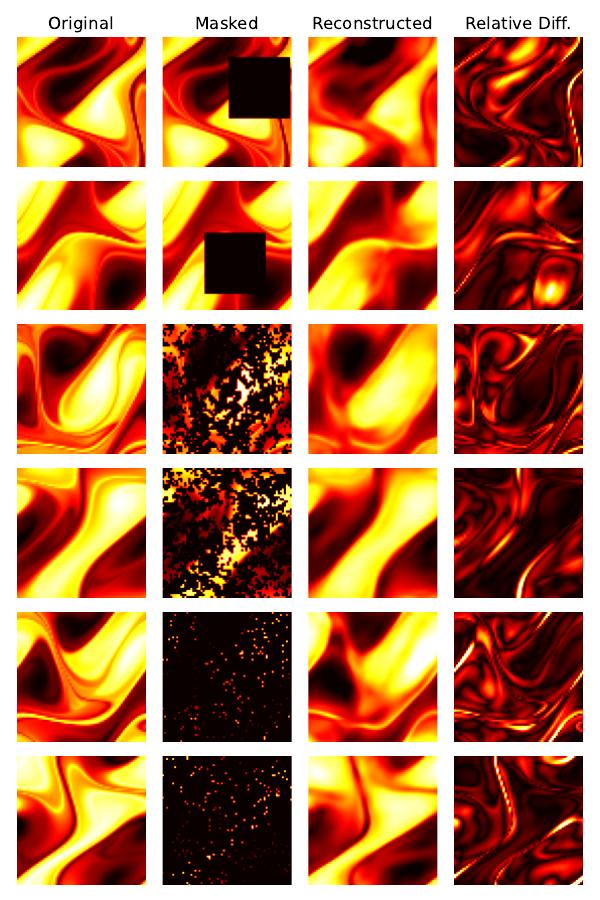}
    \end{subfigure}
    \caption{FAE reconstructions of Navier--Stokes data with viscosity $\nu = 10^{-5}$.}    \label{fig:appendix_Navier--Stokes_rec_1e-5}
\end{figure}
As described in \cref{subsec:Darcy}, we apply FAE as a generative model by fitting a Gaussian mixture with 10 components on the latent space.
Samples from models trained at $\nu = 10^{-3}$, $10^{-4}$ and $10^{-5}$ are shown in \cref{fig:appendix_Navier--Stokes_generative_1e-3,fig:appendix_Navier--Stokes_generative_1e-4,fig:appendix_Navier--Stokes_generative_1e-5} respectively.

\begin{figure}[htbp]
    \centering
    \begin{subfigure}[b]{0.49\textwidth}
        \centering
        \includegraphics[width=\textwidth]{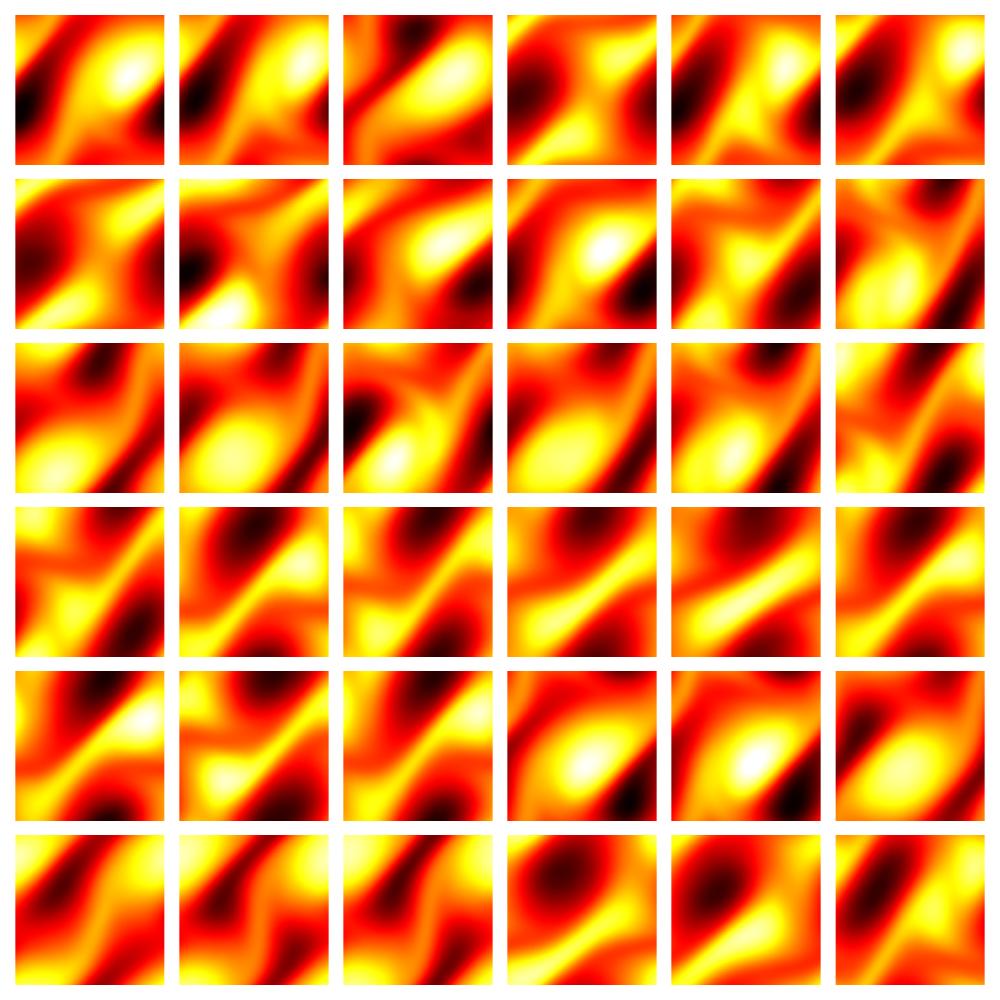}
        \caption{FAE}
    \end{subfigure}
    \begin{subfigure}[b]{0.49\textwidth}
        \centering
        \includegraphics[width=\textwidth]{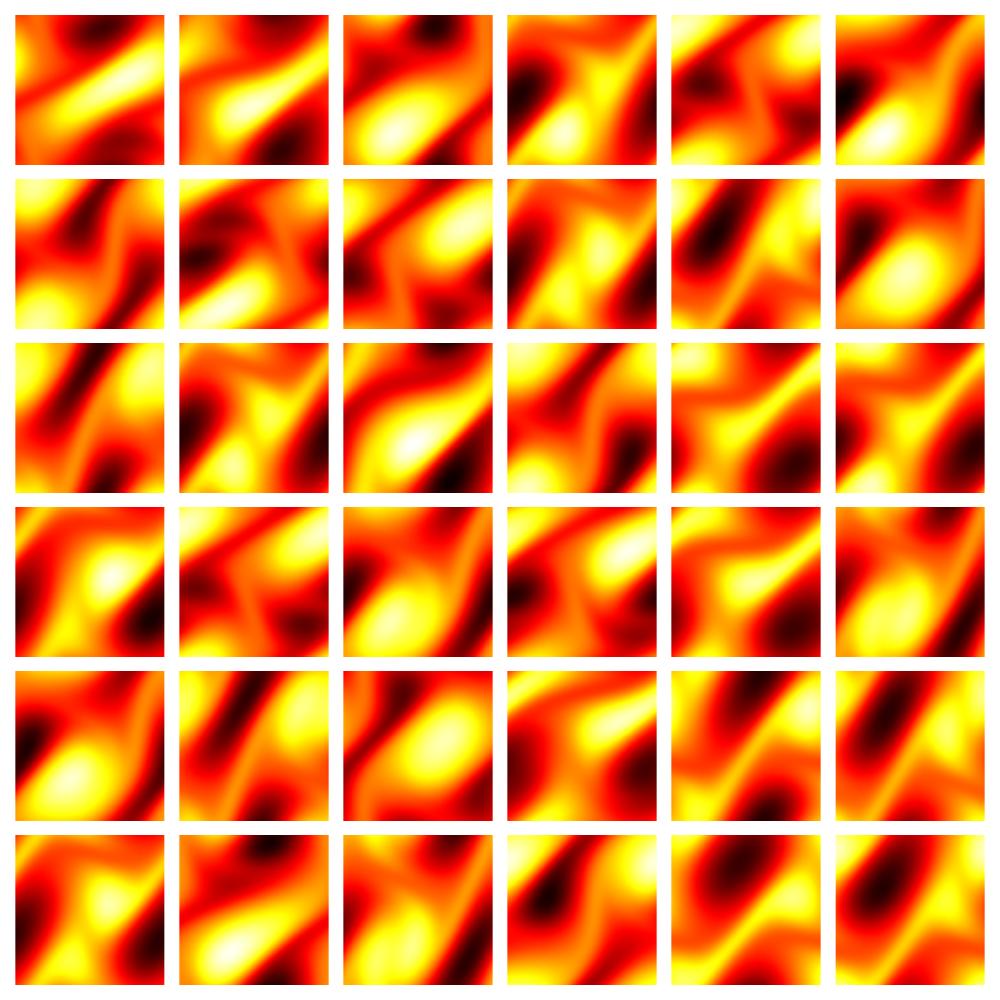}
        \caption{Data set}
    \end{subfigure}
    
    \caption{Samples of Navier--Stokes data with viscosity $\nu = 10^{-3}$.}
    \label{fig:appendix_Navier--Stokes_generative_1e-3}
\end{figure}

\begin{figure}[htbp]
    \centering
    \begin{subfigure}[b]{0.49\textwidth}
        \centering
        \includegraphics[width=\textwidth]{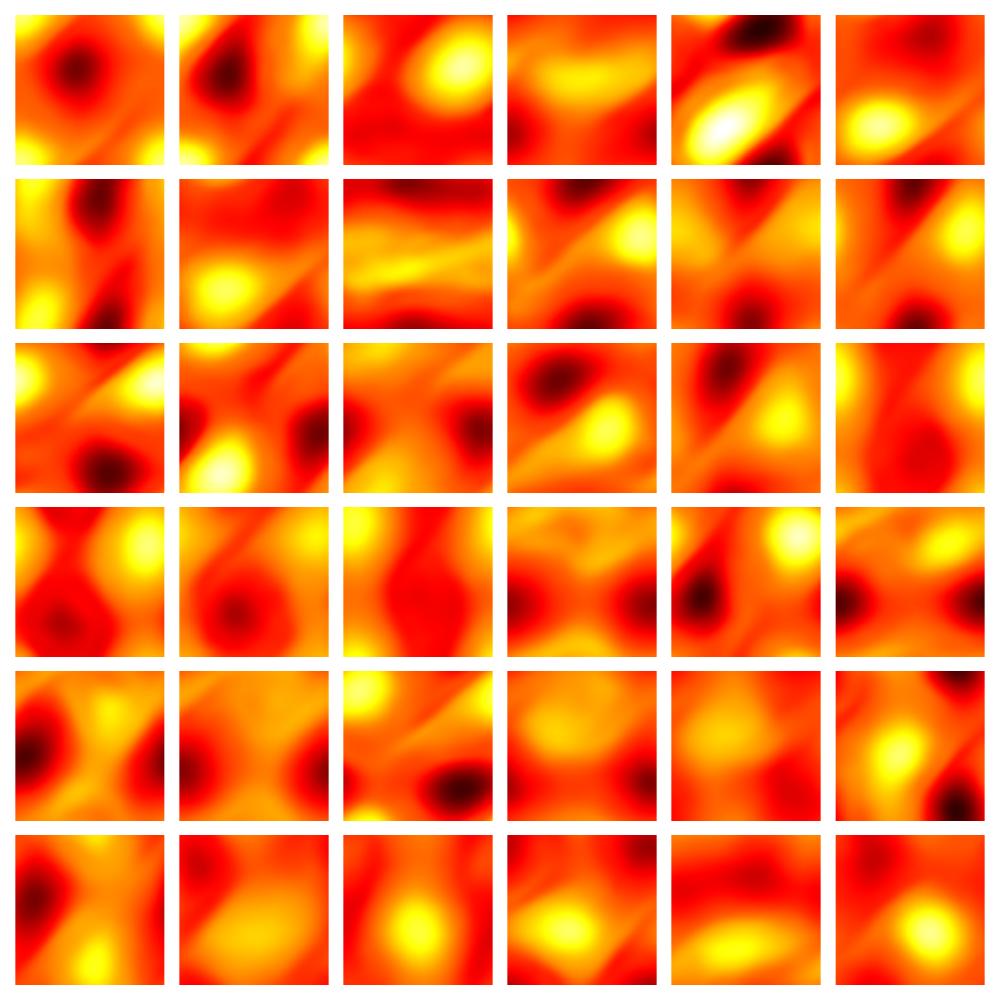}
        \caption{FAE}
    \end{subfigure}
    \begin{subfigure}[b]{0.49\textwidth}
        \centering
        \includegraphics[width=\textwidth]{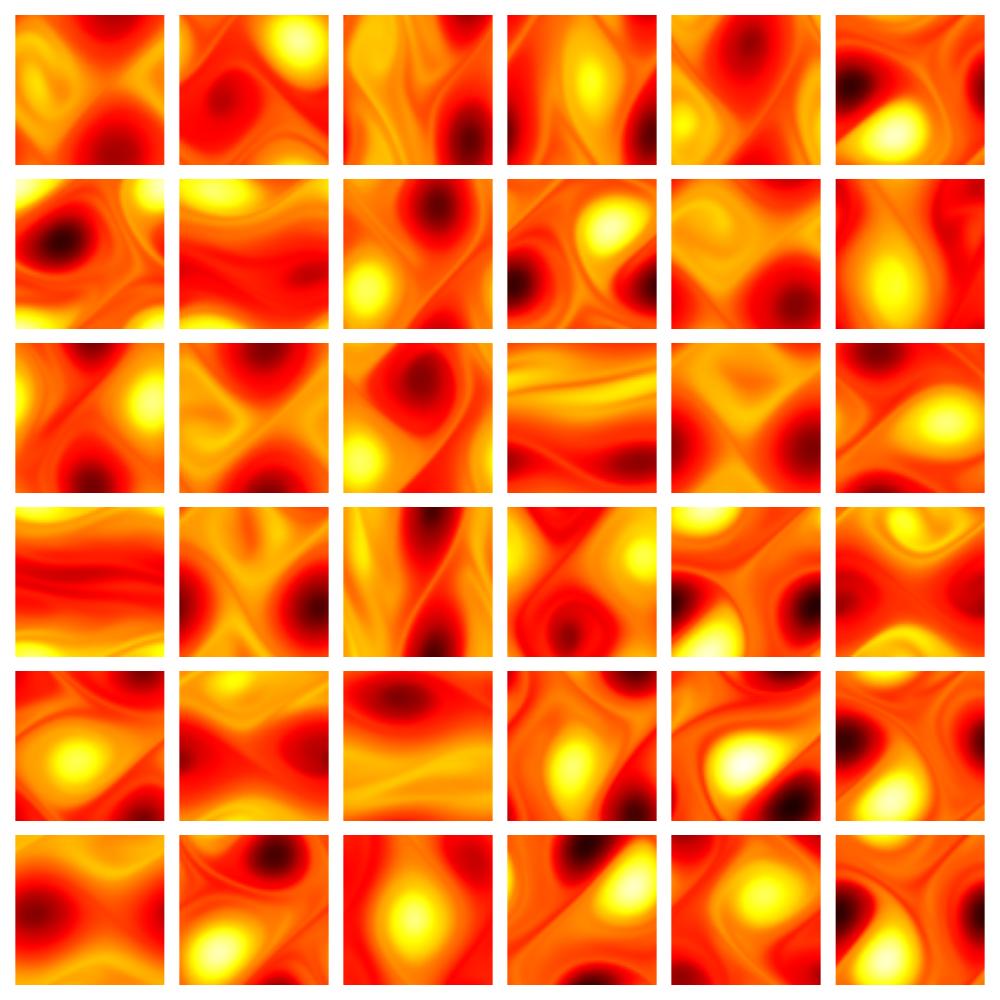}
        \caption{Data set}
    \end{subfigure}
    
    \caption{Samples of Navier--Stokes data with viscosity $\nu = 10^{-4}$.}
    \label{fig:appendix_Navier--Stokes_generative_1e-4}
\end{figure}

\begin{figure}[htbp]
    \centering
    \begin{subfigure}[b]{0.49\textwidth}
        \centering
        \includegraphics[width=\textwidth]{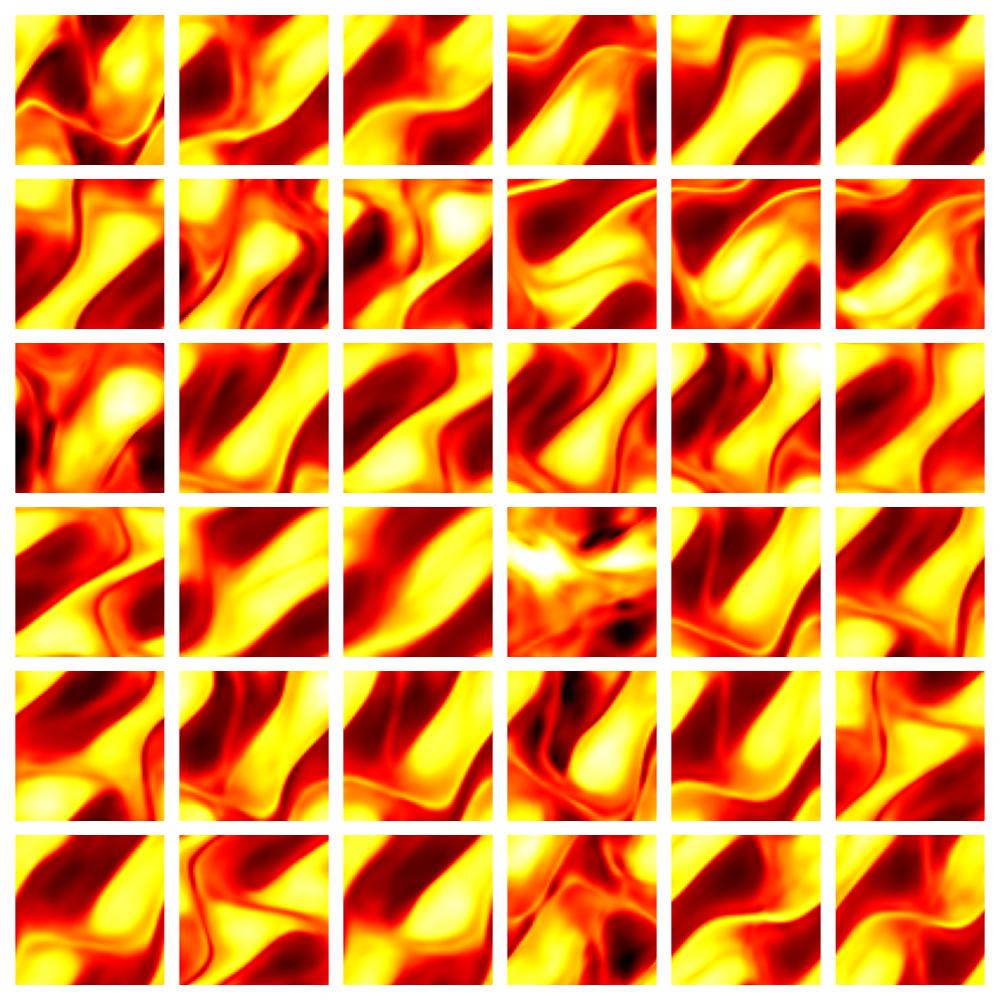}
        \caption{FAE}
    \end{subfigure}
    \begin{subfigure}[b]{0.49\textwidth}
        \centering
        \includegraphics[width=\textwidth]{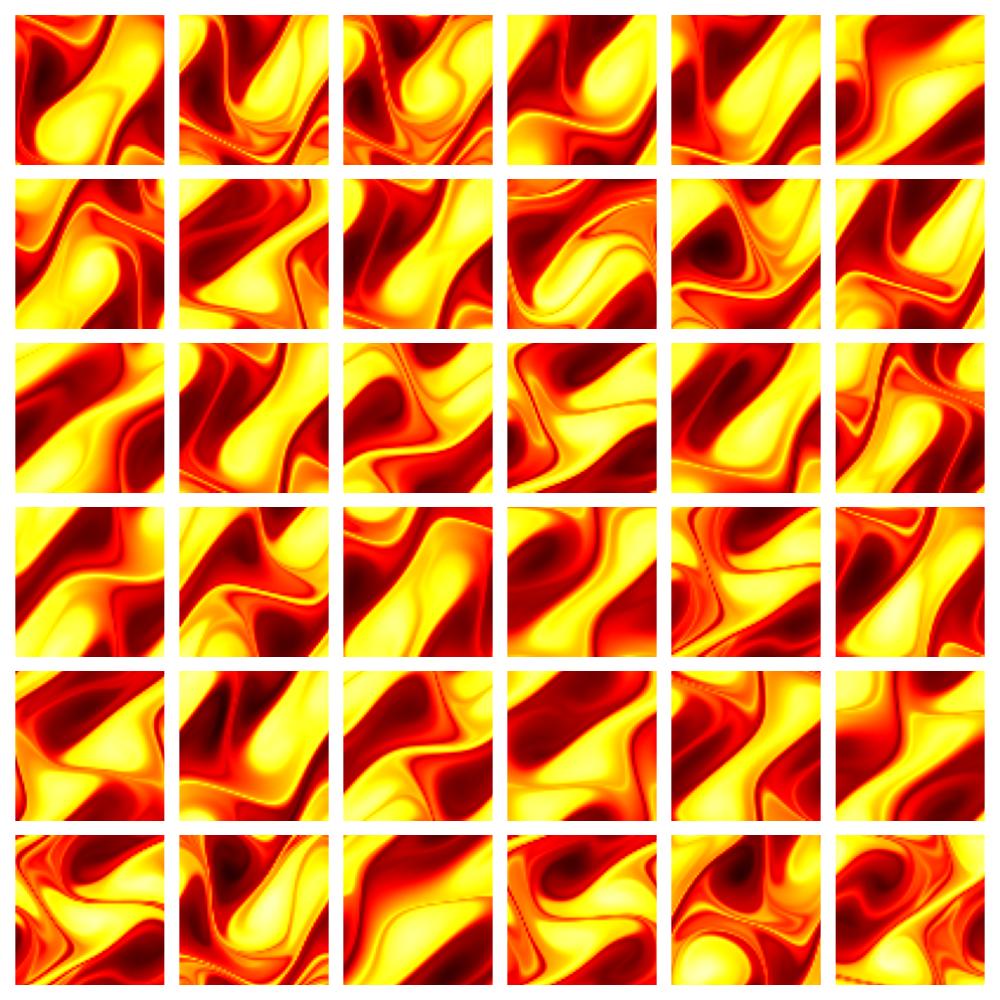}
        \caption{Data set}
    \end{subfigure}
    
    \caption{Samples of Navier--Stokes data with viscosity $\nu = 10^{-5}$.}
    \label{fig:appendix_Navier--Stokes_generative_1e-5}
\end{figure}

\paragraph{Evaluation at Very High Resolutions.}
In \cref{fig:Navier--Stokes_zero-shot_and_latent_interpolation}, we demonstrate zero-shot resolution by evaluating the decoder on grids of resolution $\text{2,048} \times \text{2,048}$ and $\text{32,768} \times \text{32,768}$.
While the former requires approximately 16~MB to store using 32-bit floating-point numbers, the latter requires 4.3~GB, and thus applying a neural network directly to the $\text{32,768} \times \text{32,768}$ image is more likely to exhaust GPU memory.
To allow evaluation of the decoder at this resolution, we partition the domain into 1,000 chunks and evaluate the decoder on each chunk in turn;
we then reassemble the resulting data in the RAM.
To ensure that each chunk has an integer number of points, we take the first 824 chunks to contain 1,073,742 mesh points ($\approx$ 4~MB), and take the remaining 176 chunks to contain 1,073,741 points.

\subsection{Darcy Flow}
\label{subsec:details_Darcy_flow}

\paragraph{Data Set.}
The data we use is based on that provided online by \citet{Lietal2021}, given on a $421 \times 421$ grid and generated through a finite-difference scheme.
Where described, we downsample this data to lower resolutions by applying a low-pass filter in Fourier space and subsampling the resulting image.
The low-pass filter is a mollification of an ideal $\mathrm{sinc}$ filter with bandwidth selected to eliminate frequencies beyond the Nyquist frequency of the target resolution, computed by convolving the ideal filter in Fourier space with a Gaussian kernel with standard deviation $\sigma = 0.1$, truncated to a $7 \times 7$ convolutional filter.

\paragraph{Experimental Setup.}
We follow the same setup used for the Navier--Stokes data set: 
we train for 50,000 steps, with batch size 32 and complement masking with $r_{\text{enc}} = 70\%$.
An initial learning rate of $10^{-3}$ is used with an exponential decay factor of 0.98 applied every 1,000 steps. 
We make use of positional embeddings (\cref{subsec:details_base_architecture}) using $k = 16$ Gaussian random Fourier features.
When performing the wall-clock training time experiment (\cref{fig:Darcy_wallclock}), we downsample the training and evaluation data to resolution $211\times211$.

\paragraph{Uncurated Reconstructions and Samples.}
Reconstructions of randomly selected examples from the held-out evaluation data set are shown in \cref{fig:appendix_Darcy_reconstructions}.
Samples from the FAE generative model and draws from the evaluation data set are shown in \cref{fig:appendix_Darcy_generative}.

\begin{figure}[htbp]
    \centering
    \begin{subfigure}[b]{0.42\textwidth}
        \centering
        \includegraphics[width=\textwidth]{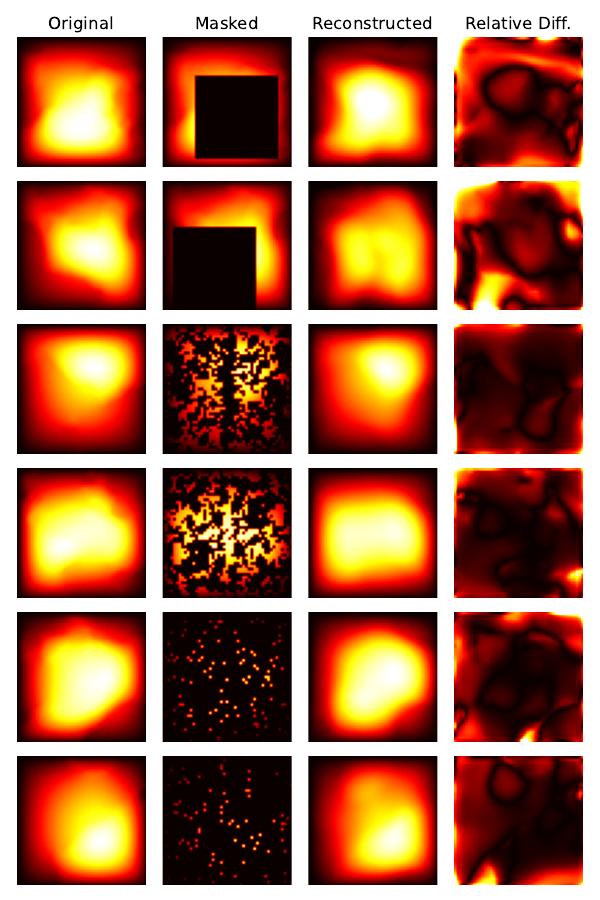}
    \end{subfigure}
    \begin{tikzpicture}
        \draw[thick] (0,0) -- (0,8); 
    \end{tikzpicture}
    \begin{subfigure}[b]{0.42\textwidth}
        \centering
        \includegraphics[width=\textwidth]{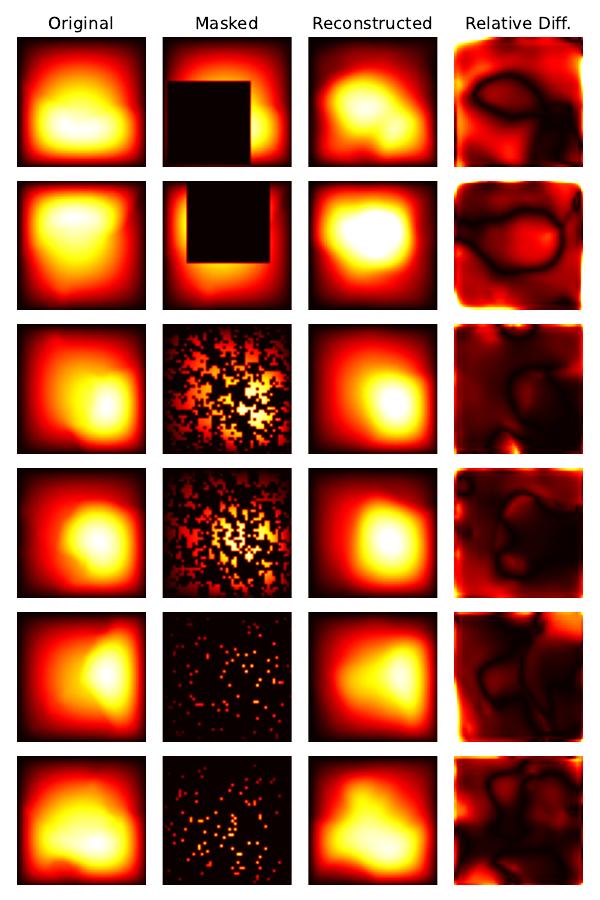}
    \end{subfigure}
    \caption{FAE reconstructions of Darcy flow data.}
    \label{fig:appendix_Darcy_reconstructions}
\end{figure}

\begin{figure}[htbp]
    \centering
    \begin{subfigure}[b]{0.48\textwidth}
        \centering
        \includegraphics[width=\textwidth]{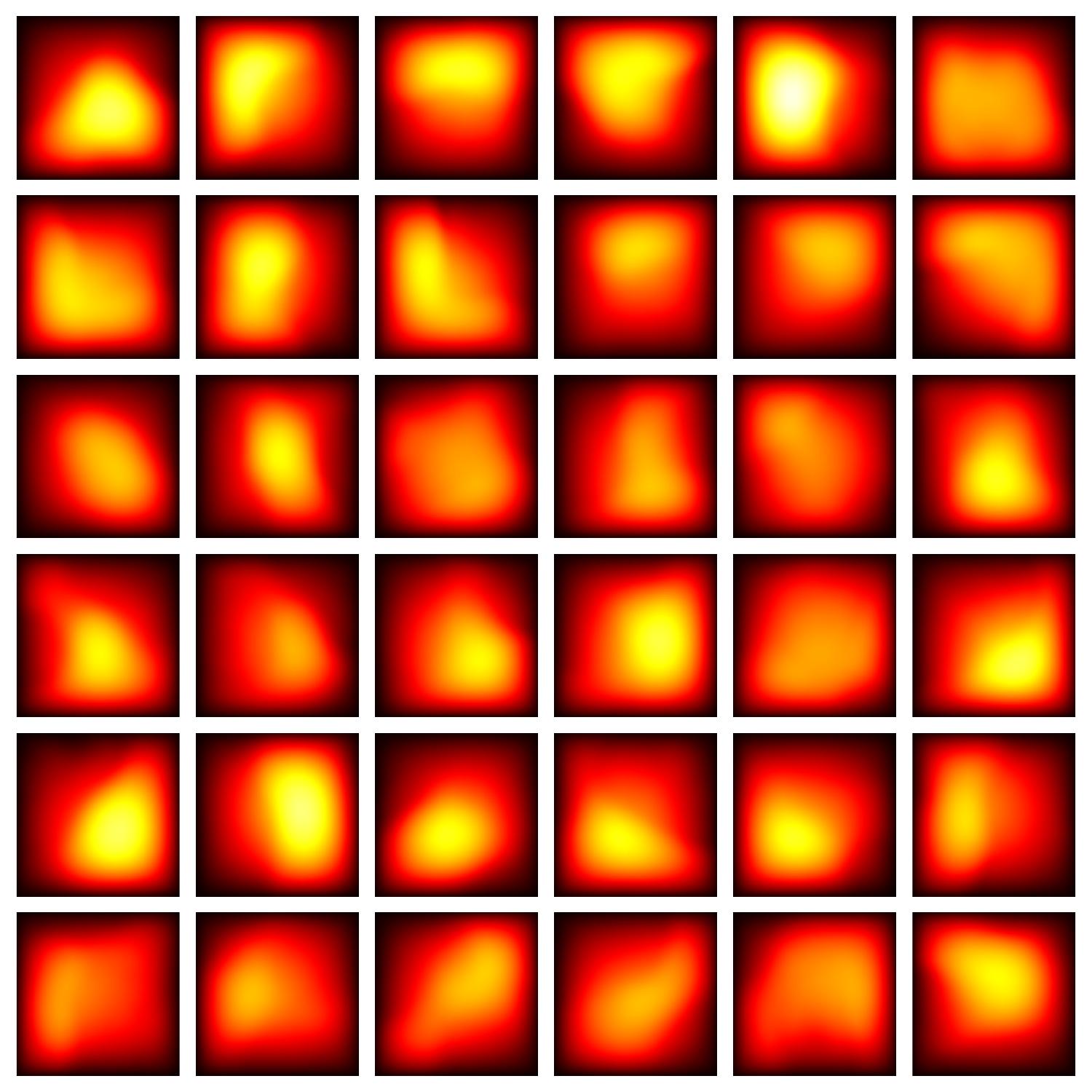}
        \caption{FAE}
    \end{subfigure}
    \begin{subfigure}[b]{0.48\textwidth}
        \centering
        \includegraphics[width=\textwidth]{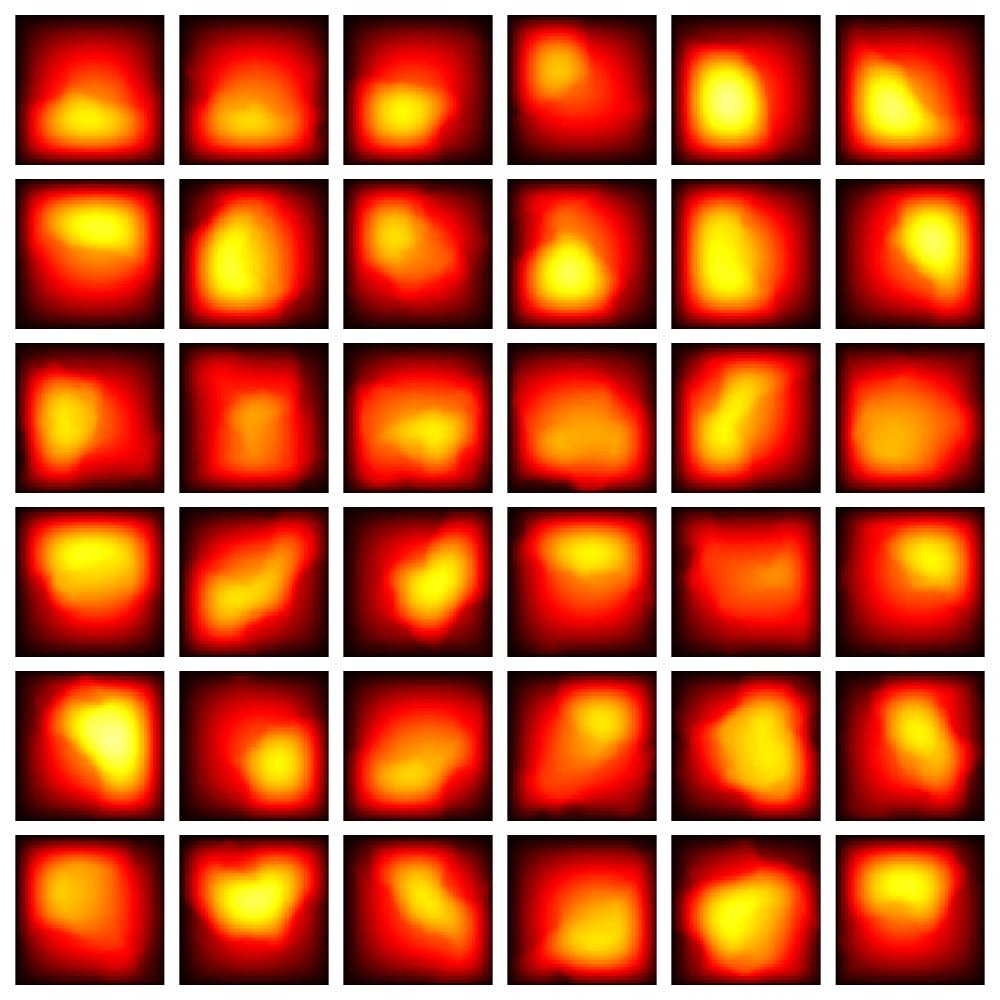}
        \caption{Data set}
    \end{subfigure}
    
    \caption{Samples of Darcy flow data.}
    \label{fig:appendix_Darcy_generative}
\end{figure}

\vskip 0.2in
\bibliography{references}

\end{document}